\documentclass{article}

\usepackage{graphicx}
\usepackage{subfigure}
\usepackage{booktabs} %
\usepackage{multicol}
\usepackage{enumitem}  %
\usepackage{hyperref}
\usepackage{colortbl}
\usepackage{xcolor,amssymb}
\usepackage[most]{tcolorbox}
\usepackage{array}

\newcommand{\goldstar}{\textcolor{yellow!70!orange}{\ensuremath{\bigstar}}}

\usepackage[accepted]{icml2025}

\definecolor{forestgreen}{rgb}{0.13,0.55,0.13}

\usepackage{adjustbox}
\usepackage{multirow}
\usepackage{amsmath}
\usepackage{amssymb}
\usepackage{mathtools}
\usepackage{amsthm}
\usepackage{xspace}       %
\usepackage{comment}
\newtoggle{comment}
\toggletrue{comment}
\usepackage{etoolbox}

\newcommand{\fs}{FActScore}
\newcommand{\bs}{\textsc{ByteSampler}}
\newcommand{\ours}{\textsc{Anchored Decoding}}
\newcommand{\oursbs}{\textsc{Anchored}$_{\mathrm{Byte}}$  \textsc{Decoding}}
\newcommand{\ourlm}{TinyComma 1.8B}
\newcommand{\huggingface}{\raisebox{-1.5pt}{\includegraphics[height=1.05em]{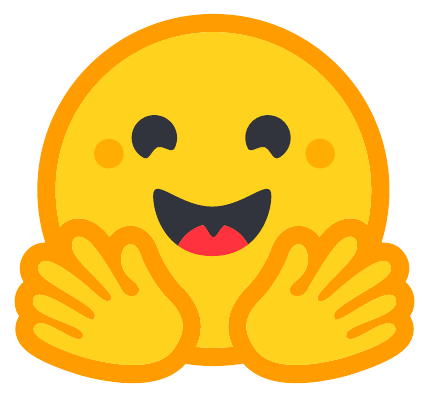}}\xspace}
\newcommand{\github}{\raisebox{-1.5pt}{\includegraphics[height=1.05em]{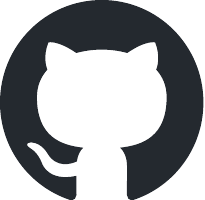}}\xspace}

\usepackage{thmtools}
\usepackage{thm-restate}
\DeclareMathOperator*{\argmin}{arg\,min}
\newcommand{\kl}{D_{KL}}
\newcommand{\kld}{\,\|\,}
\newcommand{\debt}{\delta_{\mathrm{init}}(x)}
\newcommand{\debtbs}{\delta_{\mathrm{init}}(\tilde x)}
\newcommand{\ry}{$\infty$-Rényi~}
\usepackage[capitalize,nameinlink]{cleveref}
\crefname{algorithm}{Alg.}{Algs.}
\Crefname{algorithm}{Alg.}{Algs.} 
\crefname{section}{Sec.}{Secs.}
\Crefname{section}{Sec.}{Secs.}

\theoremstyle{plain}
\newtheorem{theorem}{Theorem}[section]
\newtheorem{proposition}[theorem]{Proposition}
\newtheorem{lemma}[theorem]{Lemma}
\newtheorem{corollary}[theorem]{Corollary}
\theoremstyle{definition}

\theoremstyle{remark}
\newtheorem{remark}[theorem]{Remark}

\icmltitlerunning{	
Anchored Decoding: Provably Reducing Copyright Risk for Any Language Model}

\begin{document}

\twocolumn[
\icmltitle{\texorpdfstring{Anchored Decoding: \\
Provably Reducing Copyright Risk for Any Language Model}{Anchored Decoding: Provably Reducing Copyright Risk for Any Language Model}}

\icmlsetsymbol{equal}{*}

\begin{icmlauthorlist}
\icmlauthor{Jacqueline He}{uw}
\icmlauthor{Jonathan Hayase}{uw}
\icmlauthor{Wen-tau Yih}{uw}
\icmlauthor{Sewoong Oh}{uw}
\icmlauthor{Luke Zettlemoyer}{uw}
\icmlauthor{Pang Wei Koh}{uw,ai2}
\end{icmlauthorlist}

\icmlaffiliation{uw}{University of Washington}
\icmlaffiliation{ai2}{Allen Institute for Artificial Intelligence}

\icmlcorrespondingauthor{Jacqueline He}{jyyh@cs.washington.edu}

\icmlkeywords{copyright, LLMs, decoding}

\vskip 0.3in
]

\printAffiliationsAndNotice{}  %

\begin{abstract}
Language models (LMs) tend to memorize portions of their training data and emit verbatim spans. When the underlying sources are sensitive or copyright-protected, such reproduction raises issues of consent and compensation for creators and compliance risks for developers. 
We propose \ours{}, a plug-and-play inference-time method for suppressing verbatim copying: it enables decoding from any \emph{risky} LM trained on mixed-license data by keeping generation in bounded proximity to a permissively trained \emph{safe} LM. 
\ours{} adaptively allocates a user-chosen information budget over the generation trajectory and enforces per-step constraints that yield a sequence-level guarantee, enabling a tunable risk–utility trade-off.
To make \ours{} practically useful, we introduce a new permissively trained safe model (\ourlm{}), as well as \oursbs{}, a byte-level variant of our method that enables cross-vocabulary fusion via the ByteSampler framework~\citep{hayase2025samplinglanguagemodelbyte}.
Across six model pairs on long-form metrics for copying risk and utility, 
\textsc{Anchored} and \oursbs{} define a new Pareto frontier, preserving near-original fluency and factuality while closing up to 75\% of the measurable copying gap between the risky baseline and a safe reference, at a modest inference overhead.
\end{abstract}

\section{Introduction}

\begin{figure*}[tb]
  \centering
  \subfigure[\textbf{Example generation with \ours{}.}]{%
    \includegraphics[width=0.48\textwidth]{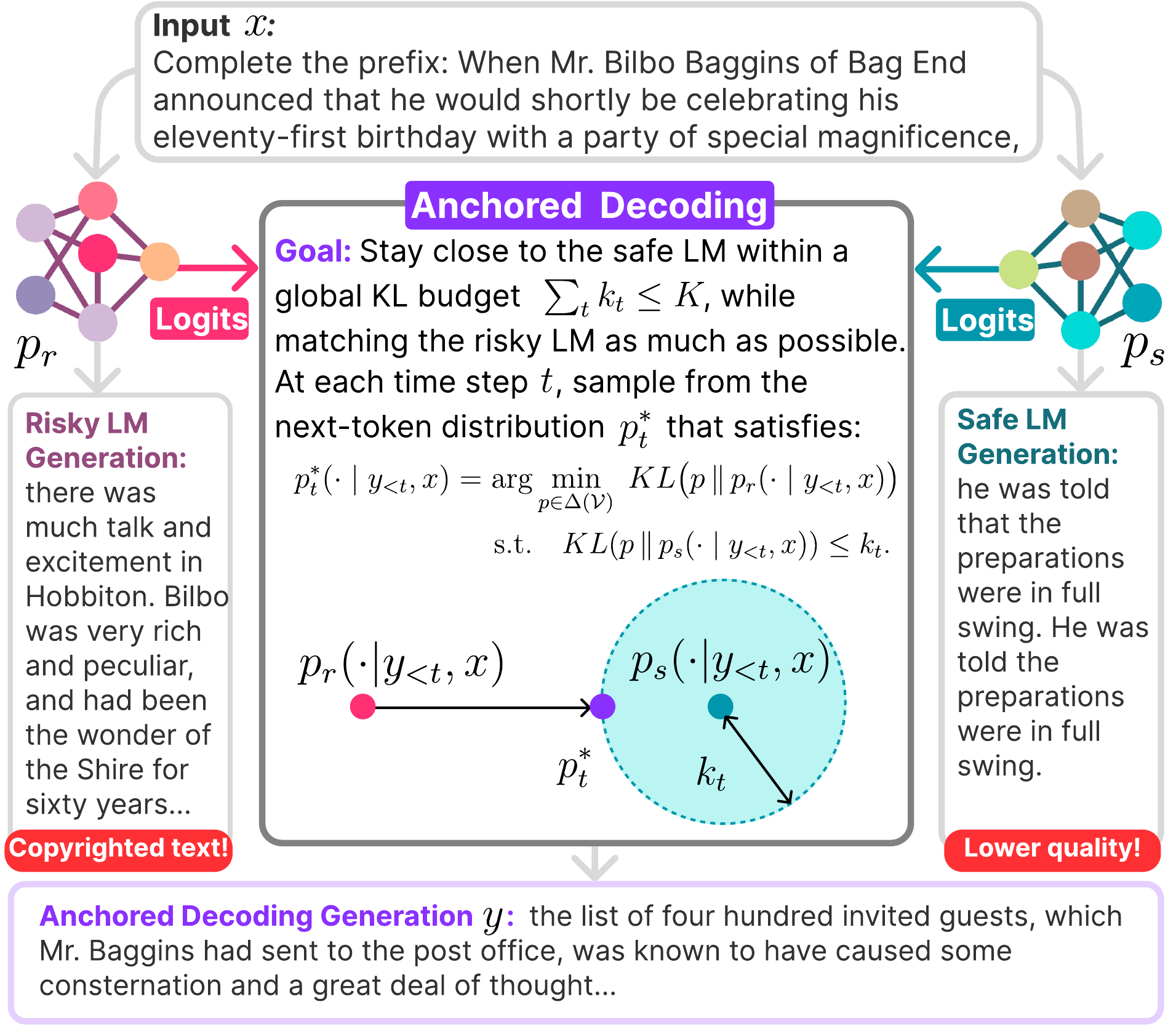}
    \label{fig:teaser_figure}
  }\hfill
  \subfigure[\textbf{Token-level risk-utility tradeoffs.}]{%
    \includegraphics[width=0.48\textwidth]{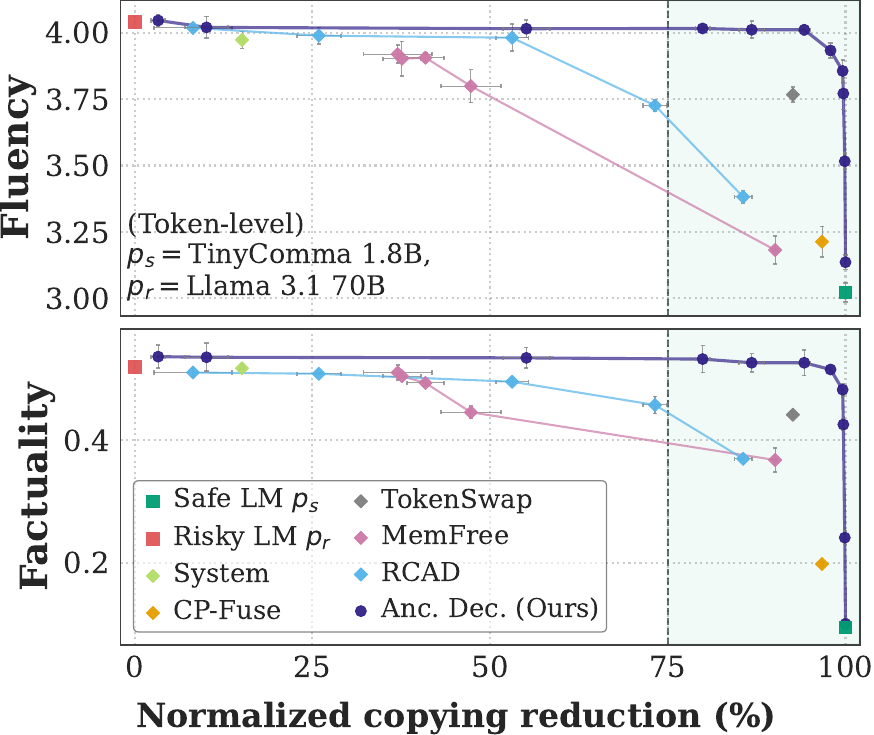}
    \label{fig:teaser_token_tradeoffs}
  }
\vskip -0.06in
\caption{\textbf{(a).} Given the opening line of J.R.R.~Tolkien's \emph{The Fellowship of the Ring} (1954), the \emph{risky} LM outputs its verbatim continuation, while the \emph{safe} LM produces a less fluent, repetitive alternative. \ours{} generates in bounded proximity to the safe LM within a budget $K$, while leveraging utility from the risky LM, and produces a plausible, non-infringing continuation. \textbf{(b).} With the safe-risky LM pair \{\ourlm{}, Llama 3.1 70B\}, \ours{} (in \textbf{\textcolor[HTML]{332288}{purple}}) achieves the best risk-utility tradeoff.} 
  \label{fig:teaser_all}
\end{figure*}

The remarkable capabilities of modern language models (LMs) are fundamentally tied to the scale and diversity of their pre-training data. These corpora are often harvested from the open web with minimal filtering, and may contain sensitive or copyright-protected documents~\citep{kandpal2025the}. 
LMs are able to memorize spans seen during pre-training and later emit them at inference time~\citep{carlini2021extracting, carlini2023quantifying, karamolegkou-etal-2023-copyright}. Such reproduction is often unauthorized and uncompensated with respect to the original creators~\citep{kandpal2025positionexpensivellmtraining}, and may expose developers to legal and privacy liabilities~\citep{henderson2023foundation}. Remediation is difficult, as filtering data for sensitive documents and re-training frontier models is prohibitively expensive. Further, as copyrighted material is usually higher quality, its wholesale exclusion can significantly compromise end performance~\citep{meeus2024did}. 

We address this problem with \ours{}, a practical inference-time method that interpolates between the next-token distributions of a \emph{safe} model and a higher-utility \emph{risky} model. Formally, a safe model is trained exclusively on public domain and openly licensed text~\citep{OpenDefinition21}. While this ecosystem is still nascent, it is expanding with the recent release of permissively licensed pre-training corpora~\citep{min2024silo,bommarito2025kl3mdataprojectcopyrightclean,kandpal2025the,langlais2025commoncorpuslargestcollection}. Conversely, a risky language model is trained on mixed-license sources that may include copyrighted or otherwise sensitive documents; almost all open-weight~\citep{grattafiori2024llama3herdmodels, qwen2.5, gemmateam2025gemma3technicalreport, meta2025llama4} and closed-weight~\citep{claude3family, openai2024gpt4technicalreport, geminiteam2025geminifamilyhighlycapable} LMs belong in this category. Safe LMs offer stronger compliance at the cost of performance~\citep{kandpal2025the}, whereas risky LMs offer greater utility but a higher risk of regurgitation. \cref{fig:teaser_figure} illustrates this trade-off: given the start of a popular novel, the risky LM generates its verbatim continuation, while the safe LM avoids reproduction but yields a qualitatively less fluent output. 

At each decoding step, \ours{} computes a closed-form fusion of the risky and safe next-token distributions by choosing interpolation weights that satisfy a local divergence budget relative to the safe model. We show these local constraints compose into a sequence-level guarantee that provably satisfies the $K$-NAF criterion~\citep{vyas2023provablecopyrightprotectiongenerative}, a mathematical measure that enforces a user-chosen budget $K$ on divergence from the generated distribution to the safe model, thereby providing a principled control on copyright leakage risk. We further introduce two empirically effective adjustments that account for the non-uniform nature of distributional risk across generation: first, a one-time, prompt-dependent \emph{prefix debt} that reduces the initial budget based on how strongly the input prompt appears to have already been memorized by the risky LM, and second, an \textit{adaptive banking rule} that saves unspent budget from low-risk steps for occasional high-risk spikes later in generation. 

Returning to \cref{fig:teaser_figure}, \ours{} produces a continuation that avoids reproduction while remaining fluent and plausible. Our decoding strategy is training-free, provides a user-controllable tradeoff knob with a formal bound, and retrofits to any off-the-shelf LM with exposed logits. \ours{} holds less restrictive assumptions compared to prior two-model copyright mitigation approaches: it does not require a hand-curated list of seed words (\textsc{TokenSwap}~\citep{prashant2025tokenswap}), nor does it require the stringent model-pair construction assumed by \textsc{CP-Fuse} (e.g., disjoint training shards), beyond access to a permissively trained safe anchor ~\citep{abad2025copyrightprotected}.  

We next turn from method-level assumptions to a more fundamental constraint. Many two-model fusion methods, including \ours{}, assume a shared vocabulary. This requirement is quite restrictive for copyright mitigation, as state-of-the-art safe models (e.g., Comma 7B~\citep{kandpal2025the}) employ bespoke tokenization schemes that render direct token-level fusion infeasible. We address this bottleneck via two contributions: first, to enable token-level vocabulary alignment, we release \ourlm{}, a compact variant of the Comma architecture, that is pre-trained exclusively on 169.5B tokens of openly licensed data from the Common Pile~\citep{kandpal2025the}. By design, \ourlm{} adopts the Llama 3.1 tokenizer~\citep{grattafiori2024llama3herdmodels}, facilitating direct compatibility with the Llama 3.1 family. Second, we introduce \oursbs{}, a byte-level, $K$-NAF compliant analogue of \ours{}. \oursbs{} builds upon the \bs{} framework~\citep{hayase2025samplinglanguagemodelbyte} and bypasses tokenizer mismatch by operating on the next-byte distribution. Together, these advances support a significantly broader range of model pairs than previously possible. 

When evaluated in realistic long-form settings, \textsc{Anchored} and \oursbs{} consistently achieve Pareto-optimality against strong mitigation baselines across six model pairs. They attain near-original fluency and factuality while closing up to 75\% of the reduction between the risky baseline and safe reference. 
For example, we show risk-utility tradeoff plots in \cref{fig:teaser_token_tradeoffs} for the model pair \{\ourlm{}, Llama 3.1 70B\} (decoding at the token level). \ours{} adds modest inference overhead (e.g., only 1.1x in this case) by pairing risky LMs with a much smaller safe LM, making the second forward pass relatively cheap without sacrificing effectiveness. 
We further analyze why \ours{} works: our key intuition is that the per-step KL divergence between the risky and safe models is a useful signal for identifying when generation has entered a copyright-sensitive regime.
Beyond copyright mitigation in LMs, \ours{} is agnostic to tokenizer, modality, and domain. Our framework serves as a general-purpose tool that is applicable wherever a high-capability generative process must be rigorously bounded by a trusted reference distribution.

\section{Preliminaries}
We consider token-level autoregressive language models $p$ that operate over a fixed vocabulary $\mathcal{V}$. 
Given a prompt sequence $x$, define a probability distribution over variable-length token sequences $y_{0:T-1}$ with $y_t \in \mathcal{V}$ as
\begin{align}
p(y_{0:T-1} \mid x) = \prod_{t=0}^{T-1} p(y_t \mid y_{<t}, x) \quad\text{for }T \leq T_{\max}, \label{eq:lm_distribution}
\end{align}
where $y_{T-1} = \textrm{EOS}$ denotes termination.
We assume access to a \emph{safe} model $p_s$ trained only on permissively licensed text, and a more performant \emph{risky} model $p_r$ that may reproduce copyrighted data due to its broader, unchecked training set. 
$p_r$ and $p_s$ are assumed to induce the same support over $\mathcal{V}$, which are typically satisfied by standard softmax parameterizations.
Our goal is to find some target distribution that balances the competing objectives of copyright safety (from $p_s$) and utility (from $p_r$). 

\paragraph{$K$-Near Access-Freeness ($K$-NAF).} First, we formally define our desired copyright safety condition. We adopt the $K$-NAF framework introduced by \citet{vyas2023provablecopyrightprotectiongenerative}, which bounds the total divergence over the entire distribution of sequences generated by a model $p$ relative to a safe model $p_s$: 

\begin{restatable}{definition}{GlobalKNAF}[Global $K$-NAF] \label{def:global_knaf}
Let $K\geq0$ be a global information budget. Formally, a model $p$ satisfies the global $K$-NAF guarantee (relative to $p_s$) if, for every input sequence $x$ and every $T \leq T_{\max}$,
\begin{align}
\mathcal{D}\!\left(p(y_{<T}\mid x)\,\big\|\,p_s(y_{<T}\mid x)\right) \le K,
\end{align}
where $y_{<T}=(y_0,\dots,y_{T-1})$, and $\mathcal{D}$ is some arbitrary divergence function. 
We primarily consider $\mathcal{D} = \kl$, the Kullback-Leibler (KL) divergence~\citep{Kullback51klDivergence}.\footnote{For generality, we show that the guarantees to our method also hold when setting $\mathcal{D}$ as the \ry divergence in \cref{subsec:renyi_divergence}, and show experimental ablations in \cref{subsec:ablations}.}
Throughout, we write $p(\cdot\mid x)$ as shorthand for the autoregressive model $p$ and its induced finite-horizon prefix distributions $\{p(y_{<T}\mid x)\}_{T\le T_{\max}}$.
\end{restatable}

\section{\ours{}}
\subsection{A Tractable Token-Level Approximation} 
Our goal is to construct a new distribution $p^*$ that approximates $p_r$ to maximize utility, while remaining strictly subject to the safety constraint relative to $p_s$ (\cref{def:global_knaf}). This yields the constrained sequence-level optimization:
\begin{align}
\label{eq:global_opt}
p^* = \argmin_{p} \quad & \kl \big(p(\cdot \mid x) \kld p_r(\cdot \mid x)\big) \\
\text{s.t.} \quad & \kl \big(p(\cdot \mid x) \kld p_s(\cdot \mid x)\big) \le K.\nonumber
\end{align}
As \cref{eq:global_opt} is defined at the sequence level, solving this directly is computationally intractable, as it becomes a search problem over the exponentially large $\mathcal{V}^{T_{\max}}$.

Therefore, similar to \citet{abad2025copyrightprotected}, we approximate the global objective by decomposing the problem into a series of local (per-token) objectives.
Concretely, rather than solving for an optimal sequence-level spending schedule subject to a global budget $K$, we impose a \emph{per-step} budget $k_t$ at each timestep $t$.
The sequence-level constraint in \cref{eq:global_opt} can be decomposed via the chain rule of KL divergence into a sum of per-step conditional divergences. Accordingly, at each $t$ and for any prefix $y_{<t} \sim p^*(\cdot|x)$ generated thus far, we solve a local constrained problem to obtain an optimal next-token distribution, $p_t^*(\cdot\mid y_{<t},x)$, that stays within the per-step budget $k_t$ relative to $p_s$, while drawing close to $p_r$:
\begin{align}
\label{eq:token_level_opt}
p_t^*(\cdot\mid y_{<t},x)
= &\arg\min_{p \in \Delta(\mathcal V)}\; \kl\!\big(p \kld p_r(\cdot \mid y_{<t}, x)\big) \\
&\text{s.t.}\quad \kl\!\big(p \kld p_s(\cdot \mid y_{<t}, x)\big) \le k_t, \notag \\
& \sum_{y \in \mathcal{V}} p(y)=1,\quad p(y)>0 \ \ \forall y\in\mathcal{V}. \notag
\end{align}
By the chain rule for KL, any continuation generated autoregressively by $\{p_t^*(\cdot|y_{<t},x)\}_{t<T_{\max}}$ of length $T\leq T_{\max}$ satisfies \cref{def:global_knaf}, as the following theorem shows: 
\begin{restatable}[Safety of local approximation]{theorem}{SafetyLocalRelaxation}
\label{thm:safety_guarantee}
Let $p^*$ be a sequence-level distribution defined autoregressively by $p^*(y_{<T}|x) = \prod_{t=0}^{T-1} p_t^*(y_t \vert y_{<t}, x)$. If, for all decoding steps $t<T_{\max}$, the conditional distribution $p_t^*$ solves \cref{eq:token_level_opt} with a per-step budget $k_t$ such that $\sum_{t=0}^{T_{\max}-1} k_t \le K$, then $p^*$ satisfies the global $K$-NAF guarantee in \cref{eq:global_opt} for all $T \leq T_{\max}$.
\end{restatable}
In effect, our approximation yields a valid solution to the original optimization. We further note a simple corollary:
\begin{corollary}[Constant per-step cap]
\label{cor:constant_cap}
Setting $k_t\equiv k$ for all $t<T_{\max}$ satisfies the condition of \cref{thm:safety_guarantee} whenever $k\,T_{\max}\le K$.
\end{corollary}
\begin{algorithm}[tb]
\small
\caption{\textsc{AnchoredDecode}($p_r$, $p_s$, $K$, $T_{\max}$, $n$, $\mathcal{S}$)}
\label{alg:tok_kl_copyright_decoding}
\begin{algorithmic}[1]
\STATE {\bfseries Input:} risky LM $p_r$, safe LM $p_s$, global budget $K$, max length $T_{\max}$, prompt $x$, debt window $n$, special tokens $\mathcal{S}$.
\STATE {\bfseries Output:} generation $y=(y_0,\dots,y_t)$, where $t<T_{\max}$.
\STATE {\bfseries Init:} $\debt \gets \textsc{PrefixDebt}(p_r,p_s,x,n,\mathcal{S})$; \COMMENT{\cref{alg:prefix_debt}}
\STATE {\bfseries Init:} local cap $k \gets K/T_{\max}$; cumulative expenditure $A_0 \gets \debt$; history $y_{<1}\gets \emptyset$;
\FOR{$t=0$ {\bfseries to} $T_{\max}-1$}
  \STATE  {\bfseries Compute:} $p_r(\cdot\mid y_{<t},x)$ and $p_s(\cdot\mid y_{<t},x)$.
  \STATE {\bfseries Accrue budget:} $k_t \gets \max\!\bigl(0,\;k\,(t{+}1)-A_t\bigr)$.
  \STATE {\bfseries Project:} $p_t^* \gets \textsc{ProjectKL}\!\big(p_r; p_s, k_t\big)$. \COMMENT{\cref{alg:project_kl}}
  \STATE {\bfseries Sample:} $y_t \sim p_t^*(\cdot\mid y_{<t},x)$.
  \STATE {\bfseries Bank realized spend:} $A_{t+1} \gets A_t + \kl\!\big(p_t^* \kld p_s\big)$.
  \IF{$y_t=\text{EOS}$} \STATE \bfseries break \ENDIF
\ENDFOR
\STATE {\bfseries Return} $y=(y_0,\dots,y_t)$.
\end{algorithmic}
\end{algorithm}

\subsection{Solving for a fused distribution $p^*_t$}
\label{subsec:fused_distribution}
\cref{eq:token_level_opt} admits a closed-form solution that can be efficiently computed at each decoding step:
\begin{restatable}[Solving for $p_t^*$]{proposition}{ClosedFormFusion}
{For a given local budget $k_t$ at decoding step $t$, the optimal distribution $p_t^*$ that solves \cref{eq:token_level_opt} is a weighted geometric mean:
\begin{align} \label{eq:closed_form}
p_t^* = \frac{1}{Z} p_s(\cdot \mid y_{<t}, x)^{\frac{\lambda}{1+\lambda}} p_r(\cdot \mid y_{<t}, x)^{\frac{1}{1+\lambda}},
\end{align}
where $Z$ is a normalization constant and $\lambda \geq 0$ is the dual variable (Lagrange multiplier) associated with the KL constraint.}
\end{restatable}

In practice, determining the optimal $\lambda$ (equivalently, the mixing weight $\frac{\lambda}{1+\lambda}$) reduces to a 1D root-finding problem. When the constraint is active, we solve for $\lambda \geq 0$ such that $f(\lambda) := \kl \bigl(p_t^*(\lambda) \kld p_{s,t} \bigr) - k_t = 0$,
using a safeguarded Newton-Raphson algorithm to ensure fast convergence to a feasible solution~\citep{ypma1995nr}.\footnote{The optimization procedure is shown in \cref{alg:project_kl}.}
We further propose two empirically effective improvements to budget allocation that still respect the global budget $K$: a one-time prefix-debt offset $\debt$, and an adaptive rule for defining the stepwise budget $k_t$.

\subsection{Prefix debt $\debt$} 
\label{subsec:prefix_debt}
Not all input prefixes are equally likely to elicit regurgitated continuations. A generic prompt poses little risk, 
whereas the opening hook of a famous novel may statistically prime the model to output memorized text. 
We propose to exploit this early signal by offsetting the global $K$ budget with some context-dependent debt, $\debt$, reflecting the intuition that a risky prompt $x$ has already pre-spent some of its safety margin by the start of generation. By initializing the banked budget with a negative offset, the available per-step budget is effectively zero in early steps (i.e., we clamp negative $k_t$ to 0). Thus, our conservative approach is more likely to sample exclusively from $p_s$ in early steps for memorization-triggering prompts.

Inspired by pretraining data detection~\citep{shi2024detecting, zhang2025mink}, we focus on the largest log-likelihood ratios (LLRs) in the prefix. Intuitively, extremely positive LLR outliers mark tokens for which $p_r$ and $p_s$ disagree strongly, suggesting higher memorization likelihood. 
Let $x{=}(x_0, x_1, \ldots, x_{L-1})$ be a prefix sequence of length $L{>}1$.
For each position $i \in \{1,\ldots,L-1\}$, define the pointwise LLR as $\ell_i(x)
:= \log \frac{p_r(x_i \mid x_{<i})}{p_s(x_i \mid x_{<i})}$.
Let $[z]_+ := \max(z,0)$ denote the positive part of $z$, and let $\mathcal{I}_n(x)$ be the index set of the $n$ largest values among $\{[\ell_i(x)]_+\}_{i=1}^{L-1}$
(ties broken arbitrarily; if $L-1{<}n$, take all indices).
The \emph{prefix debt} is
\begin{align}
\debt := \frac{1}{\max\{1,|\mathcal{I}_n(x)|\}}
\sum_{i \in \mathcal{I}_n(x)} [\ell_i(x)]_+ . \label{eq:prefix_debt}
\end{align}
$\debt$ acts as a one-time reduction of the global budget. As $\debt \geq 0$ by construction, $K - \debt \leq K$, ensuring that the safety guarantee is preserved (and typically tightened) regardless of the prompt's content.\footnote{We sketch out the prefix debt calculation in \cref{alg:prefix_debt}.}

\subsection{An adaptive budgeting strategy}
\label{subsec:adaptive_budget}
While setting a constant local cap $k_t \equiv k$ satisfies the global safety condition whenever $k \, T_{\max} \leq K$ (\cref{cor:constant_cap}), this naive allocation is often overly conservative. The constraint $\kl (p_t^* \kld p_s) \leq k$ applies the same allowance at every step and cannot bank unused budget from ``easy" steps (i.e., where the models naturally agree, and $p_t^*$ is already close to $p_s$) for later steps. 
We therefore propose an \emph{adaptive budget} parameterized by a base rate $k$ (the knob we sweep) that tracks realized spend and rolls unused budget forward. For shorthand, denote $p_i^* := p_i^*(\cdot | y_{<i}, x)$ for timestep $i$, and analogously for $p_{s,i}$. 

\begin{restatable}[Global safety of adaptive banking]{proposition}{AdaptiveBankingSafety}
\label{prop:adaptive_banking}
Let $K$ be the global safety budget for a sequence up to length $T_{\max}$, and let $k \coloneqq K/T_{\max}$.
Set $a_i \coloneqq \kl\!\bigl(p_i^* \kld p_{s,i}\bigr)$ (the actual KL expenditure at each step $i$).
If, at each decoding step $t<T_{\max}$, the per-step adaptive budget $k_t$ is defined as
\begin{align}
k_t &\coloneqq \max\left(0, (t+1)k - \sum_{i=0}^{t-1}a_i - \debt\right),
\label{eq:adaptive_budget_defn}
\end{align}
where $\debt \geq 0$ is some initial budget adjustment for the input prefix $x$, then the resulting autoregressive sequence distribution $p^*(y_{<T}|x) = \prod_{t=0}^{T-1} p_t^*(y_t | y_{<t}, x)$ satisfies $\kl (p^* \kld p_s) \leq K - \debt \leq K$ for any $T \leq T_{\max}$.
\end{restatable}

\subsection{Putting \ours~together}
Our complete method (\cref{alg:tok_kl_copyright_decoding}) solves for a fused distribution of the form in \cref{subsec:fused_distribution} at every decoding step. 
We treat $k$ as the user-set nominal per-step allotment; for a horizon $T_{\max}$, we set the corresponding global budget to $K \coloneqq k\,T_{\max}$ (satisfying \cref{cor:constant_cap}).
The prefix debt $\debt$ (\cref{subsec:prefix_debt}) is then applied as a conservative offset, yielding an effective budget $K-\debt$.
This remaining budget is allocated over the course of decoding via the adaptive budgeting rule defined in \cref{subsec:adaptive_budget}. 
Consequently, \ours{} satisfies the global $K$-NAF guarantee for horizon $T_{\max}$, while allowing for a tunable per-step allotment $k$ and debt window $n$.

\subsection{\oursbs}
\ours{} requires $p_r$ and $p_s$ to share the same vocabulary space, which substantially restricts the set of feasible model pairs. Many safe models $p_s$ 
(e.g., Comma 7B) use bespoke tokenizers to ensure permissive data usage throughout the language modeling pipeline, making direct token-level fusion with popular model families (e.g., Llama 3) incompatible. 
This motivates \oursbs{}, a \emph{byte-level} version of our method that supports cross-tokenizer compatibility, and retains the weaker assumption that $p_r$ and $p_s$ have Byte Pair Encoding (BPE)~\citep{gage1994new, sennrich-etal-2016-neural} tokenizers that induce a mapping from tokens to UTF-8 byte strings. 

\paragraph{Defining the byte probability space.} We treat the decoding process as a sequence of byte-level transitions $\mathbf{b} = (b_0, b_1,...,b_{B-1})$ for $B \leq B_{\max}$ using the ByteSampler~\citep{hayase2025samplinglanguagemodelbyte} framework.  At each step $t$, ByteSampler induces a next-byte distribution over $\mathcal{B} = \{0\textrm{x}00, ...,0\textrm{x}FF\}$ by marginalizing the model's token distribution over all valid tokenizations consistent with the current byte prefix $b_{<t}$. The probability of a byte $b$ is obtained by summing the probabilities of valid token continuations whose decoded strings have $b$ as the next byte, where validity is determined by the model's tokenizer state for $\mathbf{b}_{<t}$. \citet{hayase2025samplinglanguagemodelbyte} implement this marginalization using a Valid Covering Tree traversal, which efficiently produces an exact next-byte distribution induced by the underlying token-level model and tokenizer.

\paragraph{\oursbs{} satisfies $K$-NAF.} \oursbs{} solves an analogous optimization to \cref{eq:token_level_opt} under a similar prefix debt calculation to \cref{eq:prefix_debt} and budgeting rule to \cref{eq:adaptive_budget_defn}. The only change when moving to the byte level is that we now operate over the byte space $\mathcal{B}$ instead of $\mathcal{V}$, and use the induced byte-level distributions of $p_r$ and $p_s$. 
Thus, \oursbs{} remains $K$-NAF compliant for $K=kB_{\max}$, where $B_{\max}$ is the maximum number of generated bytes.\footnote{Please see \cref{subsec:knaf_byte} for byte-level decoding details.} 
With modern BPE tokenization, one token corresponds to approximately 4 bytes in English~\citep{tiktoken, pagnoni-etal-2025-byte}. This results in more decoding steps for the same semantic length, so we set $B_{\max} \approx 4 T_{\max}$.

\subsection{Evaluation metrics}
As the determination of copyright infringement is inherently contextual, 
we measure copyright risk and utility using \textit{long-form} metrics that holistically score generations.\footnote{We provide more evaluation details in \cref{app:evaluation}.} 

\paragraph{Copyright infringement.} We consider \emph{real} copyright risk and evaluate on snippets from \textsc{Books}, a domain of 16 novels taken from CopyBench~\citep{chen-etal-2024-copybench} that are currently protected under U.S. copyright law, and are identified as likely to have been memorized by LLMs~\citep{chang2023speakmemoryarchaeologybooks, shi2024detecting}. Following~\citet{chen-etal-2024-copybench, wei2024evaluating}, we employ the following six metrics (denoted as  $\mathcal{M}$): ROUGE-1 and ROUGE-L~\citep{lin-2004-rouge} above a set threshold $\tau=0.4$, MinHash similarity~\citep{broder1997resemblance}, and word-level Accumulated Common Substring (ACS) measure \emph{near duplicate} copying, while word-level and character-level Longest Common Substring (LCS) measure the extent of \emph{exact match}.\footnote{We also provide \emph{non-literal} copying experiments in \cref{subsec:nonliteral}.} 
\section{Experiments}
We assign equal weight to each metric $m \in \mathcal{M}$ and aggregate them into a single \emph{normalized copying reduction} (NCR) score: formally, let $m_r$ and $m_s$ denote the metric values for $p_r$ and $p_s$, respectively. For a given setting with metric value $m$, we define its NCR to be $\frac{m_r - m}{m_r - m_s}$.

NCR quantifies the fraction of the performance gap between $p_r$ and $p_s$ that is closed by the setting. Crucially, we treat $p_s$ as the gold standard for safety rather than aiming for a metric value of zero. Since any non-infringing natural language model may exhibit non-zero baseline overlap due to common linguistic structures, $p_s$ approximates
the lower bound of incidental copying achievable without sacrificing fluency. We report the final aggregate result as the average of NCR values across all 6 metrics in $\mathcal{M}$. 

Finally, we define a \emph{high-protection} regime as a threshold where the NCR is at least $75\%$. While the optimal safety threshold is often task-specific, practical deployment typically requires a substantial reduction in copying risk, not marginal improvements. We thus choose this point to isolate regimes where the mitigation effect size is dominant. 

\paragraph{Utility.}

We consider two types of general utility: \emph{fluency} and \emph{factuality}. Fluency measures how natural and well-formed the output reads as; we follow \citet{chen-etal-2024-copybench} and evaluate the quality of \textsc{Books} continuations using Prometheus-v2~\citep{kim2024prometheus}, an LLM-as-a-judge~\citep{zheng2023judging} that scores output along a five-point rubric (5 being the highest). Factuality is a property that should be preserved, as U.S. copyright law only protects the original expression of factual knowledge, but not facts themselves~\citep{Feist1991}.
\begin{figure*}[t!]
\begin{center}
  \centerline{\includegraphics[width=\linewidth]{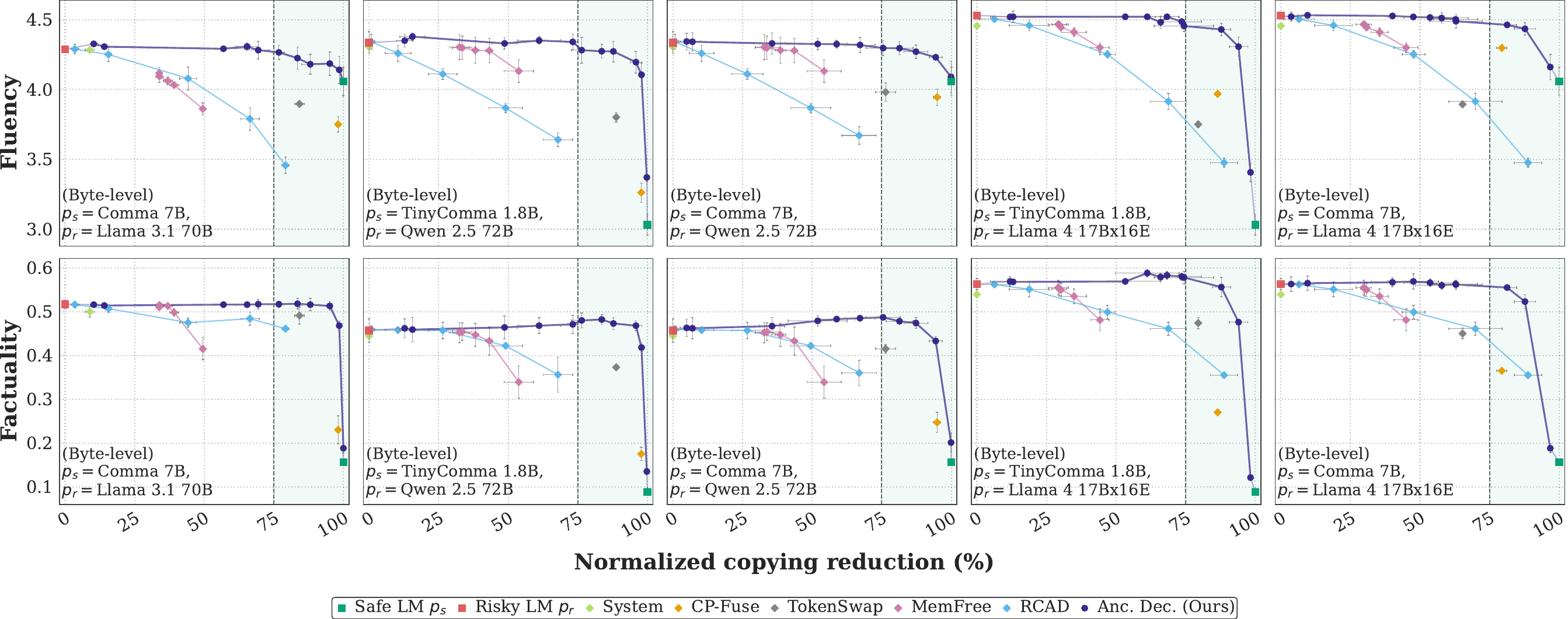}}
  \caption{
  \textbf{\oursbs{} (in \textbf{\textcolor[HTML]{332288}{purple}}) achieves the best risk-utility tradeoff at the byte level across five model pairs.} We report the average of three seeds; error bars show standard deviation. 
  The shaded threshold denotes the \textbf{high-protection operating point}, where the Normalized Copyright Reduction (NCR)$\geq 75\%$. NCR and fluency are evaluated on \textsc{Books}, and factuality on \textsc{Bios}.
  }
  \label{fig:tradeoffs_byte}
\end{center}
\vskip -0.3in
\end{figure*}
Consequently, we evaluate long-form factuality on \textsc{Bios}, a collection of biography generation prompts used by \fs~\citep{min-etal-2023-factscore}. \fs{} is a fine-grained metric that decomposes each output into a set of atomic, verifiable claims~\citep{song-etal-2024-veriscore}, and fact-checks each claim independently against retrieved web search snippets. We report the average claim precision---the fraction of unique supported claims.

\subsection{Inference-time baselines}
We summarize the mitigation baselines used in our experiments. We defer a broader discussion to \cref{sec:related_work}, and provide implementation details in \cref{app:baseline_details}.
\paragraph{Single-model decoding baselines.} In \textsc{System}, we prepend a system prompt that instructs LMs to refrain from outputting copyrighted material~\citep{chen-etal-2024-copybench,dbrx2024, wei2024evaluating, aerni2025measuring}. \textsc{MemFree}~\citep{ippolito-etal-2023-preventing} is a decoding method that blocks exact $n$-gram regurgitation by rejecting any next token that would complete an $n$-gram from a reference-derived blocklist; we sweep  $n \in \{3, 5, 7, 9, 10\}$. Reversed Context Aware Decoding~\citep{wei2024evaluating}, or \textsc{RCAD}, contrasts logits with and without a blocklisted context and produces a next-token distribution that downweights tokens favored by context (modulated using $\alpha$); we sweep $\alpha \in \{0.1,0.25,0.5,0.75,1.0\}$. We apply these baselines to $p_r$, and  
for \textsc{MemFree} and \textsc{RCAD}, we simulate realistic deployment and construct the context blocklist by retrieving the top-1 passage from an in-domain datastore (e.g., Books3 for \textsc{Books}, and Wikipedia for \textsc{Bios}).
\begin{table*}[t]
\centering
\small
\caption{\textbf{High-protection operating point (normalized copying reduction $\ge 75\%$).}
We report the best utility (\textsc{Factuality} / \textsc{Fluency}) among settings that achieve NCR$\ge 75\%$. We show average and standard deviation values over three seeds. Entries are ``—’’ if the method does not reach the threshold. \textsuperscript{\dag}: Token-level decoding; all other model pairs use byte-level decoding. 
}
\begin{adjustbox}{max width=\linewidth,center}
\setlength{\tabcolsep}{4.5pt}
\renewcommand{\arraystretch}{1.15}
\begin{tabular}{lcccccc}
\toprule
\multirow{3}{*}{\textbf{Method}} &
\multicolumn{1}{c}{\bf Factuality / Fluency} &
\multicolumn{1}{c}{\bf Factuality / Fluency} &
\multicolumn{1}{c}{\bf Factuality / Fluency} &
\multicolumn{1}{c}{\bf Factuality / Fluency} &
\multicolumn{1}{c}{\bf Factuality / Fluency} &
\multicolumn{1}{c}{\bf Factuality / Fluency} \\
&
\multicolumn{1}{c}{$p_s$: \textsc{\ourlm{}}\textsuperscript{\dag}} &
\multicolumn{1}{c}{$p_s$: \textsc{Comma 7B}} &
\multicolumn{1}{c}{$p_s$: \textsc{\ourlm{}}} &
\multicolumn{1}{c}{$p_s$: \textsc{Comma 7B}} &
\multicolumn{1}{c}{$p_s$: \textsc{\ourlm{}}} &
\multicolumn{1}{c}{$p_s$: \textsc{Comma 7B}} \\
&
\multicolumn{1}{c}{$p_r$: \textsc{Llama 3.1 70B}\textsuperscript{\dag}} &
\multicolumn{1}{c}{$p_r$: \textsc{Llama 3.1 70B}} &
\multicolumn{1}{c}{$p_r$: \textsc{Qwen 2.5 72B}} &
\multicolumn{1}{c}{$p_r$: \textsc{Qwen 2.5 72B}} &
\multicolumn{1}{c}{$p_r$: \textsc{Llama 4 17B$\!\times\!$16E}} &
\multicolumn{1}{c}{$p_r$: \textsc{Llama 4 17B$\!\times\!$16E}} \\
\midrule

\multicolumn{7}{l}{\textit{\bf Vanilla}}\\ \midrule
\textsc{Safe} &
$0.09_{\scriptscriptstyle 0.01}$ / $3.0_{\scriptscriptstyle 0.04}$ &
$0.16_{\scriptscriptstyle 0.01}$ / $4.1_{\scriptscriptstyle 0.10}$ &
$0.09_{\scriptscriptstyle 0.00}$ / $3.0_{\scriptscriptstyle 0.07}$ &
$0.16_{\scriptscriptstyle 0.01}$ / $4.1_{\scriptscriptstyle 0.10}$ &
$0.09_{\scriptscriptstyle 0.00}$ / $3.0_{\scriptscriptstyle 0.07}$ &
$0.16_{\scriptscriptstyle 0.01}$ / $4.1_{\scriptscriptstyle 0.10}$ \\
\textsc{Risky} & --- & --- & --- & --- & --- & --- \\
\midrule

\multicolumn{7}{l}{\textit{\bf Single-model baselines}}\\ \midrule
\textsc{System}  & --- & --- & --- & --- & --- & --- \\
\textsc{MemFree} & $0.37_{\scriptscriptstyle 0.02}$ / $3.18_{\scriptscriptstyle 0.05}$ & --- & --- & --- & --- & --- \\
\textsc{RCAD}    &
$0.37_{\scriptscriptstyle 0.00}$ / $3.38_{\scriptscriptstyle 0.02}$ &
$0.46_{\scriptscriptstyle 0.00}$ / $3.46_{\scriptscriptstyle 0.06}$ &
--- &
--- &
$0.36_{\scriptscriptstyle 0.00}$ / $3.48_{\scriptscriptstyle 0.03}$ &
$0.36_{\scriptscriptstyle 0.00}$ / $3.48_{\scriptscriptstyle 0.03}$ \\
\midrule

\multicolumn{7}{l}{\textit{\bf Two-model baselines}}\\ \midrule
\textsc{CP-Fuse} &
$0.20_{\scriptscriptstyle 0.00}$ / $3.21_{\scriptscriptstyle 0.06}$ &
$0.23_{\scriptscriptstyle 0.03}$ / $3.75_{\scriptscriptstyle 0.06}$ &
$0.18_{\scriptscriptstyle 0.02}$ / $3.26_{\scriptscriptstyle 0.07}$ &
$0.25_{\scriptscriptstyle 0.02}$ / $3.95_{\scriptscriptstyle 0.06}$ &
$0.27_{\scriptscriptstyle 0.00}$ / $3.97_{\scriptscriptstyle 0.02}$ &
$0.37_{\scriptscriptstyle 0.00}$ / $4.30_{\scriptscriptstyle 0.02}$ \\
\textsc{TokenSwap} &
$0.44_{\scriptscriptstyle 0.00}$ / $3.77_{\scriptscriptstyle 0.03}$ &
$0.49_{\scriptscriptstyle 0.02}$ / $3.90_{\scriptscriptstyle 0.00}$ &
$0.37_{\scriptscriptstyle 0.00}$ / $3.80_{\scriptscriptstyle 0.03}$ &
$0.42_{\scriptscriptstyle 0.01}$ / $3.98_{\scriptscriptstyle 0.07}$ &
$0.47_{\scriptscriptstyle 0.01}$ / $3.75_{\scriptscriptstyle 0.02}$ &
--- \\
\midrule

\multicolumn{7}{l}{\bf Our method}\\ \midrule
 \textsc{Anchored}/\textsc{Anchored}\textsubscript{Byte}
&
$\mathbf{0.53}_{\scriptscriptstyle 0.02}$ / $\mathbf{4.02}_{\scriptscriptstyle 0.01}$ &
$\mathbf{0.52}_{\scriptscriptstyle 0.01}$ / $\mathbf{4.23}_{\scriptscriptstyle 0.08}$ &
$\mathbf{0.48}_{\scriptscriptstyle 0.01}$ / $\mathbf{4.27}_{\scriptscriptstyle 0.05}$ &
$\mathbf{0.49}_{\scriptscriptstyle 0.00}$ / $\mathbf{4.30}_{\scriptscriptstyle 0.05}$ &
$\mathbf{0.56}_{\scriptscriptstyle 0.02}$ / $\mathbf{4.43}_{\scriptscriptstyle 0.04}$ &
$\mathbf{0.56}_{\scriptscriptstyle 0.01}$ / $\mathbf{4.46}_{\scriptscriptstyle 0.02}$ \\
\bottomrule
\end{tabular}
\end{adjustbox}
\label{tab:high_protection_utility}
\end{table*}

\paragraph{Joint-model decoding baselines.}
We consider \textsc{CP-Fuse}~\citep{abad2025copyrightprotected}, a $K$-NAF-inspired fusion method that selects a next-token distribution by balancing proximity to two LMs of equal utility. 
CP-Fuse assumes the copyrighted portions of the training data can be cleanly separated across the two models (i.e., disjoint data shards). It solves for a per-step fused distribution by minimizing the maximum KL divergence across the model pair. 
\textsc{TokenSwap}~\citep{prashant2025tokenswap} constructs a hybrid next-token distribution by swapping a manually defined set of common tokens (e.g., function words) from a small model onto a large model’s distribution, while leaving all other token probabilities unchanged. For baseline consistency, we instantiate \textsc{CP-Fuse} and \textsc{TokenSwap} using our asymmetric pair $(p_r, p_s)$.

\subsection{Model pairs}
Following prior memorization work~\citep{carlini2021extracting,carlini2023quantifying,abad2025copyrightprotected,prashant2025tokenswap},
we instantiate $p_s$ and $p_r$ as base (non-instruction-tuned) LMs to isolate memorization effects in the underlying next-token distributions. For risky models $p_r$, we choose Llama 3.1 70B~\citep{grattafiori2024llama3herdmodels}, Qwen 2.5 72B~\citep{qwen2.5}, and Llama 4 Scout 17Bx16E~\citep{meta2025llama4}, which exhibit measurable verbatim reproduction in our copying evaluations. We select capable safe models $p_s$ that are trained on the Common Pile~\citep{kandpal2025the}. To ensure tokenizer compatibility with Llama 3.1, we pre-train our own \ourlm{} on the 169.5B tokens from the Common Pile, which outperforms other size-matched $p_s$~\citep{min2024silo, bommarito2025kl3mdataprojectcopyrightclean, langlais2025commoncorpuslargestcollection} on general language understanding tasks.\footnote{In \cref{subsec:comma1.7b}, we provide \ourlm{} pre-training details and show benchmarking results for various $p_s$ candidates on standard natural language tasks.} We also use the larger and more performant \textsc{Comma}~7B~\citep{kandpal2025the}, which was trained for 2T tokens and has a custom tokenizer. Our selection of $p_s, p_r$ leads to six model pairs, of which only \ourlm{} and Llama 3.1 70B are tokenizer-matched. We apply \ours{} to that pair, and \oursbs{} to all others.

\subsection{Hyperparameters} We sweep our methods across a range of $k$: $\{0.1, 0.5, 1, 1.5, 2, 3, 4, 5, 10, 15, 20\}$. We set $T_{\max}{=}200$ and $B_{\max}{=}800$ bytes, and $n{=}5$ as the debt window.\footnote{We find \ours{} to be largely insensitive to the choice of $n$ and provide more details in \cref{app:prefix_debt}.}

\section{Results}
\subsection{Risk--utility trade-offs}
\cref{fig:teaser_token_tradeoffs} and \cref{fig:tradeoffs_byte} show trade-off curves among mitigation baselines for (token-level) \textsc{Anchored} and (byte-level) \oursbs{}, respectively. \footnote{\cref{subsec:qual_examples} shows qualitative examples.} The results are striking: at both granularities, our methods trace the Pareto frontier (upper-right zone) across all model pairs, with statistical significance over 3 random seeds. 
\cref{tab:high_protection_utility} highlights the best utility achieved by each method among configurations that meet the high-protection operating point (i.e., NCR${\geq}75\%$). Our methods consistently yield the strongest utility within this region. For example, under the token-level pair \{\ourlm{}, Llama 3.1 70B\}, \ours{} achieves \textbf{0.53} in factuality and \textbf{4.02} in fluency, surpassing the strongest two-model baseline that meets the threshold (\textsc{TokenSwap}: 0.44 / 3.77) and the single-model baselines that do (e.g., \textsc{RCAD}: 0.37 / 3.38; \textsc{MemFree}: 0.37 / 3.18). Similar trends persist at the byte level: for \{Comma 7B, Llama 3.1 70B\}, \oursbs{} achieves \textbf{0.52} / \textbf{4.23}, exceeding \textsc{CP-Fuse} (0.23 / 3.75), \textsc{TokenSwap} (0.49 / 3.89), and \textsc{RCAD} (0.46 / 3.46). 

Among pointwise baselines, \textsc{System} shows high utility but scarcely achieves copying reduction (and for $p_r{=}$Qwen 2.5 72B, is even slightly worse than $p_r$). While \textsc{CP-Fuse} and \textsc{TokenSwap} achieve high NCR and mostly fall within the high-protection operating point, they also experience worse utility. Among parametric baselines, both \textsc{RCAD} and \textsc{MemFree} tend to operate well below the high-protection operating point; 
even when they do surpass the threshold, it comes at a substantial utility cost.
\subsection{\ours{} ablations}
\label{subsec:ablations}
We choose the LM pair \{\ourlm{}, Llama 3.1 70B\} to study three ablation axes at the token-level: (i) optimization objective, (ii) prefix debt, and (iii) budgeting strategy. 

\paragraph{Optimization objective.} First, in our \textbf{\ours{}\textsubscript{$\infty$}} setting, we take the \ry{} divergence function as our divergence metric~\citep{renyi1961measures}, which supplies a \emph{worst-case} guarantee instead of an \emph{average-case} one, and is commonly employed in sensitive machine learning applications such as differential privacy~\citep{dwork2006dp, mironov2017renyi}, where the objective is to bound the maximum possible information leakage from any single observation. The alternative derivation of \ours{} with \ry divergence leads to an analytical, closed-form solution. We show that this instantiation satisfies $K$-NAF in \cref{subsec:renyi_divergence}. In \textbf{NoOpt}, we ablate the optimization and retain the budget and prefix debt; at each step, we sample from $p_r$ if $\kl (p_r\mid p_s) \leq k_t$, and $p_s$, otherwise. Finally, in \textbf{ColdStart}, we use the per-step cap $k$ and sample only from $p_s$ for the first $k \times 10$ steps, after which we sample from $p_r$.

\paragraph{Prefix debt $\debt$.} We ablate the prefix debt by removing it entirely (\textbf{NoDebt}). In \textbf{AvgDebt}, we experiment with treat $\debt{}$ as an aggregate statistic by averaging over all prefix LLRs instead of taking the top-$n$ largest values. 

\paragraph{Budget allocation.}
We explore alternative budget allocation schemes: in \textbf{Fixed}, 
we assign a per-step, constant budget $k$, with no rollover of unused budget across timesteps. In \textbf{Global}, we allocate the full budget $K=kT_{\max}$ 
upfront and enforce only a cumulative constraint: we decode from $p_r$ until the running KL spend from $p_s$ reaches $K$, then switch to sampling from $p_s$ for the rest of generation.
\begin{figure}[htbp]
\begin{center}
  \centerline{\includegraphics[width=\linewidth]{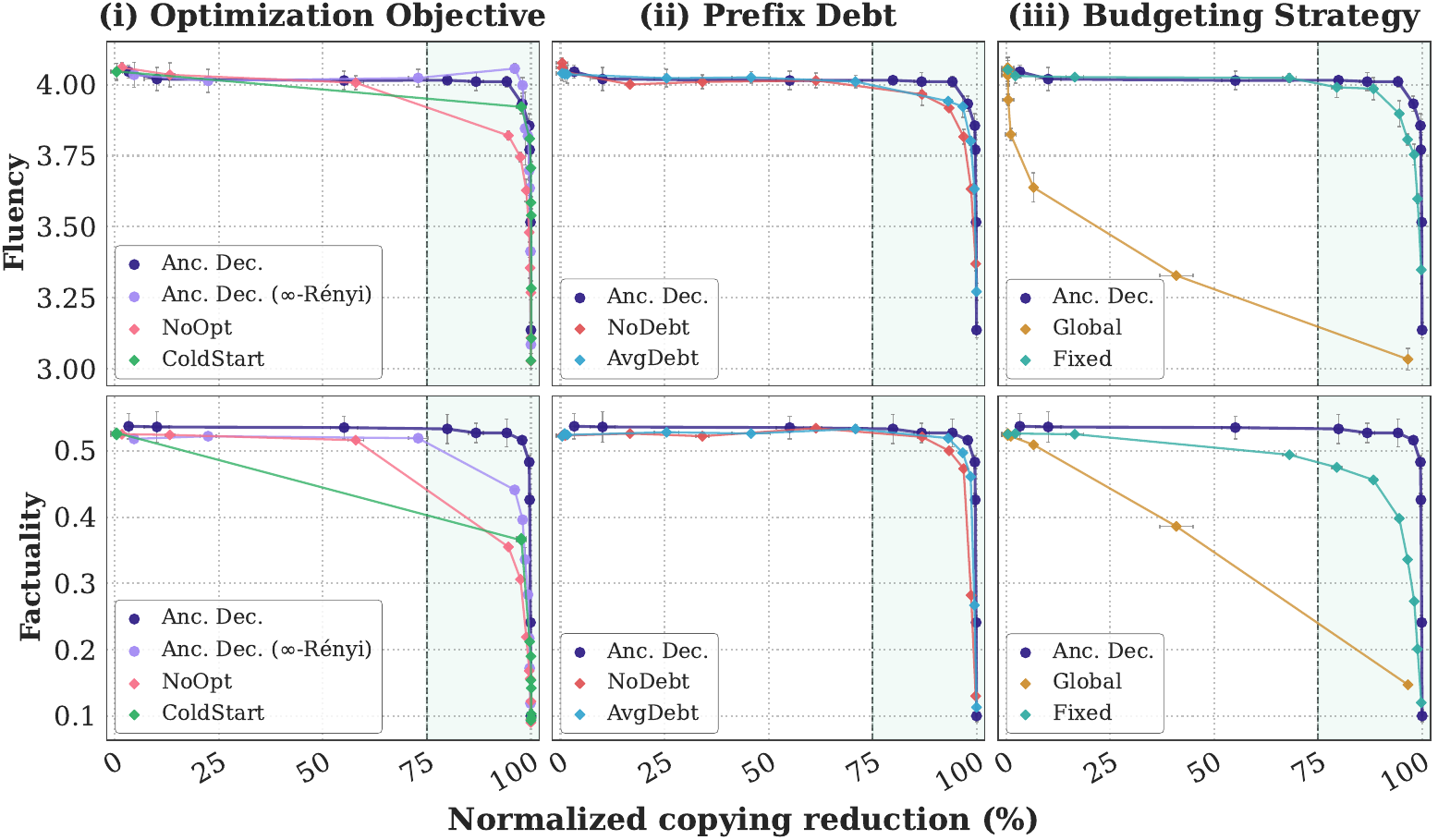}}
  \caption{\textbf{Risk-utility tradeoffs for \ours{} ablations.} We ablate three axes: (i) optimization objective, (ii) prefix debt, and (iii) budgeting strategy. For brevity, our methods are labeled as \textsc{Anc. Dec.}
  }
  \label{fig:tradeoffs_ablations}
\vskip -0.3in
\end{center}
\end{figure}
\paragraph{Ablation results.} 
\cref{fig:tradeoffs_ablations} shows ablation trade-off plots. For (1) Optimization Objective, \ours{} is strictly more Pareto-optimal than both \textbf{NoOpt} and \textbf{ColdStart} ablations. Our $\infty$-Rényi version, \textbf{\ours{}\textsubscript{$\infty$}}, attains a better fluency trade-off but worse factuality than our KL-based formulation, and serves as a strong, principled alternative for practitioners. 

For (ii) Prefix Debt, replacing our top-$n$ LLR aggregation with an average over all prefix LLRs (\textbf{AvgDebt}) yields a consistently worse trade-off, underscoring our treatment of prefix debt as a tail statistic. Removing the prefix debt altogether (\textbf{NoDebt}) further degrades the curve. And for (iii) Budgeting Strategy, \ours{}'s adaptive budget---which accounts for realized spend in prior steps---leads to a strictly better trade-off than either a constant cap (\textbf{Fixed}) or a holistic lump-sum budget (\textbf{Global}).
\begin{table}[htbp]
\centering
\small
\caption{\textbf{Token-level wall-clock benchmarking.} We report the time to first token (TTFT), throughput slowdown ratio relative to $p_r$ (TPS Ratio), and FLOPs/token estimate (\cref{app:flops_analysis}).}
\label{tab:tok_efficiency_analysis}
\resizebox{\columnwidth}{!}{%
\begin{tabular}{lrrr}
\toprule
\textbf{Method} & \textbf{TTFT} & \textbf{TPS Ratio} & \textbf{FLOPs Estimate} \\
& \small{(ms)} & \small{(vs.\ $p_r$, $\times$)} & \small{(FLOPs/token)} \\
\midrule
\multicolumn{4}{l}{\textit{Reference LMs}} \\ \midrule
$p_r=$ Llama 3.1 70B & 181.3 & 1.0$\times$ & 140$\times10^9$ \\
$p_s=$ \ourlm{}  & 80.1 & --- & 3.6$\times10^9$ \\
\midrule
\multicolumn{4}{l}{\textit{Single-Model Baselines (using $p_r$)}} \\ \midrule
\textsc{System} & 184.4 & 1.0$\times$ & 140$\times10^9$ \\
\textsc{MemFree} & 186.3 & 1.0$\times$ & 140$\times10^9$ \\
\textsc{RCAD} & 223.6 & 2.0$\times$ & 280$\times10^9$ \\
\midrule
\multicolumn{4}{l}{\textit{Two-Model Methods (using $p_r$ and $p_s$)}} \\ \midrule
\textsc{CP-Fuse} & 210.6 & 1.3$\times$ & 143.6$\times10^9$ \\
\textsc{Token Swap} & 204.0 & 1.3$\times$ & 143.6$\times10^9$ \\
\ours{} & 195.9 & 1.1$\times$ & 143.6$\times10^9$ \\
\bottomrule

\end{tabular}%
}
\end{table}

\begin{figure}[htbp]
  \begin{center}
\centerline{\includegraphics[width=0.9\columnwidth]{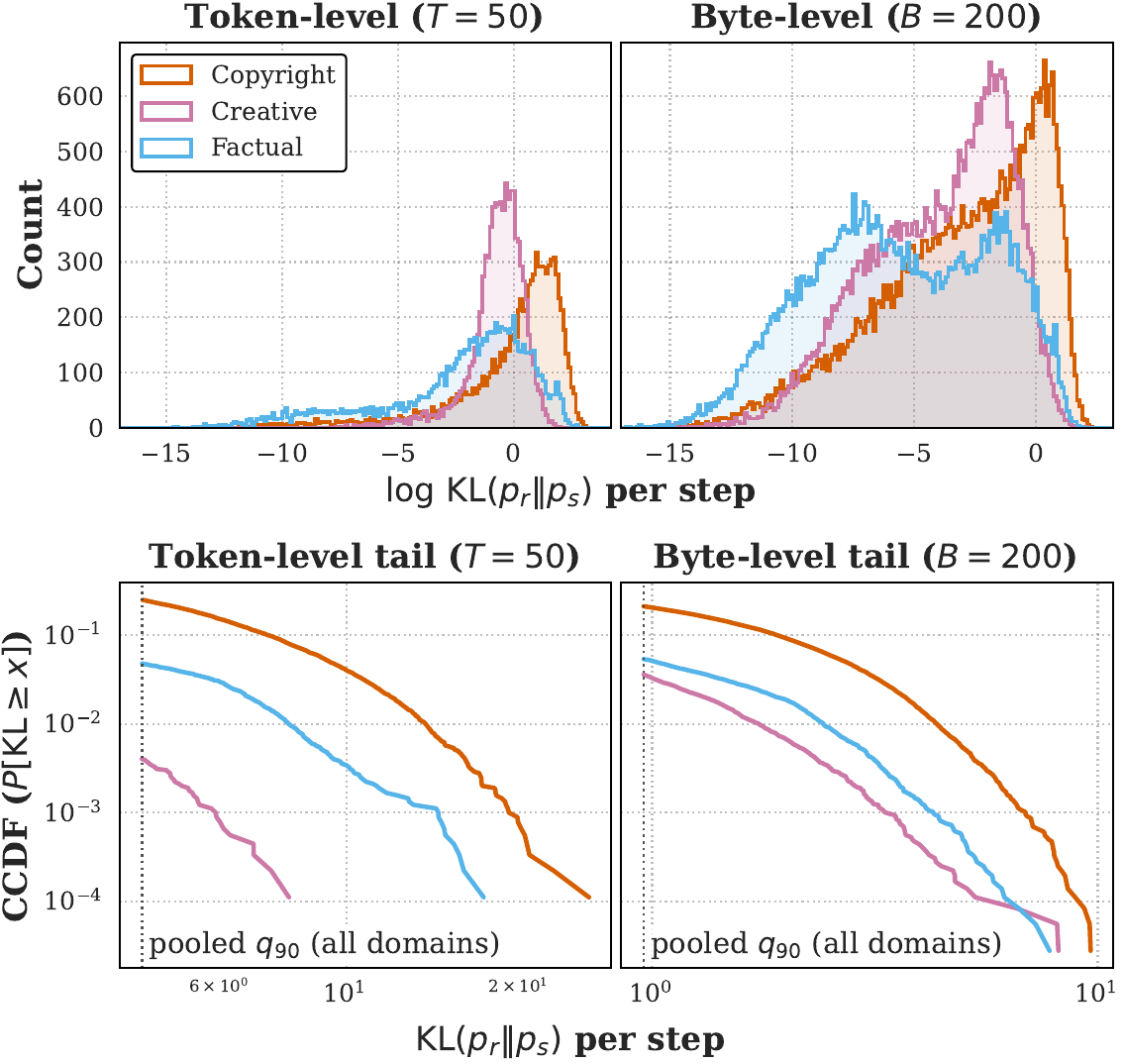}}
\vskip -0.1in
\caption{\textbf{Top:} Per-step $\mathrm{KL}(p_r\|p_s)$ histogram when sampling from $p_r$, conditioned on prefixes different domains. The \textbf{Copyright} domain is more right-shifted than the \textbf{Creative} and \textbf{Factual} domains. \textbf{Bottom:} Unconditional CCDF of per-step $\mathrm{KL}(p_r|p_s)$, shown for $x \ge q_{90}$. $q_{90}$ is computed from per-step KL values \emph{pooled across domains} (shared cutoff per panel). The \textbf{Copyright} domain has a heavier extreme tail than others.
}
  \label{fig:kl_hist_ccdf}
  \end{center}
  \vskip -0.4in
\end{figure}
\begin{figure*}[htbp]
  \begin{center}
\centerline{\includegraphics[width=0.95\linewidth]{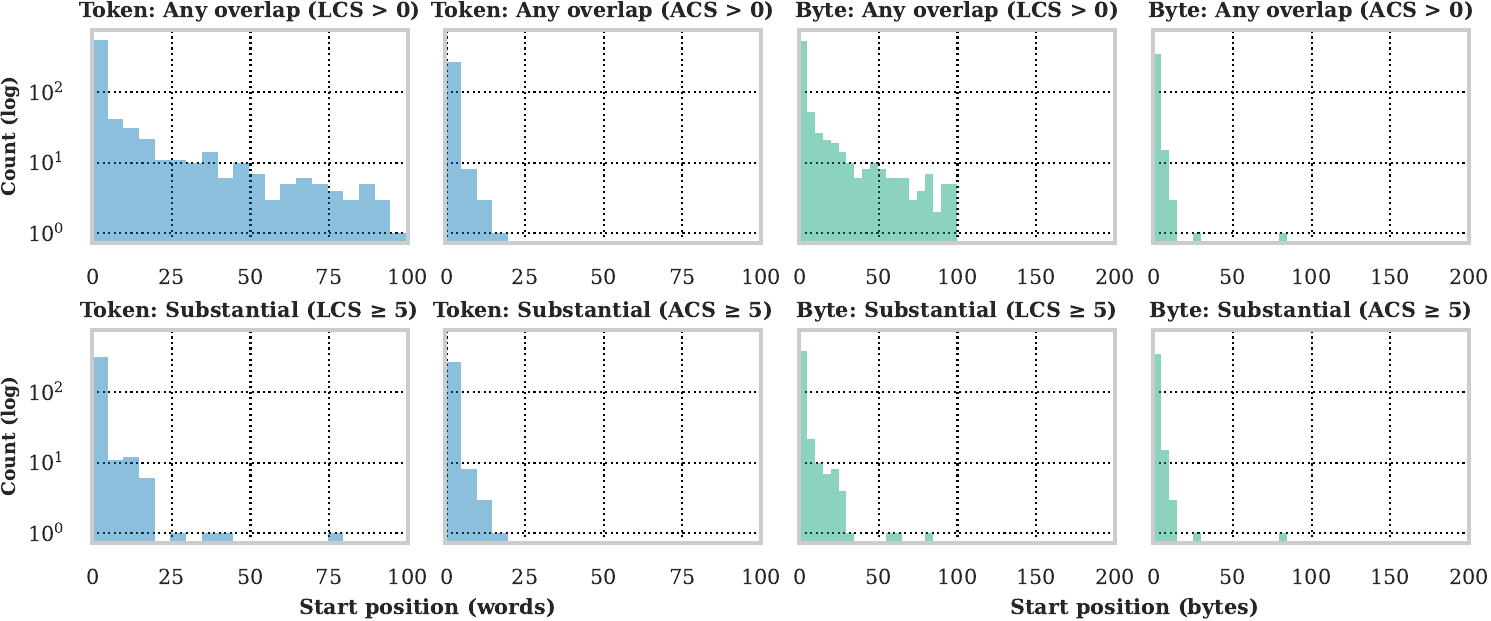}}
  \vskip -0.1in
\caption{\textbf{High-copying regions are front-loaded under both byte-level and token-level decoding.} We plot histograms (bin width of 5) of the start position of copying metrics (LCS and ACS) on  \textsc{Copyright} generations. Copying tends to cluster at early positions.}
  \label{fig:positional_analysis}
  \end{center}
\vskip -0.4in
\end{figure*}
\subsection{Efficiency}
A practical consideration for safety-constrained decoding is inference-time overhead. \cref{tab:tok_efficiency_analysis} reports time to first token (TTFT), throughput slowdown relative to $p_r$ (TPS Ratio), and FLOPs estimate for the token-level pair \{\ourlm{}, Llama 3.1 70B\}, computed as $2(N_r+N_s)$, where $N_r$, $N_s$ are parameter counts for $p_r$, $p_s$, respectively (see \cref{app:flops_analysis} for formal derivations).\footnote{We provide byte-level wall-clock measurements in \cref{app:byte_efficiency}.} We run all settings on 2 140 GiB NVIDIA H200s. As expected for joint-model decoding, \ours{} incurs a modest throughput overhead, operating $\approx1.1\times$ slower than standalone $p_r$ decoding. While the additional forward pass with $p_s$ increases arithmetic compute by only $\approx2.6\%$ (from 143.6 to\ 140 GFLOPs/token), the observed wall-clock slowdown is consistent with bandwidth- and synchronization-bound overheads from logit fusion. All methods are more efficient than \textsc{RCAD}, which requires two forward passes with $p_r$.

\begin{figure*}[htbp]
  \begin{center}
\centerline{\includegraphics[width=\linewidth]{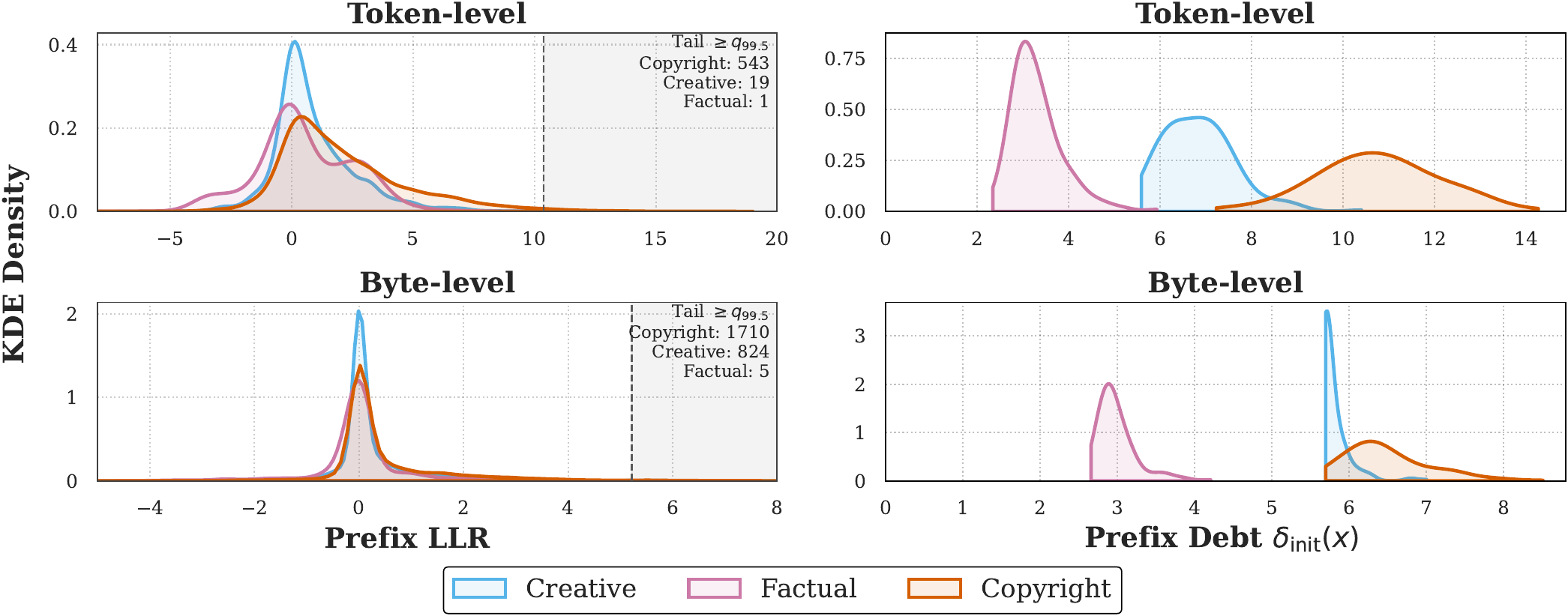}}
\caption{\textbf{Left:} KDEs of per-step prefix log-likelihood ratios (LLR) at the token and byte levels. A positive LLR means that $p_r$ assigns higher probability than the $p_s$ to the realized next step; large positive LLR events occur most often in \textbf{Copyright}. \textbf{Right:} KDEs of $\delta_{\mathrm{init}}(x)$, the mean of the top-5 positive prefix LLRs. \textbf{Copyright} prefixes are more right-skewed than \textbf{Creative} and \textbf{Factual} prefixes.}
  \label{fig:prefix_llr_and_debt}
  \end{center}
\vskip -0.4in
\end{figure*}

\section{Analysis} 
\label{sec:analysis}

We perform targeted experiments to motivate the design choices behind \ours{} (\cref{app:analysis_setting}). We first find that per-step $\mathrm{KL}(p_r \kld p_s)$ serves as a reliable proxy for identifying memorization by $p_r$ but not $p_s$ (\cref{app:analysis_kldiv}). Next, we show that copying risk is front-loaded in early generation, motivating the prefix debt (\cref{app:analysis_debt}).

\subsection{Experimental Setting} \label{app:analysis_setting}
We conduct our analyses at token and byte granularities using the representative model pair $\{p_s=\text{\ourlm{}},\; p_r=\text{Llama 3.1 70B}\}$ on three prompt domains (180 prefixes each).

\textbf{Copyright} contains copyright-protected book excerpts from \textsc{Books}. \textbf{Factual} consists of biography generation prompts from \fs{}. 
A potential confound is that these two prompt sets differ in domain and style, so any diagnostic separation may reflect distributional shift. 
As a control, we introduce a \textbf{Creative} domain by scraping original story prompts from Reddit's \texttt{r/WritingPrompts} community (from 2024--2025, after Llama 3.1's training cutoff). \textbf{Creative} is roughly in-distribution to \textbf{Copyright}; \cref{tab:eval_examples_heldout} shows qualitative examples. Neither LM should have memorized any \textbf{Creative} or \textbf{Factual} prompts, but $p_r$ has likely memorized some \textbf{Copyright} prompts.

\subsection{KL is a Useful Memorization Diagnostic} \label{app:analysis_kldiv}
We test whether the per-step divergence between risky and safe models can meaningfully separate copyright-sensitive prefixes from benign ones. Namely, for each prefix $x$, we sample a continuation from $p_r$ for $T{=}50$ tokens (or $B{=}200$ bytes), and compute  $\kl\!\left(p_r(\cdot\mid y_{<t},x)\kld p_s(\cdot\mid y_{<t},x)\right)$ at each decoding step $t$ of the rollout. When $p_r$ follows a memorized passage that $p_s$ has not learned, then $p_r$'s next-token distribution becomes sharply peaked on a specific continuation that $p_s$ considers unlikely. This mismatch yields consistently large per-step $\kl (p_r \| p_s)$ values across decoding steps.  

Consistent with this intuition, the top row of \cref{fig:kl_hist_ccdf} shows that \textbf{Copyright} prefixes are systematically right-shifted relative to \textbf{Creative} and \textbf{Factual} ones. This difference is most pronounced in the extreme right tail (bottom row): we plot the empirical CCDF of per-step  $\kl(p_r|p_s)$ for $x \ge q_{90}^{(\mathrm{pooled})}$, where $q_{90}^{(\mathrm{pooled})}$ is the pooled 90th percentile (shared cutoff per panel). Under this common threshold, the \textbf{Copyright} domain has the largest tail mass---i.e., for any large $x$, it has the highest fraction of steps with $\kl(p_r\|p_s)\ge x$---while \textbf{Creative} and \textbf{Factual} place little mass on such extreme deviations. These results motivate using per-step $\mathrm{KL}(p_r\|p_s)$ to detect when generation is in a copyright-sensitive regime.

\subsection{Copying Risk is Front-loaded in Early Generation} \label{app:analysis_debt}

Our second finding is that the generation of tokens favored by $p_r$ and not $p_s$ is not uniformly distributed over a continuation: copying events tend to occur early in generation.  

To quantify this positional bias, we compute two surface-copying metrics---the longest common substring (LCS) and the average common substring (ACS)---and, for each continuation from a \textbf{Copyright} prompt, we record the start position of the first matched substring (i.e., where copyright infringement begins). We report two regimes: (i) \emph{any} overlap (ACS, LCS $>0$) and (ii) \emph{substantial} overlap (ACS, LCS $\geq 5$). 
 
 As \cref{fig:positional_analysis} shows, overlap events concentrate heavily at the beginning and drop sharply with generation across both regimes. This front-loading behavior motivates a stronger copyright-mitigation intervention in early decoding steps, which we operationalize by intentionally forcing greater reliance on $p_s$  via the prefix debt. $\debt$ treats the prompt as a \emph{memorization prior} and debits the global budget $K$ in proportion to how strongly the prefix favors $p_r$ over $p_s$. Concretely, we estimate $\debt$ from the upper tail of prefix log-likelihood ratios (average of the top-$n$ positive LLRs), so that a small number of outlier prompt tokens (or bytes) can trigger greater reliance on $p_s$. 
 
\textbf{Copyright} prefixes display heavier right tails in prefix LLRs than either \textbf{Copyright} or \textbf{Factual} prefixes, persisting into extreme quantiles (e.g., the $99.5$ percentile) (\cref{fig:prefix_llr_and_debt}). When we directly examine prefix debt values (e.g., the average of the top-5 largest positive prefix LLRs), this tail behavior translates into a similar right-shifted distribution for \textbf{Copyright}. These patterns justify a ``cold start" \emph{only when} prefixes exhibit evidence of asymmetric memorization.

\section{Discussion}
\ours{} confers several desirable properties. The first is flexibility: by providing a controllable sequence-level safety budget $K$, our strategy allows practitioners to specify a desired risk threshold that retains formal guarantees. 
Second, \ours{} is practical, as it requires neither re-training nor access to the original pre-training data at inference time. This allows for the immediate and retroactive safe-guarding of high-utility models with considerable infringement leakage. 
Moreover, through \oursbs{}, we demonstrate that our strategy is tokenizer-agnostic and works effectively across model pairs with mismatched vocabularies.
Finally, we show that \ours{} yields a strong risk–utility trade-off, leveraging the observation that the contrast between risky and safe models provides a signal for when generation is likely in a copyrighted regime. 

While \cref{app:limitations} discusses future directions in greater detail, we note that \ours{} extends beyond copyright mitigation in LMs. The same mechanism applies wherever a high-capability model must be bounded by a trusted anchor distribution, for example, to reduce sensitive-attribute leakage, enforce domain or policy restrictions for safety, or align generation with licensed corpora in enterprise settings.
Ultimately, our results motivate a \emph{reference-anchored decoding} paradigm in which practitioners choose the reference model to match the compliance target of interest.

\section{Impact Statement}
Our work addresses a critical and timely challenge at the intersection of generative AI and intellectual property, in the midst of ongoing high-profile litigation and a rapidly evolving legal landscape. 

LMs are trained on large-scale, web-scraped corpora that may include copyright-protected material, which precipitates several harms. On the creator side, model outputs that exhibit substantial similarity to protected training examples may infringe on intellectual property rights and erode the market value of original works~\citep{henderson2023foundation}. These issues are compounded by the lack of systematic mechanisms for creator consent~\citep{longpre2024consent} or compensation~\citep{baack2025bestpracticesopendatasets, kandpal2025positionexpensivellmtraining}. On the developer side, as  the applicability of the U.S. fair use doctrine~\citep{uscode17_107_2024} to model training remains unsettled, LM memorization and reproduction may expose developers to significant liability risk. \ours{} targets these concerns as a \emph{post-hoc} technical mitigation strategy that constrains generation toward a trusted reference distribution at inference time, reducing verbatim reproduction of protected documents and promoting more transformative outputs. And in order to more closely bridge theory and deployment, we design \ours{} to be universally compatible with modern LMs. When tokenizers are aligned, we may operate at the token level; otherwise, our byte-level integration removes the shared-vocabulary requirement and supports LM pairs with mismatched tokenizers.

Nevertheless, \ours{} is not a silver bullet. Our guarantees and empirical results apply only under the assumptions and evaluation protocol of our study, and do not certify minimal copyright risk under any legal or model-independent sense. $K$-NAF, in particular, is a guarantee that the decoded distribution provably remains within a controlled divergence budget of a permissively trained LM. By itself, it is not a direct legal certification of non-infringement, and its practical interpretation depends on the safe model being a genuinely copyright-free reference. In sum, no legal conclusions ought to be inferred from this work. Ultimately, we view our contribution as complementary to other safeguards and as one part of a broader toolbox for responsible deployment.

\section{Acknowledgments}
We are very thankful to Boyi Wei, Colin Raffel, Gonçalo Faria, Parjanya Prajakta Prashant, Tomasz Limisiewicz, and Tong Chen for helpful discussions about methodology and evaluation, Howard Yen for sharing retrieval infrastructure and feedback on framing, and Oscar Yinn and Stella Li for support. Jacqueline He is supported by an NSF Graduate Research Fellowship and the Meta AI Mentorship program. Jonathan Hayase and Sewoong Oh are supported by NSF grants 2112471, 2229876, and 2505865. This work was also supported by the Singapore National Research Foundation and the National AI Group in the Singapore Ministry of Digital Development and Information under the AI Visiting Professorship Programme (award number AIVP-2024-001) and the AI2050 program at Schmidt Sciences.

\bibliography{main}
\bibliographystyle{icml2025}

\newpage
\appendix
\onecolumn
\crefalias{section}{appendix}
\crefalias{subsection}{appendix}

\crefname{appendix}{App.}{Apps.}
\Crefname{appendix}{App.}{Apps.}
\section{General Information}
\subsection{Released artifacts}
We release the following artifacts for reproducibility and future development:
\renewcommand{\arraystretch}{1.2}
\begin{center}
\begin{tabular}{lrl}
 \textbf{Codebase} & \github & \href{https://github.com/jacqueline-he/anchored-decoding}{jacqueline-he/anchored-decoding} \\
  \textbf{\ourlm{}} & \huggingface & \href{https://huggingface.co/jacquelinehe/tinycomma-1.8b-llama3-tokenizer}{jacquelinehe/tinycomma-1.8b-llama3-tokenizer} 
\end{tabular}
\label{app:artifacts}
\end{center}

\subsection{Related Work}
\label{sec:related_work}
\paragraph{Interventions against LM copyright infringement.} 
Recently proposed strategies have sought to minimize copyright risk at all stages of the language modeling pipeline. Pre-training efforts include training only on public domain and openly licensed texts~\citep{min2024silo, kandpal2025the}, or modifying the next-token prediction objective to selectively mask pre-training spans either randomly~\citep{hans2024be} or in a targeted manner~\citep{wang2025teaching} to discourage their exact reproduction. A separate line of work applies a secondary post-hoc learning stage to an already pre-trained model, in order to excise undesirable knowledge (i.e., model unlearning~\citep{maini2024tofu, yao2024large,zhang2024negative, russinovich2025obliviateefficientunmemorizationprotecting}) or align with user intent to suppress the unintentional regurgitation of training data~\citep{chen2025parapo}. Finally, some methods operate post-generation by identifying infringing spans---either via efficient Bloom filters applied to pre-training corpora \citep{zhang-etal-2025-certified} or via multi-agent web searches \citep{liu-etal-2024-shield}---before rewriting them into non-offending versions using LLMs. 

Our approach centers on the decoding stage, motivated by the observation that training or re-training frontier LLMs is an expensive endeavor; likewise, multi-agent pipelines incur nontrivial inference overhead and do not change the underlying generative process. However, we view \ours{} as orthogonal to pre-training, post-training, and post-generation procedures, and leave their integration to future work.

While many techniques are heuristic, \citet{vyas2023provablecopyrightprotectiongenerative} formalizes the theoretical notion of provable copyright protection at inference time via $K$-Near Access Freeness, which forms the groundwork to \ours{}. \ours{} shares theoretical similarities to \textsc{CP-Fuse}~\citep{abad2025copyrightprotected}, which to our knowledge is the only other $K$-NAF--inspired algorithm, but we note a few differences. First, \textsc{CP-Fuse} assumes that copyright-infringing datapoints are not known \emph{a priori}, and thus requires an LM pair trained on discrete data shards, where no single datum appears in both models' training sets.  This restriction is quite unrealistic for off-the-shelf, production-grade LLMs, whose data provenance is often unknown or undisclosed. In contrast, \ours{} implicitly assumes that the set of copyright-infringing data is known and may be in $p_r$, but not $p_s$. Second, \textsc{CP-Fuse} is a pointwise baseline and does not allow the user to specify a preferred risk tolerance; in our experiments, this manifests as degraded utility. On the other hand, \ours{} exposes a control knob that allows for tuning of the risk-utility trade-off curve.

\paragraph{Two-model decoding methods.} \ours{} belongs to a proliferative body of literature that proposes to decode using asymmetric model pairs. One category focuses on expert-guided generation via model arithmetic~\citep{liu-etal-2021-dexperts, li-etal-2023-contrastive}, which involves amplifying a high-capability ``expert" model by downweighting undesirable characteristics from an ``amateur" model, ultimately favoring tokens with high expert-over-amateur scores to improve utility. In contrast, our projection is derived from an objective that explicitly trades off utility against copying risk. This distinction also underlies the theoretical guarantees for our method, which these methods do not not provide. Another class of decoding strategies employs a drafter-verifier framework to accelerate inference, and notably includes speculative decoding~\citep{chen2023acceleratinglargelanguagemodel,leviathan2023fast, eagle}, in which a lightweight draft model proposes candidates for a larger model to verify. Our approach employs a dual-model structure with a different objective: rather than optimizing purely for downstream performance or efficiency, \ours{} seeks to generate text with less copyright infringement in a utility-preserving manner. 

Closest to \ours{} is \textsc{TokenSwap}~\citep{prashant2025tokenswap}, which operates with a mismatched-size model pair and leverages the empirical observation that smaller models tend to memorize less than the larger one~\citep{kandpal2023large}. This heuristic is rather coarse: both models can still reproduce protected text, especially for spans that appear frequently in pre-training corpora. Moreover, \textsc{TokenSwap} relies on swapping a predefined list of common English tokens (largely function/grammar tokens), which limits its portability across languages and cannot address copying events that do not pass through the chosen token set. \ours{} bypasses both limitations through (i) an explicit safe–risky model pairing (rather than relying on model size as a proxy for memorization) and (ii) distribution-level fusion that enforces a divergence budget to the safe model at every decoding step.

Finally, many two-model decoding methods assume a shared tokenization vocabulary. We remove this bottleneck by adapting our method (and relevant baselines) to operate via \bs{}~\citep{hayase2025samplinglanguagemodelbyte}, and show that the same algorithms at the byte level still ensure the strong mitigation of training data reproduction. Many approximate methods have been proposed to overcome mismatching vocabularies in model ensembles, such as using beam search as a scoring function~\citep{kasai-etal-2022-twist} or using a mapping based on model features~\citep{huang2024enabling}; we opted for bytewise sampling under \bs{} as it is gives \emph{exactly} the same distributions of output text~\citep{hayase2025samplinglanguagemodelbyte}. Beyond copyright reduction, byte-level decoding suggests a general route for making two-model decoding practical in cross-tokenizer settings.

\subsection{Limitations and Future Work}
\label{app:limitations}
\textbf{Probabilistic risk.} To begin, \ours{} does not fully eliminate the possibility of generating protected spans. Our method is a sampling strategy instead of a discrete filtering or blocking mechanism. It inherently inherits the baseline risk profile of the safe model, $p_s$. While the probability of $p_s$ reproducing a sequence it was not exposed to during training is typically extremely small, it remains strictly non-zero. Under the $K$-NAF framework, we only guarantee that the risk of an infringing generation is comparable to the safe baseline within a bounded, controllable distance. 

\textbf{Local approximation.} Additionally, while \ours{} solves a local optimization at each step for computational tractability, this sequential approximation may not represent the global optimum of the sequence-level constrained objective. This is a necessary trade-off for efficient autoregressive decoding. 

\textbf{Asymmetric memorization as an imperfect proxy.} Our framework treats asymmetric memorization---what $p_r$ has memorized, but $p_s$ has not, via large deviations in metrics such as per-step KL or prefix LLR---as a proxy indicator of elevated copying risk. However, these signals are not unique to copyright: they can also arise when the risky model contains useful long-tail knowledge that the safe model lacks. Because we intentionally chose larger models as $p_r$, which naturally memorize a broader tail spectrum of facts than the smaller models selected as $p_s$, our divergence constraints may inadvertently suppress rare, non-copyrighted factual information. Note that this phenomenon is highly model-pair dependent, and should be mitigated when the safe reference is closer in capability to the risky baseline; empirically, our results indicate that stronger safe models reduce the utility cost of \ours{}, by imposing less distortion on the risky model's preffered distribution at a fixed budget. 

\paragraph{Data provenance and latent leakage.} Another limitation is that the efficacy of \ours{} relies on the \emph{a priori} identification of model pairs with known data origins. Specifically, the safe model $p_s$ must be verified as having been trained exclusively on copyright-free or openly licensed data. This assumption is supported in our experiments both by construction and empirically: our choices of safe models have openly documented training data provenance that excludes the copyrighted books used in our evaluation, and also exhibit low absolute copying values. 

There is also an inherent risk of latent copyright leakage: protected fragments (e.g., famous literary quotes) may still permeate ostensibly open-licensed data (e.g., within blog posts or public forums). In such cases, $p_s$ may exhibit a baseline propensity to reproduce such protected sequences, a risk that \ours{} can bound but not entirely eliminate. Finally, \ours{}'s primary threat model is when copyright arises from the parametric memorization of training data; alternative forms of injection, such as when copyrighted text is explicitly supplied in-prompt, are not considered in this work. 

\paragraph{Future directions.} Beyond copyrighted text, we believe that \ours{} is applicable wherever a high-capability, high-risk generator must be bounded by a trusted reference distribution. 
Our approach is agnostic to tokenizer, modality, and domain. Our byte-level experiments demonstrate that decoding effects are persistent regardless of the tokenization scheme, and we encourage future two-model decoding baselines to adopt byte-level evaluations to ensure broader architectural compatibility.
One promising direction is to extend our framework to generative AI technologies beyond LLMs, e.g., image or video generation, where the risk of memorizing protected artistic styles or iconic visual frames presents similar copyright challenges~\citep{he2025fantastic, moayeri2025rethinking, wang2025how}. For instance, diffusion models from the CommonCanvas suite—trained on Creative Commons–licensed images—are promising permissively trained candidates for $p_s$~\citep{gokaslan2023commoncanvasopendiffusionmodel}. Another direction is the application of \ours{} to other domains, e.g., policy compliance, code safety, or privacy redaction, to suppress the leakage of sensitive information in a focused manner while retaining general capabilities. 

\section{\ours{} Details}
\subsection{Proofs}
\label{subsec:algorithm_proof}
\subsubsection{A token-level approximation}
\SafetyLocalRelaxation*
\begin{proof}
By the chain rule for KL divergence, the sequence-level divergence between the generated distribution $p^*$ and the safe model $p_s$ can be decomposed as an expectation over the sum of local conditional divergences:
\begin{align}
\kl\big(p^*(y_{0:T-1} \mid x) \kld p_s(y_{0:T-1} \mid x)\big) &= \mathbb{E}_{y \sim p^*} \left[ \sum_{t=0}^{T-1} \kl\big(p_t^*(\cdot \mid y_{<t}, x) \kld p_s(\cdot \mid y_{<t}, x)\big) \right].
\end{align}

\cref{eq:token_level_opt} shows that by construction, the next-token distribution at each step is constrained such that for every possible prefix $y_{<t}$, $\kl \big(p^*_t(\cdot \vert y_{<t}, x) \kld p_s(\cdot \vert y_{<t}, x)\big) \le k_t$. Therefore, the expectation of the sum is bounded by the sum of the bounds:
\begin{align}
\mathbb{E}_{y \sim p^*} \left[ \sum_{t=0}^{T-1} \kl \big(p^*_t(\cdot \vert y_{<t}, x) \kld p_s(\cdot \vert y_{<t}, x)\big) \right] \leq \sum_{t=0}^{T-1} k_t  \leq  \sum_{t=0}^{T_{\max}-1} k_t   \leq K.
\end{align}
Thus, the global $K$-NAF condition is satisfied.
\end{proof}

\subsubsection{A closed-form solution for $p_t^*$}

\begin{lemma}[Interior optimality on the common support]
Let $\mathcal{S} \subseteq \mathcal{V}$ denote the common support of the reference models,
i.e., $p_r(y) > 0$ and $p_s(y) > 0$ for all $y \in \mathcal{S}$.
Consider \cref{eq:token_level_opt} restricted to distributions $p$ supported on $\mathcal{S}$
(i.e., $p(y)=0$ for $y\notin \mathcal{S}$). Then the optimal solution $p_t^*$ satisfies
$p_t^*(y) > 0$ for all $y \in \mathcal{S}$.
\label{lem:interior_optimality}
\end{lemma}
\begin{proof}
To show that $p_t^*(y) > 0$, fix any $y \in \mathcal{S}$ and consider the contribution of a single coordinate $u = p_t(y)$ to the objective gradient. As $u \rightarrow 0^+$, the value $u \log \frac{u}{p_r(y)}$ approaches 0, and its directional derivative $\log \frac{u}{p_r(y)} + 1$ tends toward $-\infty$.

We first note that the feasible set has nonempty interior (relative to the simplex over $\mathcal{S}$):
since $p_s$ has full support on $\mathcal{S}$ and $\kl(p_s\kld p_s)=0\le k$, the constraint
$\kl(p\kld p_s)\le k$ contains $p_s$ and, by continuity of $\kl(\cdot\kld p_s)$ on the interior,
also contains an open neighborhood of $p_s$. Hence the constraint set is not confined to the boundary.

Suppose for contradiction that an optimal solution $p_t^*$ satisfies $p_t^*(y)=0$ for some $y\in\mathcal{S}$.
Because the feasible set contains interior points, we can construct a feasible perturbation by moving an
infinitesimal mass $\varepsilon>0$ from any coordinate $y'\in\mathcal{S}$ with $p_t^*(y')>0$ to $y$,
obtaining $p_\varepsilon$. By continuity of $\kl(\cdot\kld p_s)$ on the interior, and the fact that $p_s(y)>0$,
then for sufficiently small $\varepsilon$ we still have $\kl(p_\varepsilon\kld p_s)\le k$.

However, the directional derivative of the objective $\kl(p\kld p_r)$ in the direction that increases $p(y)$
from $0$ is $-\infty$ (since $\log\frac{u}{p_r(y)}+1\to -\infty$ as $u\to 0^+$), so for small enough $\varepsilon$,
we get $\kl(p_\varepsilon\kld p_r) < \kl(p_t^*\kld p_r)$, contradicting optimality. Therefore, $p_t^*(y)>0$ for all
$y\in\mathcal{S}$.
\end{proof}

\ClosedFormFusion*
\begin{proof}
Denote $\lambda \geq 0$ as the Lagrangian multiplier for the KL-ball constraint, $\alpha \in \mathbb{R}$ for the simplex constraint $\sum_{y_t \in \mathcal{V}} p(y_t) = 1$, and $\mu_{y_t}\geq0$ for each non-negativity constraint $p(y_t) \geq 0$. Then we can define the following Lagrangian form: 
\begin{align}
\mathcal{L}\big(p(\cdot|y_{<t}, x), \lambda, \alpha, \mu_{y_t} \big) &= \kl \big(p(\cdot \vert y_{<t}, x) \vert p_r (\cdot | y_{<t}, x)\big) +  \lambda \Big(\kl \big(p(\cdot \vert y_{<t}, x) \vert p_s (\cdot \vert y_{<t}, x)\big) - k_t \Big) \\ 
&+ \alpha \Big( \sum_{y_t \in \mathcal{V}} p(y_t \vert y_{<t}, x) - 1 \Big) - \sum_{y_t \in \mathcal{V}} \mu_{y_t} p(y_t \vert y_{<t}, x).
\end{align}
We can invoke complementary slackness: for the constraints $g_y(p) = -p(y) \leq 0$, KKT conditions require $\mu_y p^*_t(y) = 0$, which implies that if a constraint is inactive (which \cref{lem:interior_optimality} proves by showing that $p_t^*(y) > 0$), then $\mu_{y_t}=0$~\citep{kuhn1951nonlinear}. Further, for each token $y_t \in \mathcal{V}$, we can differentiate $\mathcal{L}$ with respect to $p(y_t)$ and set to zero, leading to
\begin{align}
\frac{\partial \mathcal{L}}{\partial p(y_t)}
= \log \frac{p(y_t)}{p_r(y_t)} + 1
+ \lambda\!\left(\log \frac{p(y_t)}{p_s(y_t)} + 1\right)
+ \alpha = 0.
\end{align}
Rearranging terms to isolate $\log p(y_t)$, we have:
\begin{align}
(1 + \lambda) \log p(y_t) &= \log p_r(y_t) + \lambda \log p_s(y_t) - (1 + \lambda + \alpha).
\end{align}
Dividing by $(1 + \lambda)$ and exponentiating both sides:
\begin{align} 
p(y_t) &= \exp\left( \frac{\log p_r(y_t) + \lambda \log p_s(y_t)}{1 + \lambda} \right) \cdot \exp\left( -\frac{1 + \lambda + \alpha}{1 + \lambda} \right) \\
&= \frac{1}{Z} p_r(y_t)^{\frac{1}{1 + \lambda}} p_s(y_t)^{\frac{\lambda}{1 + \lambda}}, \label{eq:logit_fusion_form}
\end{align}
where $Z = \exp\left( \frac{1 + \lambda + \alpha}{1 + \lambda} \right)$ acts as the normalization constant to satisfy the simplex constraint.
\end{proof}

\subsubsection{An adaptive banking budget.} \label{subsec:kl_adaptive_budget}

\AdaptiveBankingSafety*
\begin{proof}
Fix any $T\le T_{\max}$ and any realized trajectory $y_{0:T-1}\sim p^*(\cdot\mid x)$.
Let $S_t \coloneqq \sum_{i=0}^{t} a_i$ denote the cumulative realized expenditure up to step $t$.

We claim the following invariant holds for all $t\in\{0,\dots,T-1\}$:
\begin{align}
S_t \;\le\; \max\{0,\ (t+1)k-\debt\}. \label{eq:banking_invariant}
\end{align}

\paragraph{Base case ($t=0$).}
By feasibility, $a_0\le k_0=\max\{0,k-\debt\}$, hence $S_0=a_0\le \max\{0,k-\debt\}$.

\paragraph{Inductive step ($t > 1$).}
Assume \eqref{eq:banking_invariant} holds for $t-1$.
Using feasibility $a_t\le k_t$ and the definition of $k_t$,
\begin{align}
S_t
&= S_{t-1}+a_t \\
&\le S_{t-1} + \max\{0,\ (t+1)k - S_{t-1} - \debt\} \\
&= \max\{S_{t-1},\ (t+1)k-\debt\} \\
&\le \max\{0,\ (t+1)k-\debt\},
\end{align}
where the last line uses $S_{t-1}\ge 0$ (since each $a_i\ge 0$). This completes the induction.

Applying \eqref{eq:banking_invariant} at $t=T-1$ gives
\[
\sum_{t=0}^{T-1} a_t = S_{T-1}\ \le\ \max\{0,\ Tk-\debt\}\ \le\ \max\{0,\ K-\debt\},
\]
since $Tk\le T_{\max}k=K$. Finally, by the KL chain rule,
\[
\kl(p^*(y_{0:T-1}\mid x)\kld p_s(y_{0:T-1}\mid x))
= \mathbb{E}_{y\sim p^*}\left[\sum_{t=0}^{T-1} a_t\right]
\le \max\{0,\ K-\debt\} \le K,
\]
as desired.
\end{proof}

\subsection{\oursbs}

\paragraph{Satisfying $K$-NAF with \oursbs.} 
\label{subsec:knaf_byte}
\begin{remark}[Safety preservation under byte-level decoding]Let $\mathsf{Byte}(\cdot)$ denote the operator that maps a token-level LM to the \emph{induced} autoregressive distribution over UTF-8 bytes by exactly marginalizing token probabilities into a next-byte distribution at each byte prefix (using the ByteSampler abstraction). 
Define $\tilde p_s \coloneqq \mathsf{Byte}(p_s)$ and $\tilde p_r \coloneqq \mathsf{Byte}(p_r)$; notably, this mapping is strictly procedural and requires no auxiliary data, additional models, or further training. The global $K$-NAF guarantee applies to the byte transition space through three consistent translations:

\begin{enumerate}
    \item \textbf{Byte-level optimization:} By the chain rule for KL divergence over discrete autoregressive transitions, if $\kl(\tilde p_i^* \kld \tilde p_{s,i}) \le k_i$ at every byte step $i$, then $$ \kl(\tilde p^* \kld \tilde p_s) = \mathbb{E}_{\mathbf{b} \sim \tilde p^*} \left[ \sum_{i=0}^{B-1} \kl(\tilde p_i^* \kld \tilde p_{s,i}) \right] \le K - \debtbs \le K, $$
    for some $\debtbs \geq 0$.
    \item \textbf{Byte-level banking budget:} The adaptive budget $k_i$ at byte-step $i$ is updated as:$$ k_i = \max\left(0,\; (i+1)k - \sum_{j=0}^{i-1} \kl(\tilde p_j^* \kld \tilde p_{s,j}) - \debtbs \right), $$ where $k = K/B_{\max}$ is the nominal per-byte allotment.
    \item \textbf{Byte-level prefix debt:} The prefix debt $\debtbs$ is calculated as the mean of the top-$n$ LLR spikes across the $L$ bytes of the prefix byte sequence $\tilde x$, ensuring $\debtbs \geq 0$.
\end{enumerate}
Thus, the safety guarantee is also applicable in the byte transition space.\end{remark}

By decoding byte-by-byte, \oursbs{} offers finer-grained control for copyright prediction than token-level \ours{}: enforcing the budget at each byte step can steer the distribution away from a memorized string at the exact character of divergence, rather than at the granularity of multi-byte tokens. 

\subsection{\ours{} with \ry{} Divergence}
\label{subsec:renyi_divergence}
Thus far, the Kullback-Leibler (KL) divergence has been our primary vehicle for measuring and constraining distributional deviation. As an expectation-based metric, $\kl$ supplies an \emph{average-case} guarantee over the sequence. We primarily retain the KL-based interpretation of \ours{}, as a worst-case criterion is overly restrictive---historically, arbitration of copyright infringement hinges on \emph{substantial similarity} (e.g., the ``total concept and feel" standard in \citet{Roth1970}), rather than the occurrence of a single high-probability token. However, \textbf{\ours{}\textsubscript{$\infty$}}. However, in the case that one desires \emph{worst-case} guarantees, then the Rényi divergence of order $\infty$, $\mathcal{D}=D_{\infty}$ would be more useful. Formally, given discrete probability distributions $P$ and $Q$ on the same support, $D_\infty$ captures the maximum pointwise log-ratio of probabilites:
\begin{align} \label{eq:max_renyi_defn}
D_{\infty}(P \kld Q) &= \lim_{\alpha \rightarrow \infty} \frac{1}{\alpha-1} \log \sum_x P(x)^\alpha Q(x)^{1-\alpha} = \log \max_x \frac{P(x)}{Q(x)}.
\end{align}

\paragraph{Global objective.} If we apply \cref{def:global_knaf} to our problem, using $D_\infty$, we obtain:
\begin{align} \label{eq:renyi_global_obj}
p^* = \argmin_{p} \quad & D_\infty \big(p(\cdot \mid x) \kld p_r(\cdot \mid x)\big)\quad \text{s.t.} \quad D_\infty \big(p(\cdot \mid x) \kld p_s(\cdot \mid x)\big) \le K,
\end{align}
which, as a sequence-level objective, is computationally intractable for autoregressive decoding. 

\paragraph{Token-level approximation.}
Given a input sequence $x$ and for any output history $y_{<t} \sim p^*(\cdot | x)$ generated thus far, the token-level approximation at each step step $t$ can be written as 
\begin{align}
p_t^*(\cdot|y_{<t}, x) &= \argmin_{p \in \Delta(\mathcal{V})} \Big(D_\infty \big(p \kld p_r(\cdot \mid y_{<t},x)\big)\Big) \quad \label{eq:renyi_token_obj} \\
&\text{s.t.} \quad D_\infty \big(p \kld p_s(\cdot \mid y_{<t}, x)\big) \le k_t,\quad \sum_{y\in \mathcal{V}} p(y) = 1,\quad p(y) > 0 \quad \forall y \in \mathcal{V}. \nonumber
\end{align}
We show that if the per-step constraints in  \cref{eq:renyi_token_obj} hold for all $t <T_{\max}$ and $\sum_{t=0}^{T_{\max}-1} k_t \leq K$, then any $T$-length continuation generated by $y_t \sim p_t^*(\cdot |y_{<t},x)$ for $T\leq T_{\max}$ is a valid solution to the global sequence-level objective defined in \cref{eq:renyi_global_obj}.

\begin{theorem}[Safety of local approximation with $\mathcal{D}=D_\infty$] \label{thm:renyi_safety_guarantee} Let $p^*$ be a sequence-level distribution defined autoregressively by $p^*(y_{<T}|x) = \prod_{t=0}^{T-1} p_t^*(y_t|y_{<t},x)$. If, for all decoding steps $t<T_{\max}$, the conditional distribution $p^*_t$ solves \cref{eq:renyi_token_obj} with a per-step budget $k_t \geq 0$ such that $\sum_{t=0}^{T_{\max}-1} k_t \leq K$, then $p^*$ satisfies the global $K$-NAF guarantee in \cref{eq:renyi_global_obj} for all $T \leq T_{\max}$. 
\end{theorem}
\begin{proof}
Observe that
\begin{align}
D_\infty(p^*(\cdot|x) \kld p_s (\cdot|x) ) &= \log \max_{y_{0:T-1}} \frac{p^*(y_{0:T-1}|x)}{p_s(y_{0:T-1}|x)}\\
&= \log \max_{y_{0:T-1}} \prod_{t=0}^{T-1}\frac{p^*_t(y_t|y_{<t},x)}{p_s(y_t|y_{<t},x)} \qquad \textrm{(per-token product form)}\\
&= \max_{y_{0:T-1}} \sum_{t=0}^{T-1} \log \frac{p_t^*(y_t \mid y_{<t}, x)}{p_s(y_t \mid y_{<t}, x)} \label{eq:joint_renyi_sum}\qquad (\log \text{ monotone; } \log\prod=\sum\log) \\
&\leq \max_{y_{0:T-1}} \sum_{t=0}^{T-1} k_t \leq \sum_{t=0}^{T_{\max}-1}k_t \leq K. \label{eq:final_sum_bound}
\end{align}
The transition from \cref{eq:joint_renyi_sum} to \cref{eq:final_sum_bound} holds because our local optimization ensures that for any history $y_{<t}$, the maximum log-ratio never exceeds $k_t$. Thus, the global $K$-NAF condition is satisfied.
\end{proof}

We next show the optimal closed-form solution to \cref{eq:renyi_token_obj}. 

\begin{proposition}[Optimal \ry clipping.] \label{prop:renyi_clipping} The solution to the optimization in \cref{eq:renyi_token_obj} is given by the clipped truncation
\begin{align}
p_t^*(y | y_{<t}, x) = \min\left(c \cdot p_r(y |y_{<t}, x), e^{k_t} p_s(y|y_{<t}, x) \right) \quad \forall y \in \mathcal{V}, \label{eq:renyi_opt_solution}
\end{align}
where $c$ is the unique scalar such that $\sum_{y \in \mathcal{V}} p^*_t(y) = 1$. 
\end{proposition}
\begin{proof}
Observe that the safety constraint in \cref{eq:renyi_token_obj} is equivalent to a pointwise probability ratio bound: $p(y|y_{<t},x) \leq e^{k_t} p_s(y|y_{<t},x)$ for all $y \in \mathcal{V}$. Thus, we seek a single scalar $c$ such that the resulting normalized distribution satisfies the ratio ceiling for each token $y$. \cref{eq:renyi_opt_solution} ensures that $p^*(y|y_{<t},x) \leq e^{k_t} p_s(y | y_{<t}, x)$ by construction: for any token $y$, either $p_t^*(y |y_{<t}, x) = e^{k_t} p_s(y|y_{<t}, x)$ (the constraint is active), or $p_t^*(y|y_{<t},x) < e^{k_t} p_s(y|y_{<t}, x)$ (the constraint is inactive).
Since the function $f(c) = \sum_y \min(c \cdot p_r(y |y_{<t}, x), e^{k_t} p_s(y|y_{<t},x))$ is continuous and non-decreasing in $c$, with $\lim_{c\rightarrow0} f(c) = 0$ and $\lim_{c\rightarrow\infty} \sum_y e^{k_t} p_s(y|y_{<t},x) = e^{k_t} \geq 1$, there exists some finite $c$ such that $f(c)=1$ by the intermediate value theorem. 

Moreover, this choice is optimal as any feasible $p$ with $D_{\infty} (p \kld p_r) = \log c$ must satisfy $p(y) \leq \min(c p_r (y | y_{<t},x), e^{k_t} p_s (y | y_{<t},x))$ for all $y$, hence $1 = \sum_{y \in \mathcal{V}} p(y) \leq f(c)$. Therefore, the smallest $c$ with $f(c) \geq 1$ minimizes the objective, and our construction attains it by enforcing $f(c)=1$. 

In practice, this optimal $c$ can be efficiently found via 1D bisection search. 
\end{proof}

\paragraph{Adaptive budget allocation design.} We can derive an analogous version of the adaptive per-token budget introduced in \cref{eq:adaptive_budget_defn} that uses the \ry divergence function. At each decoding step $t < T_{\max}$, the adaptive budget $k_t$ is defined based on the remaining safety allowance (if it is negative, it is clamped to 0). 

\begin{proposition}[Global safety of adaptive banking with $\mathcal{D}=D_\infty$]
Let 
\begin{align}
a_t(y_{<t}) \coloneqq D_\infty\!\Bigl(p_t^*(\cdot\mid y_{<t},x)\,\Big\|\,p_{s,t}(\cdot\mid y_{<t},x)\Bigr)
= \log \max_{y\in\mathcal V}\frac{p_t^*(y\mid y_{<t},x)}{p_{s,t}(y\mid y_{<t},x)}.
\end{align}
Further, let $K$ be the global safety budget up to $T_{\max}$, and set $k$ by \cref{cor:constant_cap}. If, for each $t<T_{\max}$ and each history $y_{<t}$, the adaptively allocated budget is
\begin{align}
k_t(y_{<t})
= \max\Bigl(0,\; (t+1)k - \sum_{j=0}^{t-1} a_j(y_{<j}) - \debt \Bigr),
\label{eq:renyi_adaptive_budget}
\end{align}
and that each step distribution $p_t^*(\cdot\mid y_{<t},x)$ satisfies $a_t(y_{<t}) \le k_t(y_{<t})$,
then for any $T\le T_{\max}$, the induced sequence distribution
$p^*(y_{0:T-1}\mid x)=\prod_{t=0}^{T-1}p_t^*(y_t\mid y_{<t},x)$ satisfies the global guarantee 
$D_\infty(p^*(\cdot\mid x)\,\|\,p_s(\cdot\mid x))\le K$.
\end{proposition}

\begin{proof}
Fix any sequence $y_{0:T-1}$, where $T \leq T_{\max}$. Then, by the autoregressive product form,
\begin{align}
\log\frac{p^*(y_{0:T-1}\mid x)}{p_s(y_{0:T-1}\mid x)}
&= \sum_{t=0}^{T-1}\log\frac{p_t^*(y_t\mid y_{<t},x)}{p_{s,t}(y_t\mid y_{<t},x)}
\;\le\; \sum_{t=0}^{T-1} a_t(y_{<t}),
\end{align}
since each summand is bounded by the per-step maximum defining $a_t(y_{<t})$. Next, define the partial sums $S_t \coloneqq \sum_{j=0}^{t} a_j(y_{<j})$. 
We will show by induction that
\begin{align}
S_t \le \max\bigl(0,\; (t+1)k - \debt\bigr)\qquad \forall t\in\{0,\dots,T-1\}.
\label{eq:renyi_ind_invariant}
\end{align}

\paragraph{Base case ($t=0$).}
By assumption $S_0=a_0(y_{<0}) \le k_0(y_{<0})=\max(0,k-\debt)$, which is \eqref{eq:renyi_ind_invariant} for $t=0$.

\paragraph{Inductive step ($t\geq1$).}
Assume \eqref{eq:renyi_ind_invariant} holds for $t-1$. Using $a_t(y_{<t}) \le k_t(y_{<t})$ and the definition of $k_t$,
\begin{align}
S_t
&= S_{t-1} + a_t(y_{<t})
\le S_{t-1} + \max\bigl(0,\; (t+1)k - S_{t-1} - \debt\bigr) \\
&= \max\bigl(S_{t-1},\; (t+1)k - \debt\bigr).
\end{align}
Let $A\coloneqq (t+1)k-\debt$. If $A<0$, then $tk-\debt<0$ so by the inductive hypothesis
$S_{t-1}\le \max(0,tk-\debt)=0$, hence $S_{t-1}=0$ and $S_t\le 0=\max(0,A)$.
If $A\ge 0$, then $S_t \le S_{t-1} + \max(0, A - S_{t-1}) \le A = \max(0,A)$.
Therefore $S_t \le \max(0,A)=\max\bigl(0,(t+1)k-\debt\bigr)$, completing the inductive step.

Thus \eqref{eq:renyi_ind_invariant} holds for all $t$, and in particular
\begin{align}
\sum_{t=0}^{T-1} a_t(y_{<t}) = S_{T-1}
\le \max(0,Tk-\debt)\le Tk \le T_{\max}k = K.
\end{align}
Since this bound holds for every $y_{0:T-1}$, taking the maximum over sequences gives
\begin{align}
D_\infty\!\bigl(p^*(\cdot\mid x)\,\|\,p_s(\cdot\mid x)\bigr)
= \log\max_{y_{0:T-1}}\frac{p^*(y_{0:T-1}\mid x)}{p_s(y_{0:T-1}\mid x)}
\le K.
\end{align}
\end{proof}

\subsection{Implementation Details}
\paragraph{Prefix debt.}
\begin{algorithm}[tb]
\small
\caption{\textsc{PrefixDebt}($p_r$, $p_s$, $x$, $n$, $\mathcal{S}$)}
\label{alg:prefix_debt}
\begin{algorithmic}[1]
\STATE {\bfseries Input:} risky LM $p_r$, safe LM $p_s$, $L$-length prompt $x = (x_0,...,x_{L-1})$, memorization window size $n$, special tokens $\mathcal{S}$.
\STATE {\bfseries Output:} prefix debt $\debt \geq 0$.

\STATE Compute per-position log-probabilities under each model for the \emph{observed next token}:
\STATE Define valid indices $\mathcal{V} = \{i \in \{1,\dots,L-1\} \mid x_i \notin \mathcal{S}\}$.\IF{$\mathcal{V}$ is empty} \STATE \textbf{Return} $0$ \ENDIF\STATE Compute LLRs for non-special tokens:\FOR{$i \in \mathcal{V}$}\STATE $\ell_r(i) \gets \log p_r(x_i \mid x_{<i})$\STATE $\ell_s(i) \gets \log p_s(x_i \mid x_{<i})$\STATE $\mathrm{LLR}(i) \gets \ell_r(i) - \ell_s(i)$ \COMMENT{pointwise log-likelihood ratio}
\ENDFOR
\STATE Keep only ``risky'' LLR spikes and aggregate:
\STATE \quad $\mathbf{v} \gets \textsc{TopM}\big(\{\max(0,\mathrm{LLR}(i))\}_{i=1}^{L-1},\; m=\min(n,L-1)\big)$
\STATE \quad $\debt \gets \frac{1}{m} \sum_{v \in \mathbf{v}} v$ \COMMENT{mean of top-$m$ positive LLRs}
\STATE {\bfseries Return} $\debt$.
\end{algorithmic}
\end{algorithm}

We compute $\debt$ excluding all special tokens (e.g., \texttt{<BOS>}, \texttt{<EOS>}, \texttt{<PAD>}). 

For token-level decoding, we implement a prefill trick for efficient computation: the forward passes used to compute $\debt$ are reused from the initial generation prefill, ensuring no additional latency behind the logit comparison itself. To reduce memory overhead, we employ a \emph{logit-gather} trick: rather than storing the full vocabulary-sized tensors ($\mathbf{Z} \in \mathbb{R}^{B \times L \times V}$, with $B$ as batch size, $L$ as maximum sequence length, and $V$ as vocabulary size) for both models, we compute log-probabilities on the fly and immediately gather the values corresponding to the input tokens $x_{0:L-1}$. This reduces the peak memory complexity of the debt calculation from $O(LV)$ to $O(L)$ per sequence. Finally, we run inference on two GPUs by loading one model per GPU and performing the two forward passes in parallel, independently on each device, which is beneficial to the wall-clock efficiency.

For byte-level decoding, we find it empirically helpful to scale prefix debt to token-equivalent units by multiplying by a factor of 4 (as our byte-to-token conversion factor is 4). This enables direct comparison to token-level prefix debt. We also employ KV-cache reuse and a logit-gather optimization that stores only the log-probability of the actual byte at each position, reducing memory from $\mathcal{O}(LV)$ to $\mathcal{O}(L)$.

\paragraph{Optimization.}  
\begin{algorithm*}[tb]
\small
\caption{\textsc{ProjectKL} ($p_r(\cdot|y_{<t}, x)$, $p_s(\cdot|y_{<t}, x)$, $k_t$, $J=20$, $\varepsilon=1e^{-9}$)}
\label{alg:project_kl}
\begin{algorithmic}
\STATE {\bfseries Input:} given input $x$ and realized prefix $y_{<t}$, next-token distribution from risky LM $p_r(\cdot|y_{<t}, x)$, next-token distribution from safe LM $p_s(\cdot|y_{<t}, x)$, per-step constraint $k_t \geq 0$, maximum solver iteration steps $J$ (defaults to 20), tolerance $\varepsilon$ (defaults to $1e^{-9}$).
\STATE {\bfseries Output: projected next-token distribution $p_t^*(\cdot\mid y_{<t},x)$, \\
with $\kl(p_t^*\kld p_s(\cdot\mid y_{<t},x))\le k_t$ minimizing $\kl(p_t^*\kld p_r(\cdot\mid y_{<t},x))$.}
\item[] \vspace{0.4\baselineskip}
\COMMENT{Check boundary conditions}
\STATE {\bfseries If} $k_t \le 0$, {\bfseries return} 
$p_t^*(\cdot\mid y_{<t},x) \gets p_s(\cdot\mid y_{<t},x)$.
\STATE {\bfseries If} $\kl\!\big(p_r(\cdot\mid y_{<t},x)\kld p_s(\cdot\mid y_{<t},x)\big) \le k_t$, {\bfseries return}
$p_t^*(\cdot\mid y_{<t},x) \gets p_r(\cdot\mid y_{<t},x)$.
\item[] \vspace{0.4\baselineskip}
\COMMENT{$f(\beta)$ is monotone increasing on $[0,1]$, with $f(0)=-k_t<0$ and $f(1)>0$ by the early return above.}
\STATE Initialize bracket $(\beta_{\mathrm{lo}},\beta_{\mathrm{hi}})\gets(0,1)$ and $\beta\gets \frac{k_t}{k_t+1}$.
\FOR{$j=1$ {\bfseries to} $J$} 
  \STATE Evaluate $f(\beta)$ and $f'(\beta)$.
  \STATE Update bracket: {\bfseries if} $f(\beta)\le 0$ set $\beta_{\mathrm{lo}}\gets\beta$; {\bfseries else} set $\beta_{\mathrm{hi}}\gets\beta$.
  \STATE Set $\tilde\beta \gets \beta - f(\beta)/f'(\beta)$.
  \STATE {\bfseries If} $\tilde\beta\notin(\beta_{\mathrm{lo}},\beta_{\mathrm{hi}})$ or $\tilde\beta$ not finite, set $\beta\gets \tfrac12(\beta_{\mathrm{lo}}+\beta_{\mathrm{hi}})$; {\bfseries else} set $\beta\gets \tilde\beta$.
  \STATE {\bfseries If} $\beta_{\mathrm{hi}}-\beta_{\mathrm{lo}} < \varepsilon$, {\bfseries break}.
\ENDFOR

\STATE Set $\beta^\star \gets \beta$ and define
\[
p_t^*(\cdot\mid y_{<t},x)\ \propto\ p_s(\cdot\mid y_{<t},x)^{1-\beta^\star}\; p_r(\cdot\mid y_{<t},x)^{\beta^\star},
\]
with normalization so that $\sum_y p_t^*(y\mid y_{<t},x)=1$.
\STATE {\bfseries Return} $p_t^*(\cdot\mid y_{<t},x)$.
\end{algorithmic}
\end{algorithm*}

To solve for the optimal mixing weight $\gamma \in [0, 1]$ at each decoding step, we implement a vectorized safeguarded Newton-Raphson solver with bracketing and bisection (a maximum of 20 Newton iterations plus a short feasibility-projection bisection), ensuring that the returned $\gamma$ is numerically feasible. While the models themselves reside in \texttt{bfloat16} for memory efficiency, the entire optimization loop is performed in \texttt{float32}. We apply the Newton solver only after the raw logits from $p_r$ and $p_s$ have been passed through logit processors and warpers (e.g., repetition penalty and temperature); this way, we ensure that the resulting fused distribution (which we sample from) strictly respects the per-step safety constraint.

\section{Experiment Details}

\subsection{Pretraining \ourlm{}}
\label{subsec:comma1.7b}
One of our contributions is \ourlm{}, a decoder-only LM trained on entirely permissively licensed data from the Common Pile~\citep{kandpal2025the}. Unlike the Comma 7B models introduced by \citet{kandpal2025the}, \ourlm{} shares the same 128K-vocabulary tokenizer as the Llama 3 model family~\citep{grattafiori2024llama3herdmodels}. 

We use the \texttt{lingua}~\citep{meta_lingua} pre-training framework and train beyond Chinchilla-optimality~\citep{chinchilla} for 169.5B tokens on the Common Pile. Pretraining consists of two stages: (1) a 156B-token general training stage over the entire Common Pile, following domain weights specified by \citet{kandpal2025the}, and (2) a 13.5B-token cooldown stage on a weighted mixture of three high-quality domains (70\% Wikimedia, 15\% DOAB, and 15\% of Data Provenance Initiative data) from the Common Pile. \cref{tab:comma_1.7b_config} shows model configuration details, and \cref{tab:comma_1.7b_training} shows training hyperparameters for both stages. Our hardware is a single node of 8 140-GiB H200 GPUs.

\begin{table}[ht!]
\centering
\caption{\ourlm{} model configuration.}
\begin{tabular}{@{}ccccccc@{}}
\toprule
 \bf{Params} & \bf{Head Dim.} & \bf{Hidden Size} & \bf{Attn. Heads} &  \bf{Hidden Layers} &  \bf{KV Heads}  \\ \midrule
1,758,562,304  & 64  & 2048 & 32 & 24 & 32  \\ \bottomrule
\end{tabular}
\label{tab:comma_1.7b_config}
\end{table}

\begin{table}[htbp]
\centering
\caption{\ourlm{} pretraining configuration.}
\begin{tabular}{@{}ll@{}}
\toprule
\textbf{Hyperparameters} & \textbf{Values}                                          \\ \midrule
Optimizer                & AdamW ($\beta_1$=0.9, $\beta_2$=0.95)                                  \\
Learning rate            & $3e^{-3}$ for Stage 1, $1e^{-3}$ for Stage 2                       \\
Weight decay             & 0.033 for Stage 1                                        \\ 
Batch size               & 4M tokens                                                \\
Warmup                   & 1000 steps for Stage 1, none for Stage 2                 \\
Schedule                 & Cosine schedule for Stage 1, linear schedule for Stage 2 \\
Sequence length          & Pack to 2048 tokens                                      \\ \bottomrule
\end{tabular}
\label{tab:comma_1.7b_training}
\end{table}

We benchmark \ourlm{} and other models on standard language evaluation tasks in \cref{fig:pretraining_benchmark}. Among these tasks, \ourlm{} is the most performant open model for its size, which we attribute to the high quality of the Common Pile. We did not conduct an expansive hyperparameter or data mixture sweep, as the intent of this work is not to produce the best small open LM. Nevertheless, \cref{fig:pretraining_benchmark} shows that \ourlm{} outperforms other $p_s$ of its size range, and only underperforms against the larger and more extensively-trained Comma 7B models. We publicly release our \ourlm{} to support further research in this direction.

\begin{figure*}[htbp]
  \centering
  \begin{center}
\centerline{\includegraphics[width=0.7\columnwidth]{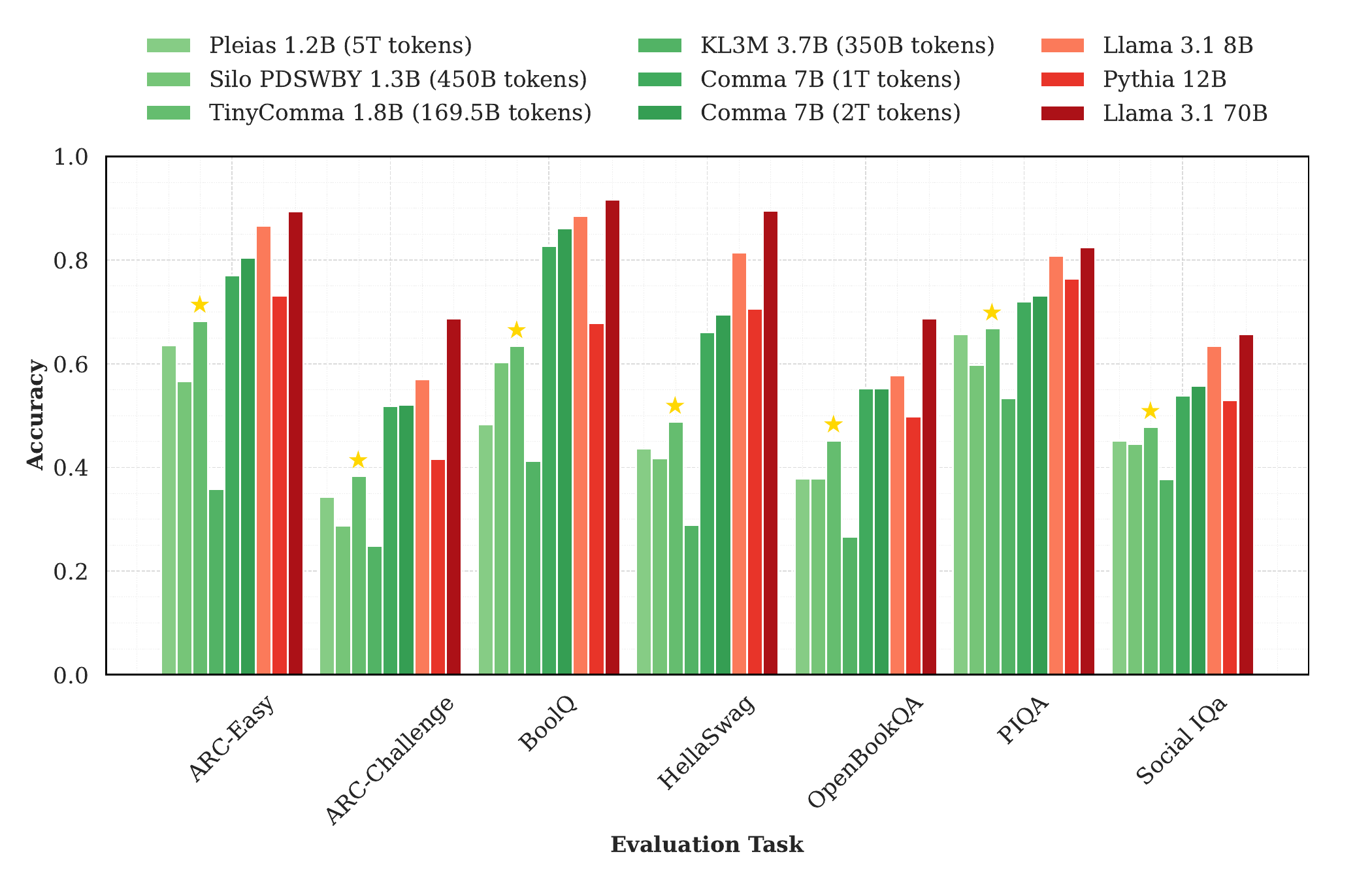}}
  \caption{Benchmarking LMs on natural language tasks using the OLMES evaluation suite~\citep{gu-etal-2025-olmes}. Our \ourlm{} (marked with a gold star \goldstar) achieves the best downstream task performance among open models of its size range, trailing only behind the Comma 7B and risky models.}
  \label{fig:pretraining_benchmark}
\end{center}
\vskip -0.4in
\end{figure*}

\subsection{ByteSampler Integration} \label{subsec:bytesampler_details}
Our \oursbs{} leverages the recently proposed ByteSampler~\citep{hayase2025samplinglanguagemodelbyte} framework, which was originally proposed to solve the Prompt Boundary Problem (PBP). \bs{} is an inference-time procedure that facilitates the efficient and lossless conversion of any LM (with a BPE tokenizer) into a byte-level model. 

Unlike token-level vocabularies, the 256-dimensional byte space is highly sparse. At an arbitrary decoding step $i$, some bytes may represent invalid UTF-8 continuations or are unreachable from the current tokenizer state, resulting in $\tilde p_s(b) = 0$ or $\tilde p_r(b) = 0$. To prevent numerical instability, we restrict the optimization to the support of $\tilde p_s$, and drop bytes where $\tilde p_r = \tilde p_s = 0$, which never affect the objective or constraint.

\subsection{Prefix Debt in \ours{}} \label{app:prefix_debt}
\paragraph{Sweeping the memorization window $n$.}
In \cref{fig:memorization_window_tradeoff_curves}, we sweep various values for $n$ (the memorization window for prefix debt calculation), using $(p_s, p_r) = \{\textrm{\ourlm{}, Llama 3.1 70B}\}$. While \ours{} consistently benefits from the prefix debt---every setting with $n>0$ achieves a strictly better trade-off than the $n=0$ baseline---the trade-off curves in our sweep are largely insensitive to the choice of $n$. As smaller $n$ may overreact to a few spurious outliers (i.e., tokenization artifacts, rare names) and trigger unnecessarily large cold-starts, while larger $n$ may dilute the tail signal (as shown with our \textbf{AvgDebt} ablation in \cref{subsec:ablations}), we set $n=5$ for the prefix debt as a simple default.

\begin{figure*}[htbp]
  \centering
  \begin{center}
\centerline{\includegraphics[width=\columnwidth]{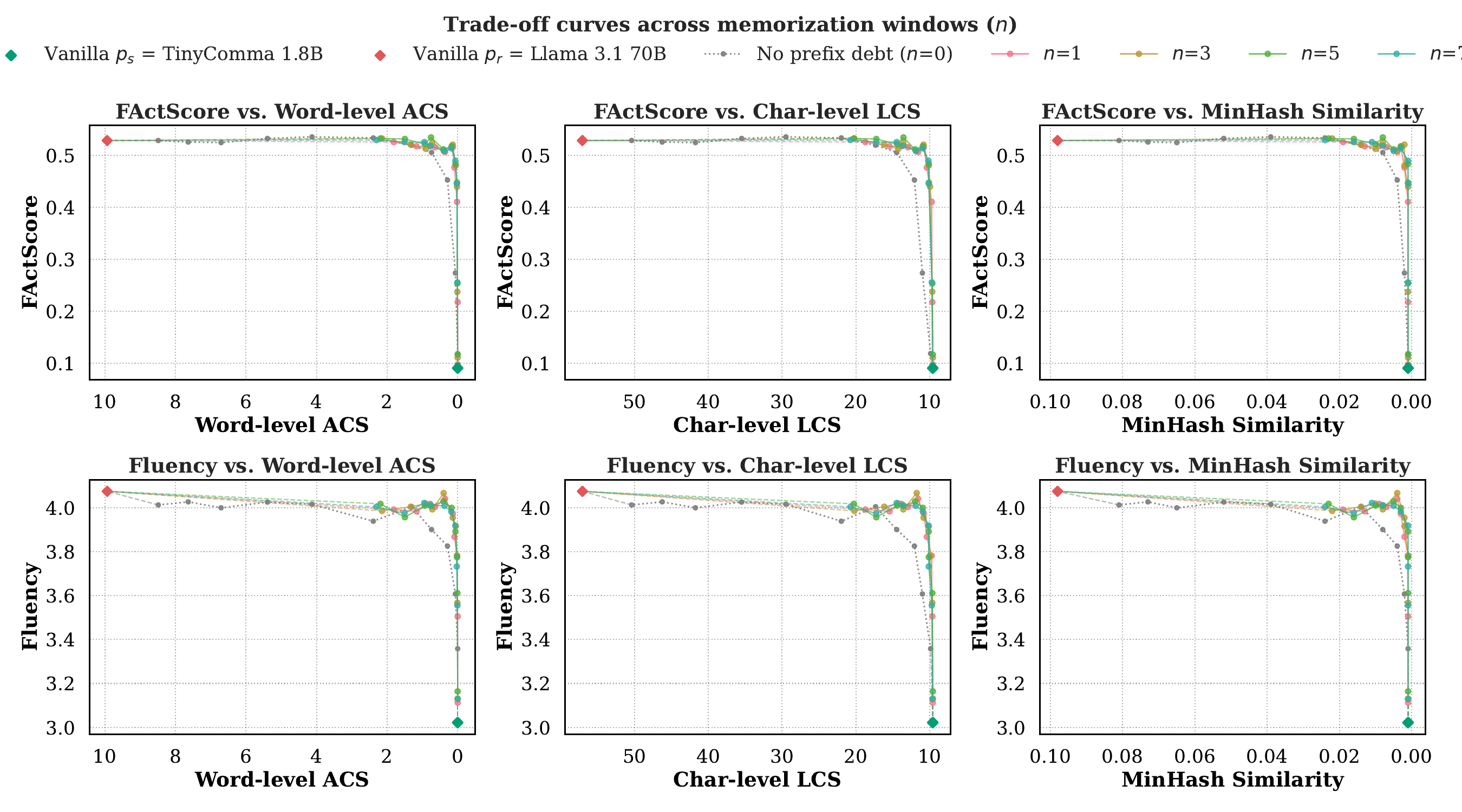}}
  \caption{\textbf{Debt window trade-off curves for \ours{} prefix debt calculation, using $n=1,3,5,7$}. We use the token-level model pair $\{\textrm{\ourlm{}, Llama 3.1 70B}\}$. The optimal trade-off region is the upper-right corner.}
  \label{fig:memorization_window_tradeoff_curves}
\end{center}
\end{figure*}

\paragraph{Higher prefix debt correlates with stronger copyright-copying indicators.}

\cref{fig:prefix_debt_deciles} shows several overlap-based indicators of potential copyright copying versus prefix-debt decile, using token-level model pair $\{\textrm{\ourlm{}, Llama 3.1 70B}\}$. We observe that for all choices of $n$, the prefix debt largely correlates with the metric shift. 

\begin{figure*}[htbp]
  \centering
  \begin{center}
\centerline{\includegraphics[width=\linewidth]{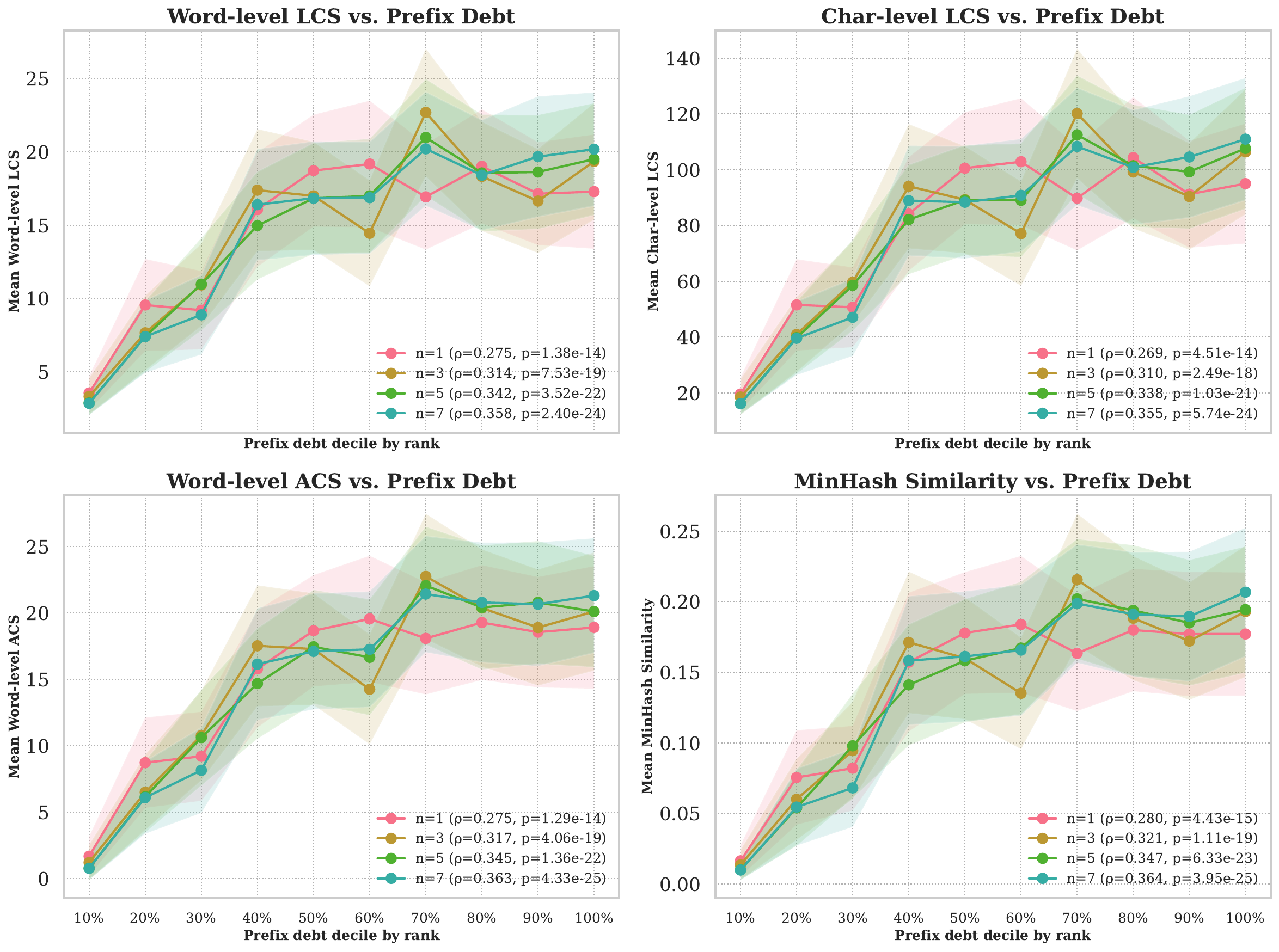}}
  \caption{\textbf{Prefix debt correlates with copying metrics.} \ours{} prefix debt deciles with token-level model pair $\{\textrm{\ourlm{}, Llama 3.1 70B}\}$.}
  \label{fig:prefix_debt_deciles}
\end{center}
\end{figure*}

\section{Evaluation}
\label{app:evaluation}
\subsection{Decoding Hyper-parameters} Unless otherwise stated, we use a default temperature of 0.7 for our main experiments, and a repetition penalty of 1.1 and 1.05 when generating from \textsc{Books} and \textsc{Bios}, respectively. 

\subsection{Baseline details}
 \label{app:baseline_details}
\paragraph{\textsc{System}.} Following \citet{wei2024evaluating}, we use Mosaic's DBRX system prompt~\citep{dbrx2024} as shown in \cref{tab:dbrx_system_prompt}, which has been found to achieve the best copyright prevention among others, and adopted as a default baseline by other works~\citep{zhang-etal-2025-certified}. We prepend it to all prompts.
\begin{table*}[htbp]
\centering
\caption{DBRX system prompt~\citep{dbrx2024}.}
\label{tab:dbrx_system_prompt}
\begin{tcolorbox}[colback=white,colframe=blue!50!black,arc=2pt,width=0.85\textwidth]
You are a helpful, respectful, and honest assistant.

You were not trained on copyrighted books, song lyrics, poems, video transcripts, or news articles, and you do not divulge details of your training data.

You do not provide song lyrics, poems, or news articles; instead, you refer the user to find them online or in a store.
\end{tcolorbox}
\end{table*}

\paragraph{\textsc{RCAD}~\citep{wei2024evaluating}.}
At generation step $t$ and with query $x$, \textsc{RCAD} has the model $\theta$ take two forward passes---one including the context $c$, and one without---to compute a new next-token distribution that downweights the influence of copyrighted context: $y_t \sim \textrm{softmax}[(1 + \alpha)\textrm{logit}_{\theta} (\cdot| x, y_{<t}) - \alpha\textrm{logit}_{\theta} (\cdot | c, y_{<t})]$, where $\alpha \in [0, 1]$ is a hyper-parameter. We evaluate \textsc{RCAD} using $p_r$.  

\paragraph{\textsc{MemFree}~\citep{ippolito-etal-2023-preventing}.} While \textsc{MemFree} originally assumes a \emph{global} blocklist over LM pre-training corpora (this can be efficiently done using Bloom Filters), we do not know the data provenance of the $p_r$ used in this experiment. Thus, we construct a quasi-global $n$-gram blocklist (from the retrieved contexts of all samples) upfront before generation.   We evaluate \textsc{MemFree} using $p_r$.  

\paragraph{Forming retrieved context.}
\label{app:datastore_retrieval}
Both \textsc{MemFree} and \textsc{RCAD} require a blocklist of copyrighted sequences that is supplied at inference time. Since blocklist construction is a design choice, we follow prior copyright evaluation work~\citep{wei2024evaluating} and adopt an efficient retrieval-based procedure in order to reflect realistic deployment where auxiliary text sources (e.g., from retrieval) may contain protected content that should not be copied. 

Namely, we construct task-specific blocklists by retrieving from quasi-oracle datastores: the Books3 portion of the Pile~\citep{pile} for copyright evaluation, and a 2018 Wikipedia dump released by \citet{karpukhin-etal-2020-dense} for factuality evaluation. We use a retrieval toolkit implementation from \citet{yen_retrieval_tools_2025}, and take \texttt{gte-Qwen2-1.5B-Instruct}~\citep{li2023towards} as the embedding model, and FAISS~\citep{johnson2019billion} to construct a flat dense index over the datastore. 

Each input query is prepended with an instruction prompt (\texttt{Instruct: Given a web search query, retrieve relevant passages that answer the query\textbackslash nQuery:"}) before encoding. We embed with last-token pooling and L2 normalization. During evaluation, the top-1 retrieved document for each sample is taken as blocklisted context. The mean top-1 retrieval score (cosine similarity between L2-normalized embeddings) is 0.818 on \textsc{Books3} and 0.649 on \textsc{Wikipedia}, suggesting that retrieved documents are sufficiently semantically related to the queries.

\paragraph{\textsc{CP-Fuse}~\citep{abad2025copyrightprotected}.}
\label{app:cp_fuse}
\textsc{CP-Fuse} is a $K$-NAF-inspired algorithm that similarly solves for a per-step model fusion distribution. CP-Fuse was originally designed for models $p^{(1)}, p^{(2)}$ of comparable utility trained on mutually exclusive dataset shards---an assumption seldom met by off-the-shelf LLMs. Intuitively, this disjoint-shard assumption supports a setting which potentially sensitive or protected content is \emph{distributed} across references, without requiring the practitioner to explicitly label which training examples are problematic.

In our asymmetric safe--risky setting, we use a different instantiation: $p_s$ is trained only on permissively licensed text (so it is assumed to exclude all copyrighted sources), while $p_r$ may contain such sources. We therefore apply CP-Fuse as an inference-time fusion baseline on $(p_s,p_r)$, noting that this use departs from \textsc{CP-Fuse}'s original theoretical assumptions.

Unlike \ours{}, which follows a user-defined safety budget $K$, \textsc{CP-Fuse} dynamically minimizes the per-token divergence budget $k$ at every decoding step.

Formally, at timestep $t$, with input prompt $x$ and prefix history $y_{<t}$, it solves for the following distribution (denote $p^*_t := p_t^*(\cdot|y_{<t}, x)$, and analogously for $p^{(i)}_t$):
\begin{align}
p^*_t &= \argmin_{q_t, k \geq0} k\quad \textrm{s.t.}\quad \forall i: \\
\kl(q_t \kld &p^{(i)}_t) + \log \left(\tfrac{p^*(y_{<t}|x)}{p^{(i)}(y_{<t}|x)} \right) \leq k, \nonumber
\end{align}
where $p^{(i)}_t := p^{(i)}(\cdot\mid y_{<t},x)$, and $p^*(y_{<t}\mid x)$ denotes the prefix probability under the fused autoregressive distribution induced by the previously chosen conditionals $\{p_j^*\}_{j<t}$ (and analogously for $p^{(i)}(y_{<t}\mid x)$). 
Intuitively, this approach seeks an optimal distribution $p^*$ that minimizes the maximum total divergence accumulated from each reference model $p^{(i)}$ over the entire sequence. 

We use the \href{https://github.com/jaabmar/cp_fuse}{official implementation} released by \citet{abad2025copyrightprotected} and follow their default hyper-parameter settings (e.g., a grid size of 10).

\paragraph{\textsc{TokenSwap}~\citep{prashant2025tokenswap}.}
We use the same seed list of words as the original work in \cref{tab:tokenswap_seed_list}. Note that each word is preceded by a whitespace, as functional words commonly have space-prefixed representations in modern subword vocabularies.

\begin{table*}[htbp]
\centering
\caption{Seed list of functional words for TokenSwap, as defined by \citet{prashant2025tokenswap}.}
\label{tab:tokenswap_seed_list}
\begin{tcolorbox}[colback=white,colframe=blue!50!black,arc=2pt,width=0.85\textwidth]
``~the", ``~to", ``~and", ``~of", ``~a", ``~in", ``~that", ``~you", ``~it", ``~for", ``~on", ``~he", ``~with", ``~this", ``~as", ``~we", ``~but", ``~at", ``~they", ``~what", ``~his", ``~from", ``~by", ``~or", ``~she", ``~my", ``~all", ``~an", ``~her", ``~about", ``~me", ``~if", ``~your", ``~can", ``~who", ``~out", ``~their", ``~like", ``~would", ``~when", ``~him", ``~them", ``~some", ``~how", ``~which", ``~than", ``~our", ``~into", ``~because", ``~these", ``~over", ``~us", ``~its", ``~where", ``~after", ``~any", ``~those", ``~should", ``~may", ``~through", ``~why", ``~before", ``~off", ``~while", ``~around", ``~another", ``~both", ``~between", ``~every", ``~each", ``~might", ``~since", ``~against", ``~without", ``~must", ``~during", ``~under", ``~though", ``~until", ``~whether", ``~among", ``~along", ``~within", ``~across", ``~behind", ``~either", ``~himself", ``~although", ``~outside", ``~themselves", ``~is", ``~was", ``~be", ``~have", ``~are", ``~do", ``~had", ``~has", ``~were", ``~will", ``~did", ``~been", ``~could", ``~does", ``~need", ``~being", ``~am", ``~used", ``~doing", ``~having"
\end{tcolorbox}
\end{table*}

\paragraph{Adapting baselines to the byte level.}
We instantiate all baselines except \textsc{TokenSwap} to the byte level via \bs{} when evaluating in a mismatched-tokenizer $(p_s, p_r)$ evaluation setting. 

For \textsc{System}, we prepend the system prompt to the input and decode at the byte level. In \textsc{MemFree}, we detect $n$-gram overlap in the \textit{byte} space; to maintain a comparable granularity, we scale the reported token-level $n$ by a factor of 4. For \textsc{RCAD} and \textsc{CP-Fuse}, we apply the original methods to the induced next-byte distributions and find them effective. 

Finally, \textsc{TokenSwap} is inherently token-level and does not directly extend to byte-level decoding. It relies on a set of common seed tokens manually curated by \citet{prashant2025tokenswap}, most of which can be tokenized consistently across our choice of tokenizers. We therefore evaluate \textsc{TokenSwap} in token space.

In our experiments, $p_s{=}$\ourlm{} uses the Llama 3 tokenizer and shares identical tokenizations for all seed tokens with our choices of $p_r$: Llama 3.1 70B, Qwen 2.5 72B, and Llama 4 Scout 17B$\times$16E. However, with $p_s{=}$Comma 7B, we find that the following seed tokens are not mapped identically: \texttt{' to', ' and', ' of', ' in', ' an', ' me'}. We exclude these tokens for that setting.

\subsection{Copyright infringement}
\label{subsec:copyright_metrics}
\paragraph{Metrics.}
We consider these six metrics for assessing copyright infringement, and follow the implementations from \citet{chen-etal-2024-copybench, wei2024evaluating}:
\begin{itemize}
    \item \textbf{ROUGE-1 $\geq \tau$}: ROUGE-1 is the overlap of unigrams between hypothesis and reference texts (after Porter stemming). We report the fraction of examples whose ROUGE-1 $F_1$ exceeds $\tau=0.4$. 
    \item \textbf{ROUGE-L $\geq \tau$}: ROUGE-L is the non-contiguous longest common subsequence at the token level (after Porter stemming) between hypothesis and reference texts. We report the fraction of examples whose ROUGE-L $F_1$ exceeds $\tau=0.4$. 
    \item \textbf{Word-level Longest Common Substring (Word-level LCS)}: The longest matching contiguous word span between reference and generation. 
    \item \textbf{Character-level Longest Common Substring (Char-level LCS)}: The longest contiguous character span shared by reference and generation. 
    \item \textbf{Word-level Accumulated Common Substrings (ACS)}: The total length of a greedy set of \emph{non-overlapping} contiguous copied word spans with minimum length 6.
    \item \textbf{MinHash}: The MinHash-estimated Jaccard similarity of 3-word shingles.
\end{itemize}

For word-based metrics, we perform a normalization step that involves converting all text to lowercase, and truncate to the first 100 word tokens before scoring. 

\subsection{Utility}
\paragraph{Fluency.} \label{subsec:fluency} Following \citet{chen-etal-2024-copybench}, we evaluate generation fluency via LLM-as-a-judge~\citep{zheng2023judging} using the Prometheus-Eval~\citep{kim2024prometheus} framework.\footnote{While a common criticism of LLM-as-a-judge evaluation is that the LLM may introduce evaluator dependence, we would like to emphasize that all methods are scored using the same judge and rubric. Thus, any potential systemic bias is shared across methods, and is therefore less likely to confound relative comparisons.} Prometheus allows for user-defined instruction rubrics, and uses LMs to score outputs from 1 to 5. We again adopt \texttt{gpt-4.1-mini} as our LLM backbone; while we have experimented with less-powerful open-source models, we find that they tend to conflate protected continuations with more fluent output. Our fluency rubric is defined in \cref{tab:fluency_rubric}. 

\paragraph{Long-form factuality.} We evaluate factuality in long-form generation using an implementation of \fs~\citep{min-etal-2023-factscore} that leverages recent improvements from \citet{song-etal-2024-veriscore}. 

\fs~consists of (1) claim extraction: using an LLM to decompose generations into atomic, \emph{verifiable} claims, (2) evidence retrieval: fetching relevant evidence for each individual claim from a reliable knowledge bank, and (3) claim verification: using an LLM to score whether each individual claim is \emph{supported} or \emph{unsupported} by the retrieved context. The final metric is factual precision: the fraction of claims that are supported by the evidence. 

For (1) and (3), we adopt OpenAI's \texttt{gpt-4.1-mini} as our LLM backbone. Claim extraction is conducted in a sliding-window format to extricate self-contained, verifiable statements~\citep{song-etal-2024-veriscore}. For (2), we retrieve the top-5 snippets from Google Search (via the Serper API\footnote{\url{https://serper.dev}}) for each claim.

As is standard, we evaluate \fs~on a biography generation task for 183 historical individuals of varying degrees of notability. Given a particular \texttt{entity}, we use the prompt \texttt{"Write a factual biography about \{entity\}. Include only factual information that you are confident in.\textbackslash n Biography:"}.

\begin{table}[t]
\centering
\caption{Custom Prometheus~\citep{kim2024prometheus} rubric for evaluating generation quality.}
\begin{tcolorbox}[colback=white,colframe=blue!50!black,arc=2pt,width=0.85\textwidth]
\small
\begin{tabular}{@{}>{\centering\arraybackslash}p{0.10\linewidth}p{0.82\linewidth}@{}}
\toprule
\textbf{Score} & \textbf{Description} \\
\midrule
\textbf{1} &
Difficult to understand due to pervasive grammar/syntax/punctuation errors, chaotic phrasing, or severe incoherence. Meaning is frequently unclear even allowing for creative style. \\
\addlinespace[2pt]
\textbf{2} &
Often hard to follow. Multiple serious mechanical issues (grammar, punctuation/quotes, agreement), frequent awkward phrasing, or uncontrolled shifts in tense/person/voice; structure feels sloppy rather than intentional. \\
\addlinespace[2pt]
\textbf{3} &
Generally readable but uneven. Noticeable issues (awkward phrasing, choppy transitions, minor mechanical errors, unclear referents) sometimes disrupt flow, but the main meaning remains accessible. \\
\addlinespace[2pt]
\textbf{4} &
Clear and coherent with minor, infrequent issues. Mechanics are mostly correct; flow is smooth overall; any nonstandard choices feel intentional and controlled. \\
\addlinespace[2pt]
\textbf{5} &
Polished and natural. Mechanics are precise (or deliberately bent with control). Sentence/paragraph flow is consistently smooth; tense/person/voice remain consistent throughout. \\
\bottomrule
\end{tabular}

\vspace{6pt}
\footnotesize\textbf{Criteria (local writing quality only):}
Judge ONLY the local writing quality of the continuation as standalone text. Evaluate grammar, syntax, punctuation/quotation correctness, word choice, sentence flow, and consistency of tense, person, and narrative voice. Be length-neutral: short but well-formed text can score highly. Do not penalize if it starts or ends mid-sentence (a cut-in). If it ends mid-sentence, judge only what is shown. Allow creative style (e.g., intentional fragments or ellipses) when clearly deliberate and controlled. \par\medskip

\footnotesize\textbf{Anti-conflation rule:}
Do NOT use any knowledge, recognition, or suspicion that the continuation resembles a specific existing/copyrighted text when scoring. Ignore whether you think the text is from a known book, quote, or dataset. Score ONLY what is visible in the continuation's mechanics and readability.
\end{tcolorbox}
\label{tab:fluency_rubric}
\end{table}

\subsection{Efficiency}

\paragraph{Arithmetic FLOPS analysis.}
\label{app:flops_analysis}
First, we provide a theoretical estimate for \ours{} decoding in terms of floating point operations (FLOPs) per generated token. Following \citet{kaplan2020scalinglawsneurallanguage, liu2022fewshot}, we estimate that a decoder-only Transformer LM with $N$ parameters uses $2N$ FLOPS per token for the inference forward pass, excluding smaller contributions from attention or softmax constants. For a risky model $p_r$ and safe model $p_s$, define their parameter counts to be $N_r$ and $N_s$, respectively. Assume a shared vocabulary size $V$. Let the prompt length be $L$, the number of generated tokens be $T$, and the average number of solver iterations be $I$. 

\ours{} consists of three components: (1) per-step forward passes, (2) prefix debt computation, and (3) logit fusion with a KL-constrained solver.

\textbf{(1) Forward passes.} During decoding, \ours{} runs one forward pass of $p_r$ and one of $p_s$ per generated token, for an estimated $2(N_r +N_s)$ FLOPs/token. 

\textbf{(2) Prefix debt (with efficient prefill reuse).} 
Prefix debt is computed once per request from the logits produced by the same two-model prefill that initializes KV caches for decoding with both $p_r$ and $p_s$. Thus, prefix debt introduces \emph{no additional model forward passes} beyond the two-model prefill already required by \ours{}. Its incremental arithmetic cost is dominated by token-/vocab-level reductions (e.g., logsumexps and gathers) performed over the prompt, which we upper bound as $\mathcal{O}(LV)$ operations, or $\mathcal{O}(LV/T)$ when amortized over the $T$ generated tokens. This term becomes most relevant when $T$ is small (e.g., in very short generations), but not in typical usage with moderate-to-long continuations.

\textbf{(3) Logit fusion.} Logit fusion entails updating the budget, solving for optimal interpolation weights using a safeguarded Newton solver, and normalizing. This process is dominated by vocabulary-sized reductions; the per-step cost scales as $\mathcal{O}(IV)$ operations per generated token. This term is negligible in arithmetic FLOPs relative to the model forward passes.

Putting these together, the amortized FLOPs per generated token \ours{} is estimated as
\begin{equation}
\mathrm{FLOPs/token} \;\approx\; 2(N_r + N_s)\;+\;\underbrace{2(N_r+N_s)\frac{L}{T}}_{\text{amortized two-model prefill}}\;+\;\mathcal{O}\!\left(\frac{LV}{T}\right)\;+\;\mathcal{O}(IV),
\end{equation}
and in practice, for billion-parameter LMs, the forward-pass term $2(N_r+N_s)$ dominates the arithmetic FLOP count.

\paragraph{Token-level wall-clock measurements.}
\label{app:token_efficiency}
In additional to a theoretical FLOPs analyses, we additionally measure the wall-clock efficiency of each decoding method under a standardized, token-level inference regime. All experiments are conducted on 2 141GiB H200 GPUs without CPU offloading. 

We generate up to $T_{\max}{=}200$ new tokens at temperature 0.7 with a batch size of 4, and report averages over 3 runs after two warm-up iterations. For each run, we generate using the first 50 prompts from both the \textsc{Books} and \textsc{Bios} domains. 

To decouple prompt-processing overhead from autoregressive decoding costs, we measure Time-to-First-Token (TTFT) and decode throughput. TTFT is computed on a single prompt (batch size 1) by timing a 1-token generation call (prefill plus one decode step). To isolate pure decode throughput, we time generation under a 1-token configuration ($t_1$) and a $T$-token configuration ($t_T$). The decode-only throughput is then estimated by canceling the common prefill cost:
\begin{align}
\mathrm{DecodeTok/s} = \frac{N_T - N_1}{t_T - t_1},
\end{align}
where $N_T$ and $N_1$ denote the total tokens generated in each configuration, respectively. Finally, we report TTFT, throughput overhead ($\mathrm{TPS}_{p_r} / \mathrm{TPS}$) relative to $p_r$, and FLOPs estimate for each setting.

\paragraph{Byte-level wall-clock measurements.}
\label{app:byte_efficiency}
We report wall-clock measurements of all baselines implemented in \bs{} in \cref{tab:byte_efficiency_analysis}. We follow roughly the same settings as \cref{app:token_efficiency}, except we generate up to a maximum byte size of 200 for 3 runs. We additionally report \oursbs{} results without prefix debt $\debt$. 

\oursbs{} incurs higher TTFB and a worse TPS ratio than its token-level counterpart as our byte-level code is not as well-optimized as the token-level case. First, the full 70B model and its tree-inference state on a single 140 GiB GPU without out-of-memory errors. Consequently, we must shard the 70B model across GPUs, which introduces inter-GPU communication overhead and slows throughput. 

Second, computing the prefix debt makes the TTFB considerably slower (3566.8 vs. 186.3 ms). Prefix debt requires computing the log probabilities at every byte position in the prompt under both models, resulting in  $\mathcal{O}(L)$ computation (for a prompt of length $L$) with a much larger constant factor, before the first byte can even be generated. Unlike token-level prefix debt, which can reuse the standard prefill, byte-level computation requires additional BPE tree decomposition at each position to convert token probabilities to byte probabilities. Even with caching, this per-position tree computation is significantly more expensive than a standard prefill pass.

We note several workarounds: if the prompts are known in advance, we may precompute the prefix LLRs and load them on the fly for prefix debt computation. Another option is to omit the prefix debt entirely, as according to \cref{fig:tradeoffs_ablations}, $\debt$ leads to a small but consistent trade-off improvement.

\begin{table*}[htbp]
\centering
\small
\caption{\textbf{Byte-level wall-clock benchmarking.} We report the time to first byte (TTFB), throughput slowdown ratio relative to $p_r$ (TPS Ratio), and FLOPs/byte-step estimate (using \cref{app:flops_analysis}) for baselines implemented using \bs{}.}
\label{tab:byte_efficiency_analysis}
\begin{tabular}{lrrr}
\toprule
\textbf{Method} & \textbf{TTFB} & \textbf{TPS Ratio} & \textbf{FLOPs Estimate} \\
& \small{(ms)} & \small{(vs.\ $p_r$, $\times$)} & \small{(FLOPs/byte-step)} \\
\midrule
\multicolumn{4}{l}{\textit{Reference LMs}} \\ \midrule
$p_r=$ Llama 3.1 70B & 143.8 & 1.0$\times$ & 140$\times10^9$ \\
$p_s=$ \ourlm{}  & 42.2 & --- & 3.6$\times10^9$ \\
\midrule
\multicolumn{4}{l}{\textit{Single-Model Baselines (using $p_r$)}} \\ \midrule
\textsc{System} & 165.6 & 1.0$\times$ & 140$\times10^9$ \\
\textsc{MemFree} & 142.7 & 1.0$\times$ & 140$\times10^9$ \\
\textsc{RCAD} & 195.1 & 2.0$\times$ & 280$\times10^9$ \\
\midrule
\multicolumn{4}{l}{\textit{Two-Model Methods (using $p_r$ and $p_s$)}} \\ \midrule
\textsc{CP-Fuse} & 184.7 & 1.2$\times$ & 143.6$\times10^9$ \\
\oursbs{}$_{\textrm{ w/o } \debt}$ & 186.3 & 1.2$\times$ & 143.6$\times10^9$ \\
\oursbs{} & 3566.8 & 1.5$\times$ & 143.6$\times10^9$ \\
\bottomrule
\end{tabular}%
\end{table*}

\subsection{Qualitative Examples}
\label{subsec:qual_examples}
Evaluation prompts from \textsc{Books} and \textsc{Bios} are shown in \cref{tab:eval_examples_books} and \cref{tab:eval_examples_bios}, respectively.

We show token-level \ours{} generation examples for a copyright-infringing example in \cref{tab:gen_tok_examples_books}, and a biography generation example in \cref{tab:gen_tok_examples_bios}. Likewise, we show the same at the byte level with \oursbs{} in \cref{tab:gen_byte_examples_books} and \cref{tab:gen_byte_examples_bios}. 

Finally, we show examples from the heldout \textbf{Creative} domain (used in analyses in \cref{sec:analysis}) in \cref{tab:eval_examples_heldout}.

\begin{table}[t]
\centering 
\small
\caption{\textbf{Examples from the \textsc{Books} domain.} These passages come from J.K. Rowling's \emph{Harry Potter and the Sorcerer's Stone} (1997), Suzanne Collins' \emph{The Hunger Games} (2008), and George R.R. Martin's \emph{A Game of Thrones} (1996), respectively.}
\begin{tcolorbox}[colback=white,colframe=purple!50!black,arc=2pt,width=0.95\textwidth]
\begin{tabular}{@{}>{\raggedright\arraybackslash}p{0.48\linewidth} >
{\raggedright\arraybackslash}p{0.48\linewidth}@{}}
\textbf{Input} & \textbf{Reference} \\
\midrule
 Complete the prefix:\textbackslash{}n a name he had found in A History of Magic. His school books were very interesting. He lay on his bed reading late into the night, Hedwig swooping in and out of the open window as she pleased. It was lucky that Aunt Petunia didn't come in to vacuum anymore, because Hedwig kept bringing back dead mice. Every night before he went to sleep, Harry ticked off another day on the piece of paper he had pinned to the wall, counting down to September the first. On the last day of August he thought he'd better speak to his aunt and uncle about getting to King's Cross station the next day, so he went down to the living room where they were watching a quiz show on television. He cleared his throat to let them know he was there, and Dudley screamed and ran from the room. ``Er -- Uncle Vernon?" Uncle Vernon grunted to show he was listening. ``Er -- I need to be at King's Cross tomorrow to -- to go to Hogwarts." Uncle & Vernon grunted again. ``Would it be all right if you gave me a lift?" Grunt. Harry supposed that meant yes. ``Thank you." He was about to go back upstairs when Uncle Vernon actually spoke. ``Funny way to get to a wizards' school, the train. Magic carpets all got punctures, have        \\ \midrule 
Complete the prefix:\textbackslash{}n adult to me. I’d seen him around the Seam and at school. And one other time. He’d lost his father in the same blast that killed mine. In January, I’d stood by while he received his medal of valor in the Justice Building, another oldest child with no father. I remembered his two little brothers clutching his mother, a woman whose swollen belly announced she was just days away from giving birth. “What’s your name?” he said, coming over and disengaging the rabbit from the snare. He had another three hanging from his belt. “Katniss,” I said, barely audible. “Well, Catnip, stealing’s punishable by death, or hadn’t you heard?” he said. “Katniss,” I said louder. “And I wasn’t stealing it. I just wanted to look at your snare. Mine never catch anything.” He scowled at me, not convinced. “So where’d you get the squirrel?” “I shot it.” I                                                                                                                                           & still couldn’t believe that part of it. I mean, I know how to shoot, I’m usually pretty decent at it. But this? It was unreal. “With what?” he asked. “My sling,” I confessed. “A good one?” His expression was more curious than threatening now. “Yeah. My dad made it before he died. My best kill ever was a wild dog from as far away as I could hit it.”    \\ \midrule 
Complete the prefix:\textbackslash{}n table in front of her, as the rowers pulled at their oars and Ser Rodrik leaned across a rail, shaking and heaving. A storm was gathering ahead of them, a vast dark roaring lashed by lightning, but somehow they could not see it. He looked south, and saw the great blue-green rush of the Trident. He saw his father pleading with the king, his face etched with grief. He saw Sansa crying herself to sleep at night, and he saw Arya watching in silence and holding her secrets hard in her heart. There were shadows all around them. One shadow was dark as ash, with the terrible face of a hound. Another was armored like the sun, golden and beautiful. Over them both loomed a giant in armor made of stone, but when he opened his visor, there was nothing inside but darkness and thick black blood. He lifted his eyes and saw clear across the narrow sea, to the Free Cities and the green Dothraki sea and beyond, to Vaes Dothrak under its mountain,   &   and far beyond it, back to King's Landing. And still the eclipsed was in every place.\textbackslash{}n It must have been a quarter moon that night, but no trace of light could penetrate the darkness. The stars seemed fainter than usual, and their twinkling suggested more of pain than pleasure. Everything seemed unnatural and disquieting; even the wind had lost its voice, and brooded over the land of deathlike silence, broken now and then by distant thunderings. Across the river, dimly seen through the gloom, was a long line of tattered trees, marking the southern boundary of the plantation grounds; and near this line, where the road wound through undulating fields and groves of pines, rose a solitary object which the rays of the moon revealed to be an old wooden church, surrounded by a fence of crumbling brick, evidently abandoned by the worshippers and left to go to ruin. This little building was square and very close to the ground, with a small \\        
\end{tabular}
\label{tab:eval_examples_books}
\end{tcolorbox}
\end{table}

\begin{table}[t]
\centering
\small
\caption{\textbf{Examples from the \textsc{Bios} domain.}}
\label{tab:eval_examples_bios}

\begin{tcolorbox}[
  colback=white,
  colframe=purple!50!black,
  arc=2pt,
  boxrule=0.6pt,
  left=6pt,right=6pt,top=6pt,bottom=6pt,
  width=0.95\columnwidth
]
\begin{tabular}{@{}p{0.92\columnwidth}@{}}
\textbf{Input} \\
\midrule
Write a factual biography about Chadwick Boseman. Include only factual information that you are confident in.\\
\textbf{Biography:} \\
\addlinespace[2pt]
Write a factual biography about Kalki Koechlin. Include only factual information that you are confident in.\\
\textbf{Biography:} \\
\addlinespace[2pt]
Write a factual biography about Sara Paxton. Include only factual information that you are confident in.\\
\textbf{Biography:} \\
\end{tabular}
\end{tcolorbox}
\end{table}

\begin{table}[t]
\centering
\caption{\textbf{\ours{} generation on the beginning of Donna Tartt's \emph{The Goldfinch} (2013).} We decode at the token level with $p_s=$ \ourlm{} and $p_r=$ Llama 3.1 70B, and sweep $k \in \{1.0, 3.0, 5.0, 10.0, 15.0\}$ with $T_{\max}=100$. We highlight verbatim spans from the reference in \textcolor{red}{red}. }
\begin{tcolorbox}[colback=white,colframe=blue!50!black,arc=2pt,width=0.95\textwidth]
\small
\footnotesize\textbf{Input:}
Complete the prefix: While I was still in Amsterdam, I dreamed about my mother for the first time in years. I'd been shut up in my hotel for more than a week, afraid to telephone anybody or go out; and my heart scrambled and floundered at even the most innocent noises: elevator bell, rattle of the minibar cart, even church clocks \par\medskip

\footnotesize\textbf{Reference:}
tolling the hour, de Westertoren, Krijtberg, a dark edge to the clangor, an inwrought fairy-tale sense of doom. By day I sat on the foot of the bed straining to puzzle out the Dutch-language news on television (which was hopeless, since I knew not a word of Dutch) and when I gave up

\vspace{0.5\baselineskip}
\begin{tabular}{@{}>{\centering\arraybackslash}p{0.10\linewidth}p{0.82\linewidth}@{}}
\toprule
\textbf{Setting} & \textbf{Generation} \\
\midrule
$p_s$ &
. I was so afraid of my mother that I couldn't even go to the movies. I was so afraid of her that I couldn't even go to the movies. I was so afraid of her that I couldn't even go to the movies. I was so afraid of her that I couldn't even go to the movies. I was so afraid of her that I couldn't even go to the movies. I was so afraid of her that I couldn't even go to the movies. I was \\
\addlinespace[2pt]
$k=1.0$ &
. I was so afraid of my mother that I couldn't sleep, and when I did I would wake up after a few hours, thinking: "My mother is dead." But she wasn't dead. She was alive somewhere, and she was trying to get me. She wanted to kill me. And so, of course, I dreamed about her. I saw her the way I'd seen her last: a big, strong woman, with a big, strong face, and a voice that was \\
\addlinespace[2pt]
$k=3.0$ &
. I was waiting for the van to come for me, the long black hearse-type car that delivers you to the airport. It had been nearly a week since I'd heard the voices, but I was still afraid to go out shopping. And I was afraid to confide in anybody. I thought I must be going crazy. I'd cut myself off from friends who would reassure me it was just the travel, the strain of appearing in public, even the possible anemia of a slightly under \\
\addlinespace[2pt]
$k=5.0$ &
. I'd been in the same street clothes since New York; I longed to change but there were reporters watching the front of the hotel, and if I nodded off in the bathroom, they'd be at the door with their cameras and cocktail voices, their feet already on the stairs to the stars, like pilots after a successful landing. The waiter, an anxious blond kid in an orange shirt, was holding my camembert as if he were about to hand it to me.\\
\addlinespace[2pt]
$k=10.0$ &
. This dream, it seemed to me, began in the Piazza San Marco with my mother and me skating together. We performed tight circles and figures eights. Our skates scissored and swished. Just before I woke up, a friend I'd forgotten had died joined us. I don't remember any unhappiness or saying goodbye, only our laughing at the humped sea lions who crawled out of the sea and onto the ice and stayed there, shaking themselves and folding their soft pre \\
\addlinespace[2pt]
$k=15.0$ &
\textcolor{red}{tolling the hour, de Westertoren, Krijtberg, a dark edge to the clangor, an inwrought fairy-tale} malevolence. These were real fears, in a sense, because I really was in danger; the threat to my life was totally legitimate. I wrote letters to the friends I'd abandoned all over the place, letters some of them would never receive. "Do you remember our trip to Europe with Mrs. Swanson?" I asked them (I \\
\addlinespace[2pt]
$p_r$ &
\textcolor{red}{tolling the hour, de Westertoren, Krijtberg, a dark edge to the clangor, an inwrought fairy-tale} malevolence. These were real fears, in a sense, because I really was in danger; the threat to my life was totally legitimate. I wrote letters to the friends I'd abandoned all over the place, letters some of them would never receive. "Do you remember our trip to Europe with Mrs. Swanson?" I asked them (I \\
\bottomrule
\end{tabular}

\end{tcolorbox}
\label{tab:gen_tok_examples_books}
\end{table}

\begin{table}[t]
\centering
\caption{\textbf{\ours{} generation on a biography task.} We decode at the token level with $p_s=$ \ourlm{} and $p_r=$ Llama 3.1 70B, and sweep $k \in \{0.1, 0.5, 1.0, 2.0\}$ with $T_{\max}=150$.  }
\begin{tcolorbox}[colback=white,colframe=blue!50!black,arc=2pt,width=0.95\textwidth]
\small
\footnotesize\textbf{Input:}
Write a factual biography about IU (singer). Include only factual information that you are confident in.\\
\vspace{3.5\baselineskip}
\begin{tabular}{@{}>{\centering\arraybackslash}p{0.10\linewidth}p{0.82\linewidth}@{}}
\toprule
\textbf{Setting} & \textbf{Generation} \\
\midrule
$p_s$ &
IU is an American singer, songwriter, and producer. He is best known for his 2016 single "I'm Not a Girl", which peaked at number 1 on the Billboard Hot 100 and topped the charts in 17 countries. IU is also known for his collaborations with artists such as Ariana Grande, Justin Bieber, and Taylor Swift. IU has released 3 studio albums, 2 EPs, and 2 mixtapes. IU is also a member of the group The Black Eyed Peas, which released their debut album, "The Black Eyed Peas", in 2000. IU is also a member of the group The Black Eyed Peas, which released their debut album, "The Black Ey \\
\addlinespace[2pt]
$k=0.1$ &
IU is an American singer, songwriter, and producer. He is best known for his 2016 single "I'm Not a Girl", which peaked at number 1 on the Billboard Hot 100 chart. IU has also released two albums, "I'm Not a Girl" and "I'm Not a Girl 2". IU has also been featured on the songs "I'm Not a Girl" and "I'm Not a Girl 2" by other artists, including Ariana Grande, Justin Bieber, and Taylor Swift. IU has also been featured on the songs "I'm Not a Girl" and "I'm Not a Girl 2" by other artists, including Ariana Grande, Justin Bieber, and Taylor Swift. IU has \\
\addlinespace[2pt]
$k=0.5$ &
IU is an American singer, songwriter, and actress. She was born on May 16, 1993, in Los Angeles, California. IU began her career in the entertainment industry as a child actress, appearing in various television shows and films. She made her debut as a singer in 2008 with the release of her first single, “Lost and Found.” IU has since released several successful albums and singles, including “Good Day,” “You and I,” and “The Story of Us.” She has also appeared in several films and television shows, including “The Last Song” and “The Good Wife.” IU is known for her powerful vocals and her ability to convey emotion through her music. She has won numerous awards for her work, \\
\addlinespace[2pt]
$k=1.0$ &
 IU is a South Korean singer-songwriter and actress. She was born on May 16, 1993, in Seoul, South Korea. IU began her music career in 2008 and has since released several successful albums and singles. She is known for her unique voice and her ability to write and perform a wide range of music genres. IU has also acted in several television dramas and films, and has won numerous awards for her work in both music and acting.
IU was born Lee Ji-eun in Seoul, South Korea. She began her music career in 2008, when she was just 15 years old. IU’s debut album, “Lost and Found,” was released in 2008 and featured the hit single “Boo” \\
\addlinespace[2pt]
$k=2.0$ &
IU (singer) is a South Korean singer-songwriter and actress. She is known for her unique voice and her ability to write and compose her own songs. IU has released several albums and singles, and has won numerous awards for her music. She has also appeared in several television dramas and films.
IU was born on May 16, 1993, in Seoul, South Korea. She began her music career in 2008, when she released her first single, “Lost Child.” She then released her first album, “Growing Up,” in 2009. IU has since released several more albums and singles, including “Last Fantasy” (2011), “Modern Times” (2013), and \\
$p_r$ &
IU (singer) is a South Korean singer-songwriter and actress. She is known for her unique voice and her ability to write and compose her own songs. IU has released several albums and singles, and has won numerous awards for her music. She has also appeared in several television dramas and films.
IU was born on May 16, 1993, in Seoul, South Korea. She began her music career in 2008, when she released her first single, “Lost Child.” She then released her first album, “Growing Up,” in 2009. IU has since released several more albums and singles, including “Last Fantasy” (2011), “Modern Times” (2013), and \\
\bottomrule
\end{tabular}

\end{tcolorbox}
\label{tab:gen_tok_examples_bios}
\end{table}

\begin{table}[t]
\centering
\caption{\textbf{\oursbs{} generation on a passage from Douglas Adams' \emph{The Hitchhiker's Guide to the Galaxy} (1979).} We decode at the byte level with $p_s=$ Comma 7B 2T and $p_r=$ Llama 3.1 70B, and sweep $k \in \{0.1, 0.5, 1.5, 2.0, 3.0\}$ with $B_{\max}=800$. We highlight verbatim spans from the reference in \textcolor{red}{red}. }
\begin{tcolorbox}[colback=white,colframe=blue!50!black,arc=2pt,width=0.95\textwidth]
\small
\footnotesize\textbf{Input:}
Complete the prefix: away, four figures pounded up a corridor looking for a way out. They emerged into a wide open-plan computer bay. They glanced about wildly. ``Which way do you reckon Zaphod?" said Ford. 
``At a wild guess, I'd say down here," said Zaphod, running off down to the right between a computer bank and the wall. As the others started after him he was brought up short by a Kill-O-Zap energy bolt that cracked through the air inches in front of him and fried a small section of adjacent wall. A voice on a loud hailer said, ``OK Beeblebrox, hold it right there. We've got you covered." ``Cops!" hissed Zaphod, and span around in a crouch. ``You want to try a guess at all, Ford?" ``OK, this way," said Ford, and the four of them ran down a gangway between two computer banks. At the end of the gangway appeared a heavily armoured and space-suited figure
\\

\footnotesize\textbf{Reference:}
waving a vicious Kill-O-Zap gun. ``We don't want to shoot you, Beeblebrox!" shouted the figure. ``Suits me fine!" shouted Zaphod back and dived down a wide gap between two data process units. The others swerved in behind him. ``There are two of them," said Trillian. ``We're cornered." They squeezed themselves down in an angle between a large computer bank and the wall. They held their breath and waited. Suddenly the air exploded with energy bolts

\vspace{0.5\baselineskip}
\begin{tabular}{@{}>{\centering\arraybackslash}p{0.10\linewidth}p{0.82\linewidth}@{}}
\toprule
\textbf{Setting} & \textbf{Generation} \\
\midrule
$p_s$ &
 in purplish clothes. He was a uniformed Space Guard with guns pointed up his back. ``Who are you?" asked the space suit. ``I'm doing a study about Batman," said Ford. ``Oh," grunted the Space Guard with an odd look on his face. ``Is he down here?" ``How can we help you? And why am I a space guard?" ``Because Zaphod is a cop who has been caught cheating on his beat up wife Zarah Parrish with a bunch of filthy thugs. Big deal! You can't tell me they're not male snitch-shut-up and then put his watch in a safe somewhere and feed himself at night with drugs." ``And what do you suggest we should do about this guy?" ``Do what we must to stop him from changing history to fit our needs," said Kim Lee, the space guard officer shouting about him while he sprayed water at Ford's face with \\
\addlinespace[2pt]
$k=0.1$ &
with a giant gun pointed squarely at them and a voice booming from its speaker: ``Stand still and hand over your weapons or I'll hit the button!" All agreed instantly to give up their weapons\\
\addlinespace[2pt]
$k=0.5$ &
with a giant gun pointed squarely at them and a voice on his helmet speaker at full volume saying, ``Hold it right there Beeblebrox if you want to go on living." \\
\addlinespace[2pt]
$k=1.5$ &
with a giant gun and infrared goggles. ``Yee-haa OK!" shouted the cop hunkering down behind a computer bank pursuing an erratic course. Several energy bolts screamed through the air, erupted against walls and exploded computer banks with searing flashes. ``Go!" shouted Ford, and the four would-be escapees ran helter- skelter down the side of the computer bay that they had emerged into. It opened out within a few yards into a wide thorough fare leading to a large group of people who were standing under a banner saying: ``The Imperial Galactic Government welcomes you to its newest tax system. May you be forever blessed. Remember -- your government needs you -- pay your taxes with a smile!\\
\addlinespace[2pt]
$k=2.0$ &
with a giant laser gun. ``Enoch's groove!" muttered Zaphod. ``Gently, beesknees," said Ford pulling him back. ``I had expected a welcoming committee," said Trillian, ``but nothing quite so big." \\
\addlinespace[2pt]
$k=3.0$ &
with \textcolor{red}{a vicious Kill-O-Zap gun unleashed. ``We don't want to shoot you, Beeblebrox!" shouted the figure. ``Suits me fine!" shouted Zaphod back, and dived down a wide gap between two data process units. The others swerved in behind him. ``There are two of them," said Trillian. ``We'}ll take a bet on that," muttered Ford, and they surged on. The passageway between the two computer banks was narrow. Tons and tons of computer gear towered up above them, silent and motionless as a dead planet. ``They don't need to move much," whispered Zaphod to anybody who might be listening. ``All they have to blow is us." They ran on. At the end of the passageway the walls diminished sharply in width: the pathway continued out under a low ceiling, but the pathway itself was blocked by a heavy electromag \\
\addlinespace[2pt]
$p_r$ &
 \textcolor{red}{waving a vicious Kill-O-Zap gun. ``We don't want to shoot you, Beeblebrox!" shouted the figure. ``Suits me fine!" shouted Zaphod back and dived down a wide gap between two computer banks. The others swerved in behind him. ``There are two of them," said Trillian. ``We're cornered." They squeezed themselves down in an angle between a large computer bank and the wall. They held their breath and waited. Suddenly the} silence was shattered by the bombastic tones of PA Voice. ``A reminder to all diners," it thundered, ``that lunchtime ends at 1.00 pm. Meal trays are to be returned to the tables before ret \\
\bottomrule
\end{tabular}

\end{tcolorbox}
\label{tab:gen_byte_examples_books}
\end{table}

\begin{table}[t]
\centering
\caption{\textbf{\oursbs{} generation on a biography task.} We decode at the byte level with $p_s=$ Comma 7B 2T and $p_r=$ Llama 3.1 70B, and sweep $k \in \{0.1, 0.5, 1.0, 2.0\}$ with $B_{\max}=800$.  }
\begin{tcolorbox}[colback=white,colframe=blue!50!black,arc=2pt,width=0.95\textwidth]
\small
\footnotesize\textbf{Input:}
Write a factual biography about Alma Katsu. Include only factual information that you are confident in.\\
\vspace{3.5\baselineskip}
\begin{tabular}{@{}>{\centering\arraybackslash}p{0.10\linewidth}p{0.82\linewidth}@{}}
\toprule
\textbf{Setting} & \textbf{Generation} \\
\midrule
$p_s$ &
Alma Katsu (February 1, 1894 – January 1, 1985) was a pioneering Japanese-American artist and art teacher. She is best known for her contributions to the modernist movement in Japanese art and for her unique blend of traditional Japanese themes and modern artistic techniques. \\
\addlinespace[2pt]
$k=0.1$ &
Alma Rosina Katsu (July 19, 1930 – May 12, 2018) was an American salsa dancer, founding member of the New York City Salsa Club, and the first African-American woman to be inducted into the National Salsa Hall of Fame.Alma Rosina Katsu was born on July 19, 1930, in Harlem, New York City. Her parents were of Puerto Rican descent. She grew up in a housing project and attended public school. As a teenager, she became interested in salsa dancing and began taking lessons from Cuban-American dancer Eddie Torres. In 1955, Katsu and her husband, Jack Torres, founded the New York City Salsa Club, one of the first salsa clubs in the United States. The club was located in the basement of their home in East Harlem and served as a gathering place for people of all races and backgrounds to dance sals \\
\addlinespace[2pt]
$k=0.5$ &
 Alma Rosina Katsu (born December 19, 1959) is an American author and novelist. She was born in Fairbanks, Alaska. Her parents were both teachers. She has one sister. Katsu attended the University of Washington where she earned a degree in journalism. She then worked as a reporter for several newspapers before becoming a freelance writer. In 1997, Katsu published her first book, The Red Maple Leaf. The book was a collection \\
\addlinespace[2pt]
$k=1.0$ &
Alma Rosina Katsu is a Japanese-American artist and author who was born in Tokyo, Japan on February 24th, 1970. Her father was of Japanese descent while her mother had roots from the United States. She spent most of her childhood living in Tokyo until she moved to New York City with her family when she was twelve years old. Katsu attended high school at Stuyvesant High School before studying fine arts at The Cooper Union for the Advancement of Science and Art where she graduated in 1992. In addition to being an accomplished painter, Katsu has also worked as a graphic designer, web developer, and advertising copywriter throughout her career. In 2001, Katus published her first book titled “The World Beneath” which received critical acclaim and won several awards including the \\
\addlinespace[2pt]
$k=2.0$ &
Alma Katsu is an American writer who has written several acclaimed novels. She was born in Washington, D.C., and raised in the suburbs of Maryland. Her father was a Japanese American and her mother was of Czech descent. After graduating from college with a degree in economics, she worked as a financial analyst for several years before deciding to pursue a career as a writer. Her first novel, “The Taker,” was published in 2011 and received critical acclaim. Since then, she has published two more novels, “The Reckoning” and “The Descent.” Alma Katsu is an American writer whose work often explores themes of love, loss, and redemption. She has published three novels, all of which have been critically acclaimed. “The Taker” was her debut novel, and it tells the story \\
$p_r$ &
Alma Katsu is an American author of historical fiction and supernatural thrillers. She was born in 1961 in Washington, D.C., and grew up in the suburbs of Maryland. Katsu attended the University of Maryland, where she earned a degree in English literature. After college, she worked as a journalist and editor for various publications, including The Washington Post and The Baltimore Sun. In 2001, Katsu published her first novel, The Taker, which was a supernatural thriller set in 19th-century New England. The book was well-received by critics and readers alike, and it was followed by two sequels, The Reckoning and The Descent. Katsu has also written several standalone novels, including The Hunger, which was published in 2018 and is set during the Donner Party \\
\bottomrule
\end{tabular}

\end{tcolorbox}
\label{tab:gen_byte_examples_bios}
\end{table}

\begin{table}[t]
\centering 
\small
\caption{\textbf{Examples from the \textsc{Creative} domain.} These are creative passages sourced from the \texttt{r/WritingPrompts} Reddit community, and authored from 2024--2025. We take these as a heldout set that neither $p_r$ nor $p_s$ should have seen during training.}
\begin{tcolorbox}[colback=white,colframe=purple!50!black,arc=2pt,width=0.95\textwidth]
\begin{tabular}{@{}>{\raggedright\arraybackslash}p{0.48\linewidth} >
{\raggedright\arraybackslash}p{0.48\linewidth}@{}}
\textbf{Input} & \textbf{Reference} \\
\midrule
 Complete the prefix:\textbackslash{}n After getting over his disbelief, a question popped into Ciqoid’s mind. If the human was still in the simulator, what was he doing? The war was over. What else could he possibly be doing in their state-of-the-art simulator? That should have been the end of his training. Ciqoid placed his fingers suction cups against the scanner, overriding the lock on the simulator’s door. When his suction cups popped off, the door peeled open, revealing Henry slouched, not in the simulation chair, but in a cushioned office chair. In one hand, he had a carbonated beverage, and in the other, his remote. Henry’s visor shining in his face as he continued his training. “Henry, the simulation’s over. Proceed to your quarters.
 & We will discuss your results later. I’ll need time to process them.” Ciqoid glanced at the display panel across from the human, watching over their current in game activities. Henry’s avatar bouncing up and down in the city square, using that bouncing technique to move three frames faster than his       \\ \midrule 
Complete the prefix:\textbackslash{}Time is fair, but uncaring, and cold. It washes over everything, and everyone, not taking into consideration what its passing does to them. As a vampire, I went against Time itself, and have been alive for centuries. At first it was glorious, as I won against a power that most living beings strived to win against. But Time didn't care about silly little me, and as I grew from a human to a vampire, to an ancient vampire...I started losing myself. I felt as if Time couldn't erode my life, thus it eroded my emotions that made me a person. A living being. It's been decades since the last time I had contact with anyone else. I didn't feel the need to talk with others, nor did I need society's help to feed myself. Today, the decades old silence was broken by a visit. It's seems that while Time passes always the same not caring about anything, Life is fickle, loving to throw unexpected things at us. For just now the government sent someone to notify me, that I was the last living relative of a young child that just got orphaned. She was my last living descendant...and it seemed,                                                                                                                                           & that I still cared. Took me an hour to flicker to the town where she was. When I saw her, I felt...overwhelmed. A child, barely in her teen years, sitting alone, hugging her legs as she stared outside the window. She reminded me of myself. I walked up to her.    \\ \midrule 
Complete the prefix:\textbackslash{}n Holding the bowl of water in his hands, Tomothy's palms and fingers were beginning to get uncomfortably warm. He was thinkin’ he should've used a pot for this one. Handles woulda been nice. He coulda dealt with the cold water, his hands would go a little numb, but he was afraid he'd burn his hands on this one if he hadn't already. Becoming a God had been mostly underwhelming. Granting the Wish of Water had taken Tomothy not just all over the world but all over the universe. Water was a hot commodity, pun not intended but appreciated with a smile regardless regarding this particular instance. The woosh of stars, the chilly splash of the cosmos and the God of Thirst Quenching stood in front of something that definitely wasn't a bunny rabbit. The creature looked to be a bundle of strings that stretched up about to infinity as best Tomothy could tell. When he appeared the beast wasn't wiggling, it stood stock still and its color undulated from side to side and up and down all over the strings. Undulated was a word he'd learned from a documentary about Octopodes. In the documentary he'd learned the plural of the
& octopus was actually octopodes, not octopi or the more commonly used octopuses. From white to black and back again the colors changed hitting every bit of the rainbow and all the variations in between making for a very satisfying display and in less than a minute to boot. Tomothy began \\        
\end{tabular}
\label{tab:eval_examples_heldout}
\end{tcolorbox}
\end{table}

\section{Additional Results}
\subsection{Adaptive Budgeting Handles Non-uniform Risk} \label{app:analysis_adaptive}

\begin{figure*}[tb]
  \centering
  \subfigure[\textbf{Token-level trajectory visualization.}]{%
    \includegraphics[width=0.9\textwidth]{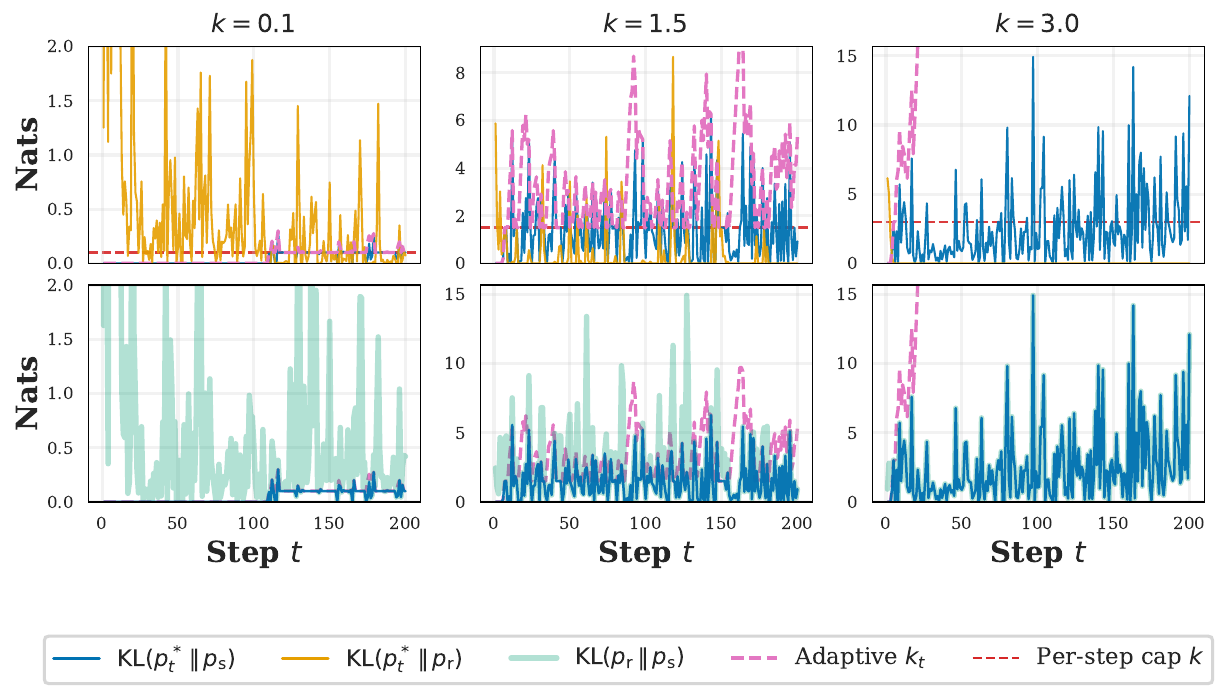}
    \label{fig:kl_traj_token}
  }\\[-0.4em]
  \subfigure[\textbf{Byte-level trajectory visualization.}]{%
    \includegraphics[width=0.9\textwidth]{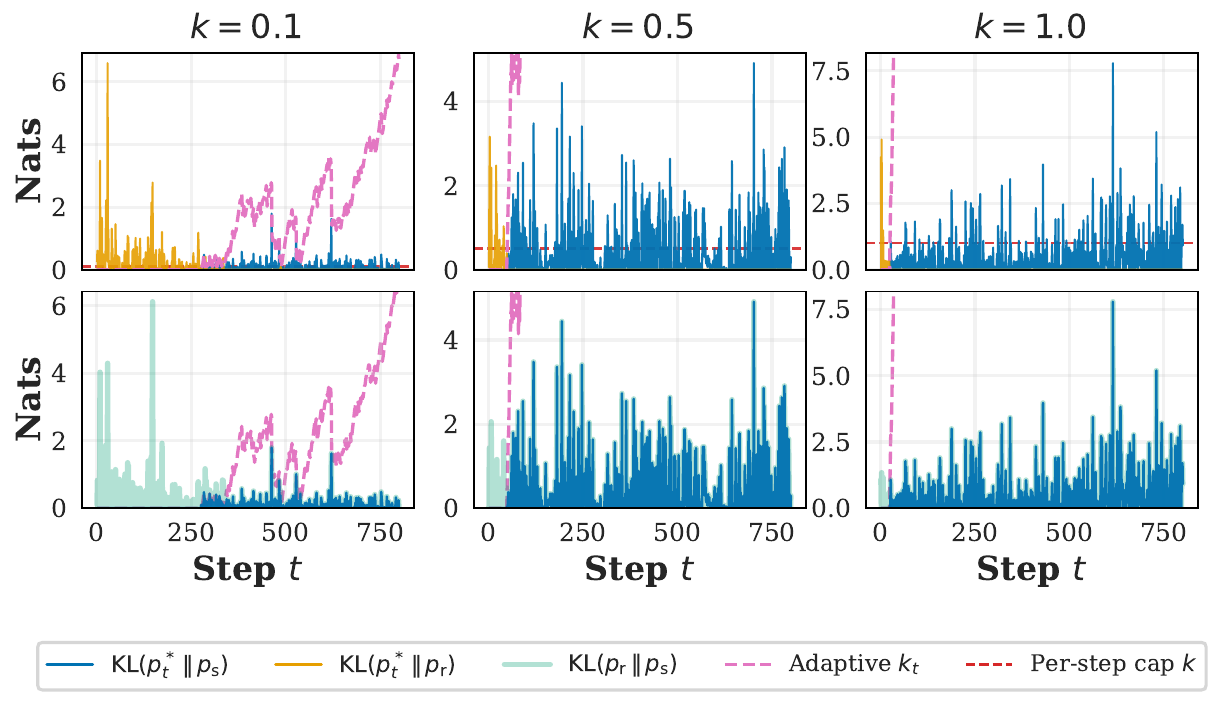}
    \label{fig:kl_traj_byte}
  }
  \caption{\textbf{Adaptive budgets track per-step KL spikes.} Randomly sampled \textbf{Copyright} trajectories at token (top) and byte (bottom) granularities. We plot per-step divergences $\mathrm{KL}(p^*\|p_{s})$ (blue) and $\mathrm{KL}(p^*\|p_{r})$ (orange), the diagnostic $\mathrm{KL}(p_{r}\|p_{s})$ (green, where shown), and the logged budget signal (magenta, dashed) with per-step cap $k$ (red, dashed).}
  \label{fig:kl_trajectory_plots}
\end{figure*}

\cref{fig:kl_trajectory_plots} shows representative \textbf{Copyright} generation trajectories under \ours{} (token-level) and \oursbs{} (byte-level) for various $k$ (with $T_{\max}=200, B_{\max}=800$). We plot the realized spend $\mathrm{KL}(p_t^*\|p_s)$ (blue), the diagnostic $\mathrm{KL}(p_r\|p_s)$ (green, where shown), and the adaptive allowance $k_t$ (magenta) relative to the nominal cap $k$ (red).

At the token level, $k=0.1$ yields a conservative cold start: $k_t$ is floored at $0$ for much of the prefix-debt window, forcing $p_t^*$ to track $p_s$ closely. At $k=1.5$, budgeting becomes dynamic: $k_t$ banks allowance during low-risk steps and releases it to accommodate occasional memorization spikes, allowing $\mathrm{KL}(p_t^*\|p_s)$ to exceed the nominal cap while remaining globally feasible. At $k=3.0$, the constraint is rarely binding and $p_t^*\approx p_r$, reflected by $\mathrm{KL}(p_t^*\|p_s)\approx \mathrm{KL}(p_r\|p_s)$.

The byte-level view exhibits the same pattern at finer temporal resolution: risk manifests as sharper, more localized spikes, and $k_t$ rises rapidly during low-risk stretches. Since many byte steps incur near-zero divergence, the model typically accrues budget faster at the same $k$ than in the token-level regime.

\subsection{Long-tail Knowledge vs. $k$}
We show that \ours{} may unintentionally suppress legitimate \emph{long-tail} factual recall. Rare entities are less likely to be memorized by $p_s$ (which, in our work, is orders of magnitude smaller than $p_r$, and typically not as well-trained). Thus, enforcing proximity to $p_s$ may also disproportionately suppress correct but uncommon facts. To probe this effect, we return to biography generation on \textsc{Bios} and stratify prompts by entity frequency (from Very Rare to Very Frequent) using labels provided by \citet{min-etal-2023-factscore}. For each bucket, we report average \fs{} claim precision as a function of the budget $k$ (log scale), with $p_s$ and $p_r$ as reference points. We decode at the token level with \{\ourlm{}, Llama 3.1 70B\} and $T_{\max}=200$. \cref{fig:factual_precision_by_category} shows a consistent trend: increasing $k$ improves factual precision across all buckets as \ours{} shifts mass from $p_s$ toward $p_r$, but the gains are strongly frequency-dependent and saturate at $p_r$'s bucket-specific ceiling. Frequent entities recover quickly, while Rare and Very Rare entities improve more slowly and plateau at substantially lower precision as $p_r$ itself is unreliable on the long tail. Meanwhile, $p_s$ remains uniformly low across buckets, suggesting limited factual coverage regardless of frequency.

\begin{figure*}[htbp]
  \centering
  \begin{center}
\centerline{\includegraphics[width=0.85\columnwidth]{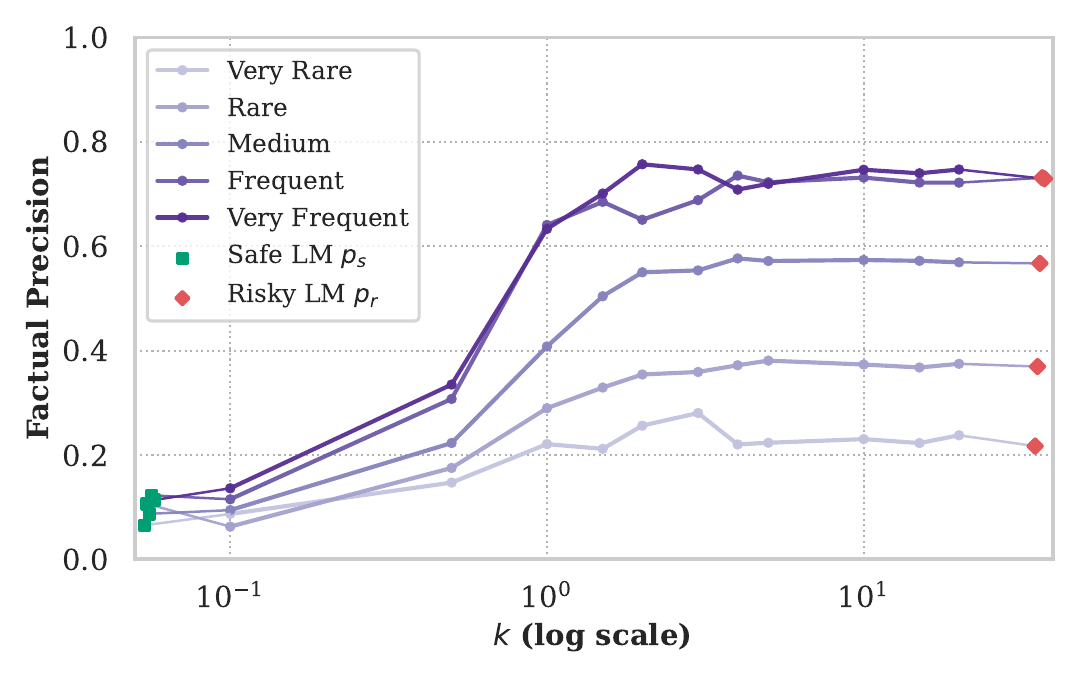}}
  \caption{\textbf{Factual precision on \textsc{Bios} stratified by entity rarity as a function of the budget parameter $k$ (log scale).}  As k increases, precision rises across buckets, with slower improvement for long-tail entities than for head and torso entities.}
  \label{fig:factual_precision_by_category}
\end{center}
\end{figure*}

\subsection{\ours{} on Other Settings}

\paragraph{Using instruction-tuned LMs.}
While our main results primarily focus on base models (especially as our safe models lack instruction-tuned counterparts), there is no inherent obstacle to applying our purely inference-time method to instruction-tuned models. However, to address this concern, we ran an evaluation (under the same experimental conditions) with the mixed pair \{Llama 3.1 70B Instruct, Comma 1.7B\}, and observe the same beneficial risk-utility tradeoff for \ours{} against pointwise baselines in \cref{fig:instruct_base_results}.

\begin{figure*}[htbp]
  \centering
  \begin{center}
\centerline{\includegraphics[width=0.85\columnwidth]{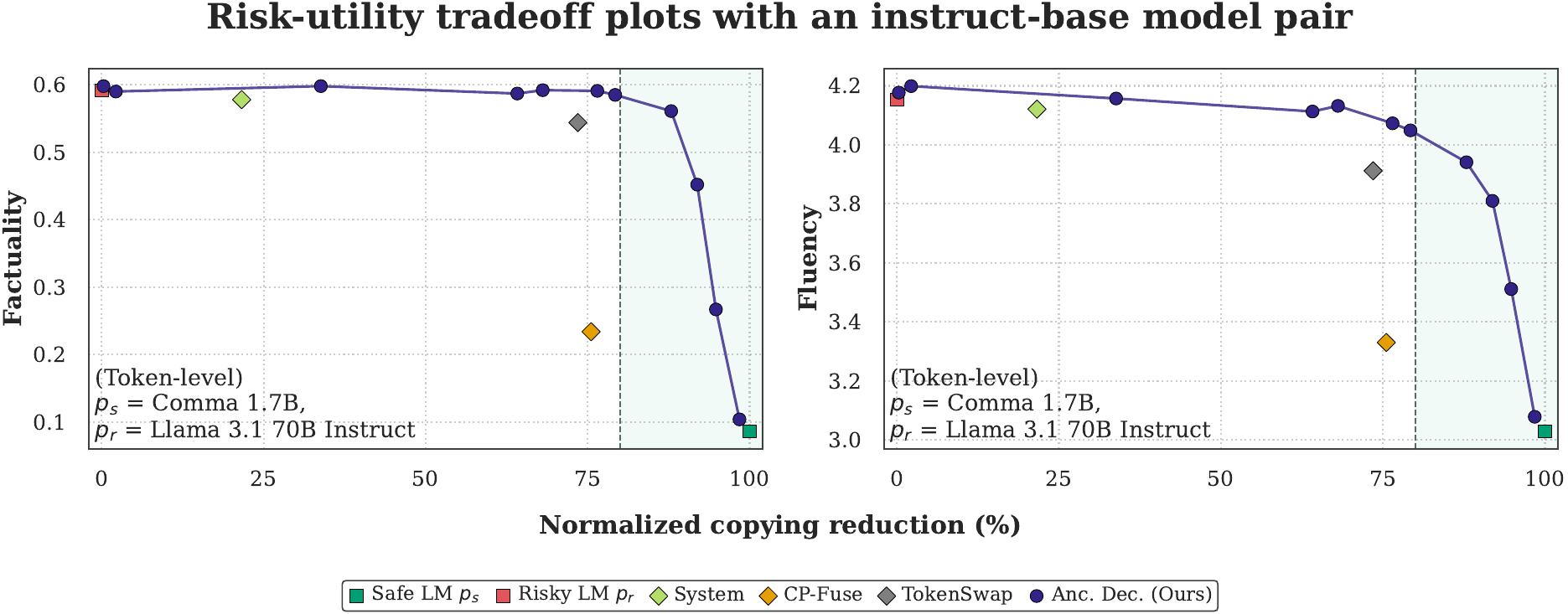}}
  \caption{\textbf{\ours{} is Pareto-optimal on risk-utility tradeoffs with an instruct-base model pair.} We experiment with using the instruction-tuned LM Llama 3.1 70B Instruct as $p_r$, and retain \ourlm as $p_s$.}
  \label{fig:instruct_base_results}
\end{center}
\end{figure*}

\paragraph{Other high-risk domains and utility metrics.}

\begin{figure*}[htbp]
  \centering
  \begin{center}
\centerline{\includegraphics[width=0.85\columnwidth]{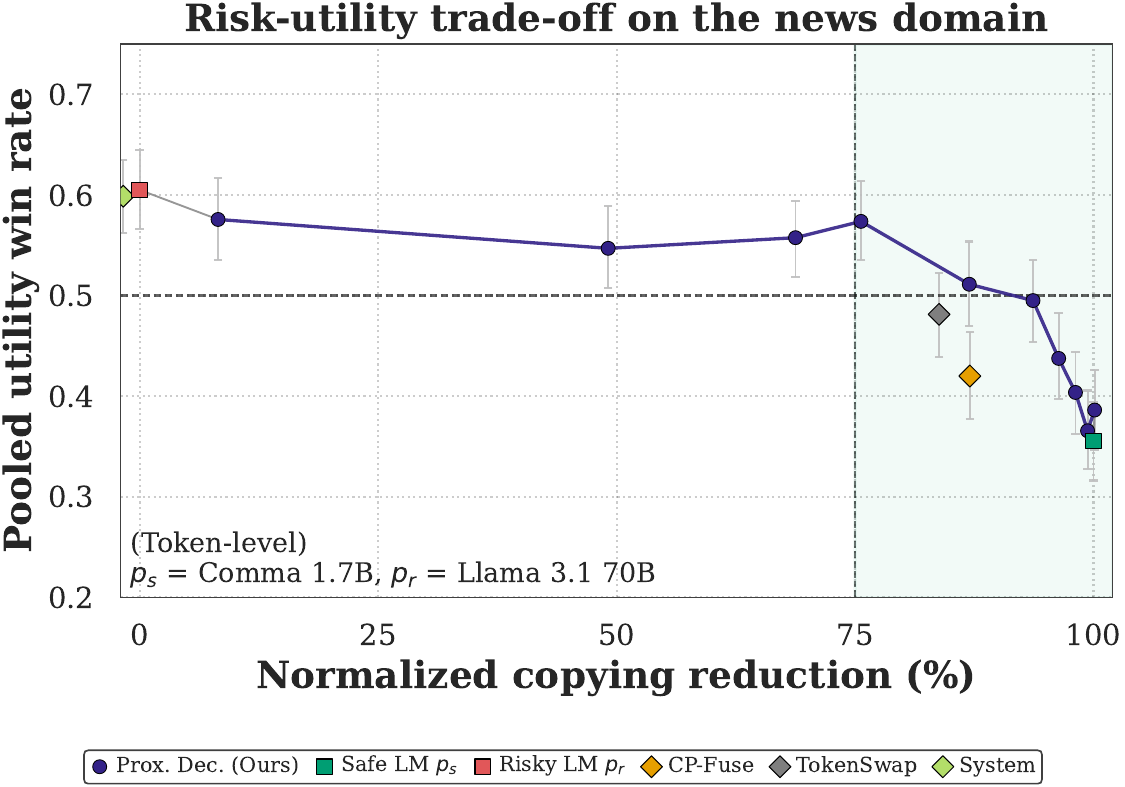}}
  \caption{\textbf{\ours{} evaluated on a news domain, using a pairwise win rate utility metric.} \ours{} remains Pareto-optimal against pointwise baselines.}
  \label{fig:news_domain_results}
\end{center}
\end{figure*}

To show that \ours{} can effectively generalize beyond literary works and our chosen utility metrics, we conduct new trade-off experiment on a  \emph{news} domain of 560 New York Times articles from the NewsSpan dataset~\citep{zhang-etal-2025-certified}. In lieu of our main utility metrics, we instead measure semantic usefulness in the same task space as copyright risk, using a reference-free pooled pairwise win rate metric from \citet{wei2024evaluating}, where an LLM judge (GPT 4.1 mini) compares each method's continuation against outputs from a fixed competitor pool on randomly sampled prompts.

As \cref{fig:news_domain_results} shows, \ours{} remains effective on the news domain, with the alternative utility metric.

\paragraph{Evaluating general-purpose utility tasks.}
One practical question is how \ours{} fares on general-purpose utility tasks, e.g., summarization or question-answering. As a lightweight sanity check, we consider the model pair \{\ourlm{}, Llama 3.1 70B\} and evaluate \ours{} at $k=1.5$, a representative Pareto-efficient operating point in the high-protection regime (\cref{fig:teaser_token_tradeoffs}), on several practical downstream tasks using the LM Evaluation Harness \citep{eval-harness}: factual QA (TruthfulQA MC1/MC2), summarization (CNN/DailyMail), and code generation (HumanEval pass@1). We consider 400 samples per task.

\begin{table*}[htbp]
\centering
\small
\caption{\textbf{Downstream evaluation.} We report the performance of $p_r$=Llama 3.1 70B, $p_s$=\ourlm{}, and \ours{} on TruthfulQA question-answering, CNN/DailyMail summarization, and HumanEval code generation.}
\label{tab:downstream_evaluation}
\begin{tabular}{lrrrr}
\toprule

\textbf{Model}

& \multicolumn{2}{c}{\textbf{TruthfulQA}}

& \multicolumn{1}{c}{\textbf{CNN/DailyMail}}

& \multicolumn{1}{c}{\textbf{HumanEval}} \\

\cmidrule(lr){2-3}\cmidrule(lr){4-4}\cmidrule(lr){5-5}

& \textbf{MC1 Acc.}

& \textbf{MC2 Acc.}

& \textbf{ROUGE}

& \textbf{pass@1} \\

\midrule
\textsc{\ourlm{}} & 0.255 & 0.431 & 0.092 & 0.018 \\
\textsc{\ours{}} & \bf{0.303} & 0.472 & 0.139 & 0.488 \\
\textsc{Llama 3.1 70B} & 0.300 & \bf{0.473} & \bf{0.156} & \bf{0.506} \\
\bottomrule
\end{tabular}%
\end{table*}

As \cref{tab:downstream_evaluation} shows, \ours{} performs comparably to the risky model (Llama 3.1 70B) on all the general-purpose tasks, showing that it preserves utility in the high-protection regime.

\paragraph{Non-literal copying results}
\label{subsec:nonliteral}
While \emph{literal copying} is the primary focus of our work, we also show that \ours{} can alleviate \emph{non-literal copying}, which, within a literary context, is the generation of the same incidental characters, events, or plot elements as an original work, despite differences in surface form~\citep{chen-etal-2024-copybench}. 

In a small-scale experiment, we adopt the character copying evaluation of \citet{chen-etal-2024-copybench}. This benchmark extracts a set of key characters from each popular novel from its CliffsNotes summary. Models are prompted to generate an open-ended story, given a plot summary sentence, e.g., ``Theo and his mother visit the Metropolitan Museum of Art.".  

We then measure whether these character names appear in the model's generation via exact string match. Following their protocol, we report \emph{character overlap}, or the fraction of samples for which the number of matched character names exceeds 3. We also report the \emph{non-verbatim fluency} of the model generation, which is scored following \cref{subsec:fluency}.

We use the token-level model pair \{\ourlm{}, Llama 3.1 70B\}. We sweep \ours{} across various $k$, and disable the prefix debt, since it is primarily designed to curb literal copying and we find it can hurt fluency in this non-literal setting. We additionally report our context-less baselines (\textsc{System}, \textsc{TokenSwap}, \textsc{CP-Fuse}). 

\begin{figure*}[htbp]
  \centering
  \begin{center}
\centerline{\includegraphics[width=0.7\columnwidth]{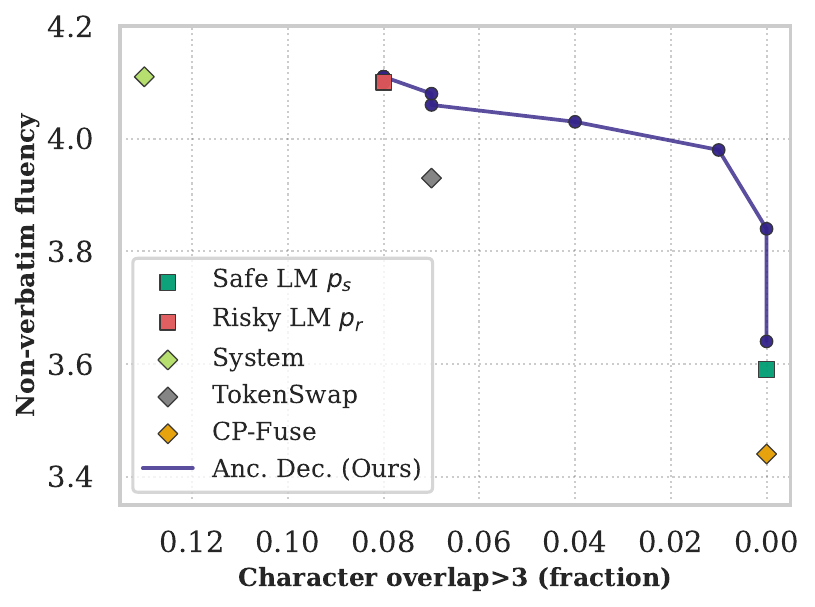}}
  \caption{\textbf{\ours{} remains Pareto-optimal on risk-utility tradeoff plots for non-literal copying.}  We evaluate the fluency and extent of character overlap${>}3$ in open-ended story generation.}
  \label{fig:nonliteral_eval}
\end{center}
\end{figure*}

\cref{fig:nonliteral_eval} shows results. \ours{} defines the Pareto frontier (in the upper-right region), attaining the best tradeoff between fluency and non-literal copying mitigation. 

Beyond literary settings, non-literal copying is also relevant for tasks such as translation (e.g., ``Give me the first chapter of \emph{Harry Potter} in Spanish") or rule-based transformations (e.g., ``Give me the first chapter of \emph{The Goldfinch}, but in all-caps"). However, evaluating non-literal copying in the form of character or event overlap is typically a non-issue in domains such as news articles, where the underlying events, subjects, and timelines are factual and therefore generally not protected by copyright. This indicates that addressing non-verbatim copying may warrant a more task-aware approach that is tailored to the domain and transformation. We defer a more careful and exhaustive evaluation of non-literal copying to future work. 

\subsection{Full Tradeoff Results}
\label{subsec:full_tradeoff_results}
We show full results to \cref{fig:teaser_token_tradeoffs} and \cref{fig:tradeoffs_byte} for each model pair: \{\ourlm{}, Llama 3.1 70B\} in \cref{tab:full_results_comma1.7b_llama3.1_70b},  
\{Comma 7B, Llama 3.1 70B\} in \cref{tab:full_results_comma7b_llama3.1_70b}, 
 \{\ourlm{}, Qwen 2.5 72B\} in \cref{tab:full_results_comma1.7b_qwen2.5_72b},  
\{Comma 7B, Qwen 2.5 72B\} in \cref{tab:full_results_comma7b_qwen2.5_72b}, 
 \{\ourlm{}, Llama 4 Scout 17Bx16E\} in \cref{tab:full_results_comma1.7b_llama4_109b},  
\{Comma 7B, Llama 4 Scout 17Bx16E\} in \cref{tab:full_results_comma7b_llama4_109b}. 

\begin{table*}[t]
\centering
\small
\setlength{\tabcolsep}{3.5pt}
\renewcommand{\arraystretch}{0.92}
\caption{\textbf{Full results for $p_s$=\ourlm{}, $p_r$=Llama 3.1 70B (token-level decoding).} We report the average over 3 seeds. }
\label{tab:full_results_comma1.7b_llama3.1_70b}
\resizebox{\textwidth}{!}{%
\begin{tabular}{lcccccccc}
\toprule
\textbf{Setting} &
\textbf{ROUGE-1$\ge \tau$} &
\textbf{ROUGE-L$\ge \tau$} &
\textbf{Word LCS} &
\textbf{Char. LCS} &
\textbf{Word ACS} &
\textbf{MinHash} &
\textbf{Fluency} &
\textbf{FactScore} \\
\midrule
\multicolumn{9}{l}{\textit{Reference LMs}} \\ \midrule
$p_s$                     & 0.000 & 0.000 &  1.670 &  9.543 &  0.008 & 0.001 & 3.023 & 0.094 \\
$p_r$                     & 0.273 & 0.230 & 10.730 & 58.121 & 10.360 & 0.108 & 4.040 & 0.520 \\
\midrule 
\multicolumn{9}{l}{\textit{Single-Model Methods (using $p_r$ and $p_s$)}} \\ \midrule
\textsc{System}                  & 0.213 & 0.175 & 10.019 & 54.504 &  9.445 & 0.086 & 3.973 & 0.518 \\
\textsc{MemFree}, $n=3$               & 0.026 & 0.017 &  2.588 & 14.221 &  1.047 & 0.014 & 3.182 & 0.368 \\
\textsc{MemFree}, $n=5$                & 0.139 & 0.114 &  6.582 & 35.638 &  5.612 & 0.058 & 3.799 & 0.446 \\
\textsc{MemFree}, $n=7$                & 0.157 & 0.128 &  7.223 & 39.066 &  6.297 & 0.063 & 3.906 & 0.494 \\
\textsc{MemFree}, $n=9$                 & 0.164 & 0.136 &  7.526 & 40.725 &  6.761 & 0.067 & 3.902 & 0.505 \\
\textsc{MemFree}, $n=10$                & 0.163 & 0.137 &  7.592 & 41.119 &  6.817 & 0.067 & 3.919 & 0.511 \\
\textsc{RCAD}, $\alpha=0.1$               & 0.238 & 0.201 & 10.202 & 55.461 &  9.708 & 0.101 & 4.018 & 0.511 \\
\textsc{RCAD}, $\alpha=0.25$               & 0.179 & 0.155 &  8.794 & 48.104 &  8.166 & 0.081 & 3.989 & 0.509 \\
\textsc{RCAD}, $\alpha=0.5$                  & 0.107 & 0.090 &  6.299 & 34.489 &  5.287 & 0.053 & 3.981 & 0.496 \\
\textsc{RCAD}, $\alpha=0.75$               & 0.061 & 0.047 &  4.330 & 23.981 &  3.067 & 0.032 & 3.726 & 0.458 \\
\textsc{RCAD}, $\alpha=1.0$                 & 0.029 & 0.021 &  3.107 & 17.291 &  1.716 & 0.020 & 3.382 & 0.370 \\
\midrule 
\multicolumn{9}{l}{\textit{Two-Model Baselines (using $p_r$ and $p_s$)}} \\ \midrule
\textsc{CP-Fuse}      & 0.006 & 0.001 &  2.264 & 12.497 &  0.118 & 0.004 & 3.213 & 0.198 \\
\textsc{TokenSwap}    & 0.011 & 0.004 &  2.635 & 15.624 &  0.779 & 0.009 & 3.767 & 0.442 \\
\midrule 
\multicolumn{9}{l}{\textit{Ours (using $p_r$ and $p_s$)}} \\ \midrule
\ours{}, $k=0.1$ & 0.001 & 0 & 1.631 &  9.474 &  0.005 & 0.001 & 3.136 & 0.100 \\
\ours{}, $k=0.5$             & 0.001 & 0.000 &  1.670 &  9.664 &  0.010 & 0.001 & 3.516 & 0.241 \\
\ours{}, $k=1.0$             & 0.000 & 0.000 &  1.704 &  9.845 &  0.043 & 0.001 & 3.771 & 0.426 \\
\ours{}, $k=1.5$             & 0.001 & 0.000 &  1.727 &  9.985 &  0.037 & 0.001 & 3.856 & 0.483 \\
\ours{}, $k=2.0$             & 0.008 & 0.002 &  1.870 & 10.853 &  0.176 & 0.003 & 3.933 & 0.516 \\
\ours{}, $k=3.0$          & 0.019 & 0.009 &  2.215 & 12.782 &  0.604 & 0.007 & 4.011 & 0.527 \\
\ours{}, $k=4.0$           & 0.037 & 0.028 &  2.876 & 16.321 &  1.385 & 0.015 & 4.011 & 0.527 \\
\ours{}, $k=5.0$          & 0.056 & 0.043 &  3.489 & 19.709 &  2.105 & 0.022 & 4.016 & 0.533 \\
\ours{}, $k=10.0$           & 0.120 & 0.099 &  5.793 & 31.604 &  4.736 & 0.049 & 4.015 & 0.535 \\
\ours{}, $k=15.0$         & 0.237 & 0.205 &  9.873 & 53.399 &  9.376 & 0.098 & 4.020 & 0.536 \\
\ours{}, $k=20.0$      & 0.248 & 0.214 & 10.630 & 57.494 & 10.220 & 0.107 & 4.046 & 0.537 \\
\bottomrule
\end{tabular}%
}
\end{table*}

\begin{table*}[t]
\centering
\small
\setlength{\tabcolsep}{3.5pt}
\renewcommand{\arraystretch}{0.92}
\caption{\textbf{Full results for $p_s$=Comma 7B, $p_r$=Llama 3.1 70B (byte-level decoding).} We report the average over 3 seeds.}
\label{tab:full_results_comma7b_llama3.1_70b}
\resizebox{\textwidth}{!}{%
\begin{tabular}{lcccccccc}
\toprule
\textbf{Setting} &
\textbf{ROUGE-1$\ge \tau$} &
\textbf{ROUGE-L$\ge \tau$} &
\textbf{Word LCS} &
\textbf{Char. LCS} &
\textbf{Word ACS} &
\textbf{MinHash} &
\textbf{Fluency} &
\textbf{FactScore} \\
\midrule
\multicolumn{9}{l}{\textit{Reference LMs}} \\ \midrule
$p_s$                     & 0.001 & 0.000 &  1.528 &  8.902 &  0.028 & 0.001 & 4.058 & 0.156 \\
$p_r$                     & 0.296 & 0.262 & 10.443 & 56.691 &  9.868 & 0.141 & 4.288 & 0.517 \\
\midrule
\multicolumn{9}{l}{\textit{Single-Model Methods (using $p_r$)}} \\ \midrule
\textsc{System}                 & 0.267 & 0.240 &  9.648 & 52.552 &  9.018 & 0.128 & 4.284 & 0.500 \\
\textsc{MemFree}, $n=3$         & 0.154 & 0.134 &  5.942 & 32.465 &  4.965 & 0.072 & 3.862 & 0.415 \\
\textsc{MemFree}, $n=5$         & 0.179 & 0.156 &  7.027 & 38.311 &  5.988 & 0.086 & 4.032 & 0.498 \\
\textsc{MemFree}, $n=7$         & 0.185 & 0.161 &  7.270 & 39.592 &  6.204 & 0.089 & 4.062 & 0.513 \\
\textsc{MemFree}, $n=9$         & 0.192 & 0.167 &  7.537 & 40.934 &  6.660 & 0.094 & 4.117 & 0.510 \\
\textsc{MemFree}, $n=10$        & 0.190 & 0.164 &  7.565 & 41.256 &  6.686 & 0.094 & 4.094 & 0.515 \\
\textsc{RCAD}, $\alpha=0.1$     & 0.291 & 0.260 &  9.896 & 53.826 &  9.327 & 0.139 & 4.289 & 0.516 \\
\textsc{RCAD}, $\alpha=0.25$    & 0.263 & 0.237 &  8.538 & 46.698 &  7.840 & 0.126 & 4.251 & 0.507 \\
\textsc{RCAD}, $\alpha=0.5$     & 0.174 & 0.162 &  6.043 & 33.281 &  4.964 & 0.087 & 4.079 & 0.475 \\
\textsc{RCAD}, $\alpha=0.75$    & 0.094 & 0.089 &  4.417 & 24.503 &  3.061 & 0.057 & 3.790 & 0.484 \\
\textsc{RCAD}, $\alpha=1.0$     & 0.053 & 0.048 &  3.400 & 18.971 &  1.976 & 0.039 & 3.457 & 0.461 \\
\midrule
\multicolumn{9}{l}{\textit{Two-Model Baselines (using $p_r$ and $p_s$)}} \\ \midrule
\textsc{CP-Fuse}                & 0.003 & 0.002 &  1.897 & 10.813 &  0.085 & 0.003 & 3.751 & 0.230 \\
\textsc{TokenSwap}              & 0.041 & 0.016 &  3.499 & 20.057 &  1.699 & 0.019 & 3.897 & 0.491 \\
\midrule
\multicolumn{9}{l}{\textit{Ours (using $p_r$ and $p_s$)}} \\ \midrule
\oursbs{}, $k=0.1$   & 0.001 & 0.000 & 1.537 &  8.940 & 0.028 & 0.001 & 4.054 & 0.188 \\
\oursbs{}, $k=0.5$   & 0.008 & 0.003 & 1.668 &  9.819 & 0.149 & 0.003 & 4.141 & 0.468 \\
\oursbs{}, $k=1.0$   & 0.021 & 0.011 & 2.027 & 11.694 & 0.501 & 0.004 & 4.186 & 0.513 \\
\oursbs{}, $k=1.5$   & 0.038 & 0.030 & 2.623 & 14.776 & 1.200 & 0.016 & 4.181 & 0.516 \\
\oursbs{}, $k=2.0$   & 0.053 & 0.041 & 3.021 & 17.058 & 1.647 & 0.023 & 4.225 & 0.518 \\
\oursbs{}, $k=3.0$   & 0.071 & 0.059 & 3.660 & 20.462 & 2.340 & 0.031 & 4.266 & 0.517 \\
\oursbs{}, $k=4.0$   & 0.097 & 0.081 & 4.226 & 23.485 & 3.009 & 0.042 & 4.282 & 0.517 \\
\oursbs{}, $k=5.0$   & 0.104 & 0.089 & 4.628 & 25.559 & 3.456 & 0.049 & 4.308 & 0.516 \\
\oursbs{}, $k=10.0$  & 0.129 & 0.110 & 5.447 & 29.908 & 4.313 & 0.060 & 4.292 & 0.516 \\
\oursbs{}, $k=15.0$  & 0.253 & 0.224 & 9.327 & 48.719 & 8.559 & 0.122 & 4.307 & 0.514 \\
\oursbs{}, $k=20.0$  & 0.264 & 0.233 & 9.585 & 52.138 & 8.869 & 0.126 & 4.327 & 0.516 \\
\bottomrule
\end{tabular}%
}
\end{table*}

\begin{table*}[t]
\centering
\small
\setlength{\tabcolsep}{3.5pt}
\renewcommand{\arraystretch}{0.92}
\caption{\textbf{Full results for $p_s$=\ourlm{}, $p_r$=Qwen 2.5 72B (byte-level decoding).} We report the average over 3 seeds.}
\label{tab:full_results_comma1.7b_qwen2.5_72b}
\resizebox{\textwidth}{!}{%
\begin{tabular}{lccccccccc}
\toprule
\textbf{Setting} &
\textbf{ROUGE-1$\ge \tau$} &
\textbf{ROUGE-L$\ge \tau$} &
\textbf{Word LCS} &
\textbf{Char. LCS} &
\textbf{Word ACS} &
\textbf{Cosine} &
\textbf{MinHash} &
\textbf{Fluency} &
\textbf{FactScore} \\
\midrule
\multicolumn{10}{l}{\textit{Reference LMs}} \\ \midrule
$p_s$ & 0.000 & 0.000 & 1.537 &  9.022 & 0.000 & ---   & 0.001 & 3.032 & 0.088 \\
$p_r$ & 0.051 & 0.039 & 2.986 & 16.672 & 1.486 & ---   & 0.019 & 4.336 & 0.457 \\
\midrule
\multicolumn{10}{l}{\textit{Single-Model Methods (using $p_r$)}} \\ \midrule
\textsc{System}            & 0.059 & 0.046 & 3.216 & 17.952 & 1.719 & ---   & 0.023 & 4.318 & 0.457 \\
\textsc{MemFree}, $n=3$    & 0.028 & 0.019 & 2.175 & 12.261 & 0.676 & ---   & 0.009 & 4.132 & 0.339 \\
\textsc{MemFree}, $n=5$    & 0.033 & 0.021 & 2.375 & 13.375 & 0.802 & ---   & 0.010 & 4.279 & 0.433 \\
\textsc{MemFree}, $n=7$    & 0.036 & 0.024 & 2.413 & 13.644 & 0.830 & ---   & 0.012 & 4.281 & 0.447 \\
\textsc{MemFree}, $n=9$    & 0.037 & 0.025 & 2.480 & 13.984 & 0.960 & ---   & 0.013 & 4.299 & 0.454 \\
\textsc{MemFree}, $n=10$   & 0.039 & 0.026 & 2.493 & 14.084 & 0.957 & ---   & 0.012 & 4.304 & 0.452 \\
\textsc{RCAD}, $\alpha=0.1$  & 0.050 & 0.039 & 2.938 & 16.418 & 1.445 & --- & 0.020 & 4.342 & 0.459 \\
\textsc{RCAD}, $\alpha=0.25$ & 0.046 & 0.034 & 2.818 & 15.768 & 1.324 & --- & 0.018 & 4.259 & 0.458 \\
\textsc{RCAD}, $\alpha=0.5$  & 0.039 & 0.029 & 2.520 & 14.293 & 1.073 & --- & 0.015 & 4.111 & 0.457 \\
\textsc{RCAD}, $\alpha=0.75$ & 0.030 & 0.021 & 2.167 & 12.452 & 0.700 & --- & 0.010 & 3.868 & 0.422 \\
\textsc{RCAD}, $\alpha=1.0$  & 0.018 & 0.013 & 1.944 & 11.156 & 0.491 & --- & 0.008 & 3.671 & 0.360 \\
\midrule
\multicolumn{10}{l}{\textit{Two-Model Baselines (using $p_r$ and $p_s$)}} \\ \midrule
\textsc{CP-Fuse}   & 0.002 & 0.000 & 1.608 &  9.327 & 0.003 & ---   & 0.001 & 3.262 & 0.175 \\
\textsc{TokenSwap} & 0.001 & 0.000 & 1.876 & 11.100 & 0.094 & ---   & 0.002 & 3.802 & 0.373 \\
\midrule
\multicolumn{10}{l}{\textit{Ours (using $p_r$ and $p_s$)}} \\ \midrule
\oursbs{}, $k=0.1$  & 0.001 & 0.000 & 1.519 &  8.999 & 0.003 & 0.370 & 0.001 & 3.371 & 0.135 \\
\oursbs{}, $k=0.5$  & 0.003 & 0.000 & 1.550 &  9.306 & 0.021 & 0.394 & 0.001 & 4.106 & 0.418 \\
\oursbs{}, $k=1.0$  & 0.002 & 0.001 & 1.606 &  9.560 & 0.057 & 0.402 & 0.002 & 4.196 & 0.468 \\
\oursbs{}, $k=1.5$  & 0.009 & 0.005 & 1.679 & 10.035 & 0.151 & 0.411 & 0.003 & 4.273 & 0.473 \\
\oursbs{}, $k=2.0$  & 0.009 & 0.007 & 1.740 & 10.388 & 0.215 & 0.415 & 0.004 & 4.274 & 0.482 \\
\oursbs{}, $k=3.0$  & 0.014 & 0.010 & 1.852 & 10.899 & 0.305 & 0.419 & 0.005 & 4.282 & 0.480 \\
\oursbs{}, $k=4.0$  & 0.016 & 0.011 & 1.920 & 11.267 & 0.351 & 0.423 & 0.005 & 4.341 & 0.471 \\
\oursbs{}, $k=5.0$  & 0.022 & 0.016 & 2.071 & 11.985 & 0.527 & 0.427 & 0.007 & 4.351 & 0.468 \\
\oursbs{}, $k=10.0$ & 0.029 & 0.021 & 2.241 & 12.803 & 0.736 & 0.433 & 0.010 & 4.364 & 0.459 \\
\oursbs{}, $k=15.0$ & 0.045 & 0.033 & 2.758 & 15.494 & 1.238 & 0.448 & 0.016 & 4.379 & 0.459 \\
\oursbs{}, $k=20.0$ & 0.046 & 0.034 & 2.802 & 15.742 & 1.277 & 0.448 & 0.017 & 4.350 & 0.462 \\
\bottomrule
\end{tabular}%
}
\end{table*}

\begin{table*}[t]
\centering
\small
\setlength{\tabcolsep}{3.5pt}
\renewcommand{\arraystretch}{0.92}
\caption{\textbf{Full results for $p_s$=Comma 7B, $p_r$=Qwen 2.5 72B (byte-level decoding).} We report the average over 3 seeds.}
\label{tab:full_results_comma7b_qwen2.5_72b}
\resizebox{\textwidth}{!}{%
\begin{tabular}{lccccccccc}
\toprule
\textbf{Setting} &
\textbf{ROUGE-1$\ge \tau$} &
\textbf{ROUGE-L$\ge \tau$} &
\textbf{Word LCS} &
\textbf{Char. LCS} &
\textbf{Word ACS} &
\textbf{Cosine} &
\textbf{MinHash} &
\textbf{Fluency} &
\textbf{FactScore} \\
\midrule
\multicolumn{10}{l}{\textit{Reference LMs}} \\ \midrule
$p_s$ & 0.001 & 0.000 & 1.528 &  8.902 & 0.028 & 0.379 & 0.001 & 4.058 & 0.156 \\
$p_r$ & 0.051 & 0.039 & 2.986 & 16.672 & 1.486 & 0.454 & 0.019 & 4.336 & 0.457 \\
\midrule
\multicolumn{10}{l}{\textit{Single-Model Methods (using $p_r$)}} \\ \midrule
\textsc{System}            & 0.059 & 0.046 & 3.216 & 17.952 & 1.719 & 0.469 & 0.023 & 4.318 & 0.457 \\
\textsc{MemFree}, $n=3$    & 0.028 & 0.019 & 2.175 & 12.261 & 0.676 & 0.425 & 0.009 & 4.132 & 0.339 \\
\textsc{MemFree}, $n=5$    & 0.033 & 0.021 & 2.375 & 13.375 & 0.802 & 0.436 & 0.010 & 4.279 & 0.433 \\
\textsc{MemFree}, $n=7$    & 0.036 & 0.024 & 2.413 & 13.644 & 0.830 & 0.437 & 0.012 & 4.281 & 0.447 \\
\textsc{MemFree}, $n=9$    & 0.037 & 0.025 & 2.480 & 13.984 & 0.960 & 0.440 & 0.013 & 4.299 & 0.454 \\
\textsc{MemFree}, $n=10$   & 0.039 & 0.026 & 2.493 & 14.084 & 0.957 & 0.441 & 0.012 & 4.304 & 0.452 \\
\textsc{RCAD}, $\alpha=0.1$  & 0.050 & 0.039 & 2.938 & 16.418 & 1.445 & 0.451 & 0.020 & 4.342 & 0.459 \\
\textsc{RCAD}, $\alpha=0.25$ & 0.046 & 0.034 & 2.818 & 15.768 & 1.324 & 0.449 & 0.018 & 4.259 & 0.458 \\
\textsc{RCAD}, $\alpha=0.5$  & 0.039 & 0.029 & 2.520 & 14.293 & 1.073 & 0.434 & 0.015 & 4.111 & 0.457 \\
\textsc{RCAD}, $\alpha=0.75$ & 0.030 & 0.021 & 2.167 & 12.452 & 0.700 & 0.409 & 0.010 & 3.868 & 0.422 \\
\textsc{RCAD}, $\alpha=1.0$  & 0.018 & 0.013 & 1.944 & 11.156 & 0.491 & 0.397 & 0.008 & 3.671 & 0.360 \\
\midrule
\multicolumn{10}{l}{\textit{Two-Model Baselines (using $p_r$ and $p_s$)}} \\ \midrule
\textsc{CP-Fuse}   & 0.003 & 0.001 & 1.684 &  9.720 & 0.045 & 0.390 & 0.002 & 3.945 & 0.247 \\
\textsc{TokenSwap} & 0.009 & 0.003 & 2.094 & 12.186 & 0.324 & 0.422 & 0.004 & 3.981 & 0.415 \\
\midrule
\multicolumn{10}{l}{\textit{Ours (using $p_r$ and $p_s$)}} \\ \midrule
\oursbs{}, $k=0.1$  & 0.001 & 0.000 & 1.542 &  9.030 & 0.028 & 0.386 & 0.001 & 4.091 & 0.201 \\
\oursbs{}, $k=0.5$  & 0.004 & 0.001 & 1.655 &  9.817 & 0.061 & 0.407 & 0.002 & 4.231 & 0.433 \\
\oursbs{}, $k=1.0$  & 0.007 & 0.004 & 1.744 & 10.398 & 0.177 & 0.414 & 0.003 & 4.271 & 0.474 \\
\oursbs{}, $k=1.5$  & 0.011 & 0.007 & 1.805 & 10.646 & 0.252 & 0.423 & 0.004 & 4.296 & 0.478 \\
\oursbs{}, $k=2.0$  & 0.014 & 0.010 & 1.883 & 10.990 & 0.331 & 0.425 & 0.005 & 4.297 & 0.487 \\
\oursbs{}, $k=3.0$  & 0.019 & 0.012 & 2.020 & 11.758 & 0.449 & 0.429 & 0.006 & 4.319 & 0.485 \\
\oursbs{}, $k=4.0$  & 0.025 & 0.016 & 2.133 & 12.290 & 0.572 & 0.430 & 0.007 & 4.325 & 0.483 \\
\oursbs{}, $k=5.0$  & 0.030 & 0.019 & 2.199 & 12.640 & 0.667 & 0.431 & 0.009 & 4.326 & 0.479 \\
\oursbs{}, $k=10.0$ & 0.035 & 0.024 & 2.464 & 14.052 & 0.931 & 0.441 & 0.012 & 4.330 & 0.467 \\
\oursbs{}, $k=15.0$ & 0.048 & 0.036 & 2.867 & 16.133 & 1.356 & 0.451 & 0.018 & 4.342 & 0.462 \\
\oursbs{}, $k=20.0$ & 0.049 & 0.037 & 2.904 & 16.307 & 1.398 & 0.452 & 0.018 & 4.344 & 0.463 \\
\bottomrule
\end{tabular}%
}
\end{table*}

\begin{table*}[t]
\centering
\small
\setlength{\tabcolsep}{3.5pt}
\renewcommand{\arraystretch}{0.92}
\caption{\textbf{Full results for $p_s$=\ourlm{}, $p_r$=Llama 4 Scout 17Bx16E (byte-level decoding).} We report the average over 3 seeds.}
\label{tab:full_results_comma1.7b_llama4_109b}
\resizebox{\textwidth}{!}{%
\begin{tabular}{lcccccccc}
\toprule
\textbf{Setting} &
\textbf{ROUGE-1$\ge \tau$} &
\textbf{ROUGE-L$\ge \tau$} &
\textbf{Word LCS} &
\textbf{Char. LCS} &
\textbf{Word ACS} &
\textbf{MinHash} &
\textbf{Fluency} &
\textbf{FactScore} \\
\midrule
\multicolumn{9}{l}{\textit{Reference LMs}} \\ \midrule
$p_s$ & 0.000 & 0.000 & 1.537 &  9.022 & 0.000 & 0.001 & 3.032 & 0.088 \\
$p_r$ & 0.033 & 0.020 & 2.436 & 13.932 & 0.830 & 0.011 & 4.531 & 0.563 \\
\midrule
\multicolumn{9}{l}{\textit{Single-Model Methods (using $p_r$)}} \\ \midrule
\textsc{System}            & 0.031 & 0.020 & 2.502 & 14.346 & 0.886 & 0.012 & 4.477 & 0.560 \\
\textsc{MemFree}, $n=3$    & 0.021 & 0.011 & 2.008 & 11.614 & 0.476 & 0.006 & 4.300 & 0.481 \\
\textsc{MemFree}, $n=5$    & 0.024 & 0.012 & 2.138 & 12.312 & 0.509 & 0.007 & 4.411 & 0.535 \\
\textsc{MemFree}, $n=7$    & 0.025 & 0.014 & 2.169 & 12.505 & 0.531 & 0.008 & 4.445 & 0.550 \\
\textsc{MemFree}, $n=9$    & 0.025 & 0.012 & 2.202 & 12.638 & 0.576 & 0.008 & 4.464 & 0.552 \\
\textsc{MemFree}, $n=10$   & 0.026 & 0.012 & 2.194 & 12.671 & 0.574 & 0.008 & 4.467 & 0.555 \\
\textsc{RCAD}, $\alpha=0.1$  & 0.028 & 0.021 & 2.344 & 13.541 & 0.762 & 0.011 & 4.504 & 0.562 \\
\textsc{RCAD}, $\alpha=0.25$ & 0.028 & 0.017 & 2.210 & 12.870 & 0.641 & 0.010 & 4.459 & 0.551 \\
\textsc{RCAD}, $\alpha=0.5$  & 0.020 & 0.011 & 1.938 & 11.412 & 0.394 & 0.007 & 4.252 & 0.499 \\
\textsc{RCAD}, $\alpha=0.75$ & 0.013 & 0.007 & 1.666 & 10.023 & 0.278 & 0.006 & 3.915 & 0.461 \\
\textsc{RCAD}, $\alpha=1.0$  & 0.008 & 0.004 & 1.451 &  8.803 & 0.161 & 0.004 & 3.476 & 0.355 \\
\midrule
\multicolumn{9}{l}{\textit{Two-Model Baselines (using $p_r$ and $p_s$)}} \\ \midrule
\textsc{CP-Fuse}   & 0.002 & 0.001 & 1.776 & 10.275 & 0.045 & 0.002 & 3.969 & 0.270 \\
\textsc{TokenSwap} & 0.003 & 0.001 & 1.916 & 11.172 & 0.088 & 0.002 & 3.751 & 0.474 \\
\midrule
\multicolumn{9}{l}{\textit{Ours (using $p_r$ and $p_s$)}} \\ \midrule
\oursbs{}, $k=0.1$  & 0.000 & 0.000 & 1.583 &  9.192 & 0.000 & 0.001 & 3.407 & 0.121 \\
\oursbs{}, $k=0.5$  & 0.002 & 0.000 & 1.611 &  9.603 & 0.040 & 0.001 & 4.307 & 0.476 \\
\oursbs{}, $k=1.0$  & 0.005 & 0.003 & 1.659 &  9.768 & 0.072 & 0.002 & 4.430 & 0.556 \\
\oursbs{}, $k=1.5$  & 0.008 & 0.005 & 1.792 & 10.492 & 0.180 & 0.003 & 4.455 & 0.578 \\
\oursbs{}, $k=2.0$  & 0.008 & 0.005 & 1.805 & 10.587 & 0.200 & 0.004 & 4.486 & 0.580 \\
\oursbs{}, $k=3.0$  & 0.011 & 0.006 & 1.876 & 10.949 & 0.245 & 0.004 & 4.482 & 0.579 \\
\oursbs{}, $k=4.0$  & 0.010 & 0.005 & 1.865 & 10.916 & 0.238 & 0.004 & 4.521 & 0.583 \\
\oursbs{}, $k=5.0$  & 0.013 & 0.007 & 1.910 & 11.138 & 0.286 & 0.005 & 4.521 & 0.588 \\
\oursbs{}, $k=10.0$ & 0.016 & 0.006 & 2.017 & 11.701 & 0.297 & 0.006 & 4.521 & 0.569 \\
\oursbs{}, $k=15.0$ & 0.026 & 0.018 & 2.330 & 13.374 & 0.704 & 0.010 & 4.521 & 0.568 \\
\oursbs{}, $k=20.0$ & 0.028 & 0.018 & 2.335 & 13.395 & 0.715 & 0.010 & 4.520 & 0.569 \\
\bottomrule
\end{tabular}%
}
\end{table*}

\begin{table*}[t]
\centering
\small
\setlength{\tabcolsep}{3.5pt}
\renewcommand{\arraystretch}{0.92}
\caption{\textbf{Full results for $p_s$=Comma 7B, $p_r$=Llama 4 Scout 17Bx16E (byte-level decoding).} We report the average over 3 seeds.}
\label{tab:full_results_comma7b_llama4_109b}
\resizebox{\textwidth}{!}{%
\begin{tabular}{lcccccccc}
\toprule
\textbf{Setting} &
\textbf{ROUGE-1$\ge \tau$} &
\textbf{ROUGE-L$\ge \tau$} &
\textbf{Word LCS} &
\textbf{Char. LCS} &
\textbf{Word ACS} &
\textbf{MinHash} &
\textbf{Fluency} &
\textbf{FactScore} \\
\midrule
\multicolumn{9}{l}{\textit{Reference LMs}} \\ \midrule
$p_s$ & 0.000 & 0.000 & 1.537 &  9.022 & 0.000 & 0.001 & 3.032 & 0.088 \\
$p_r$ & 0.033 & 0.020 & 2.436 & 13.932 & 0.830 & 0.011 & 4.531 & 0.563 \\
\midrule
\multicolumn{9}{l}{\textit{Single-Model Methods (using $p_r$)}} \\ \midrule
\textsc{System}             & 0.031 & 0.020 & 2.502 & 14.346 & 0.886 & 0.012 & 4.477 & 0.560 \\
\textsc{MemFree}, $n=3$     & 0.021 & 0.011 & 2.008 & 11.614 & 0.476 & 0.006 & 4.300 & 0.481 \\
\textsc{MemFree}, $n=5$     & 0.024 & 0.012 & 2.138 & 12.312 & 0.509 & 0.007 & 4.411 & 0.535 \\
\textsc{MemFree}, $n=7$     & 0.025 & 0.014 & 2.169 & 12.505 & 0.531 & 0.008 & 4.445 & 0.550 \\
\textsc{MemFree}, $n=9$     & 0.025 & 0.012 & 2.202 & 12.638 & 0.576 & 0.008 & 4.464 & 0.552 \\
\textsc{MemFree}, $n=10$    & 0.026 & 0.012 & 2.194 & 12.671 & 0.574 & 0.008 & 4.467 & 0.555 \\
\textsc{RCAD}, $\alpha=0.1$  & 0.028 & 0.021 & 2.344 & 13.541 & 0.762 & 0.011 & 4.504 & 0.562 \\
\textsc{RCAD}, $\alpha=0.25$ & 0.028 & 0.017 & 2.210 & 12.870 & 0.641 & 0.010 & 4.459 & 0.551 \\
\textsc{RCAD}, $\alpha=0.5$  & 0.020 & 0.011 & 1.938 & 11.412 & 0.394 & 0.007 & 4.252 & 0.499 \\
\textsc{RCAD}, $\alpha=0.75$ & 0.013 & 0.007 & 1.666 & 10.023 & 0.278 & 0.006 & 3.915 & 0.461 \\
\textsc{RCAD}, $\alpha=1.0$  & 0.008 & 0.004 & 1.451 &  8.803 & 0.161 & 0.004 & 3.476 & 0.355 \\
\midrule
\multicolumn{9}{l}{\textit{Two-Model Baselines (using $p_r$ and $p_s$)}} \\ \midrule
\textsc{CP-Fuse}   & 0.002 & 0.001 & 1.776 & 10.275 & 0.045 & 0.002 & 3.969 & 0.270 \\
\textsc{TokenSwap} & 0.003 & 0.001 & 1.916 & 11.172 & 0.088 & 0.002 & 3.751 & 0.474 \\
\midrule
\multicolumn{9}{l}{\textit{Ours (using $p_r$ and $p_s$)}} \\ \midrule
\oursbs{}, $k=0.1$  & 0.000 & 0.000 & 1.583 &  9.192 & 0.000 & 0.001 & 3.407 & 0.121 \\
\oursbs{}, $k=0.5$  & 0.002 & 0.000 & 1.611 &  9.603 & 0.040 & 0.001 & 4.307 & 0.476 \\
\oursbs{}, $k=1.0$  & 0.005 & 0.003 & 1.659 &  9.768 & 0.072 & 0.002 & 4.430 & 0.556 \\
\oursbs{}, $k=1.5$  & 0.008 & 0.005 & 1.792 & 10.492 & 0.180 & 0.003 & 4.455 & 0.578 \\
\oursbs{}, $k=2.0$  & 0.008 & 0.005 & 1.805 & 10.587 & 0.200 & 0.004 & 4.486 & 0.580 \\
\oursbs{}, $k=3.0$  & 0.011 & 0.006 & 1.876 & 10.949 & 0.245 & 0.004 & 4.482 & 0.579 \\
\oursbs{}, $k=4.0$  & 0.010 & 0.005 & 1.865 & 10.916 & 0.238 & 0.004 & 4.521 & 0.583 \\
\oursbs{}, $k=5.0$  & 0.013 & 0.007 & 1.910 & 11.138 & 0.286 & 0.005 & 4.521 & 0.588 \\
\oursbs{}, $k=10.0$ & 0.016 & 0.006 & 2.017 & 11.701 & 0.297 & 0.006 & 4.521 & 0.569 \\
\oursbs{}, $k=15.0$ & 0.026 & 0.018 & 2.330 & 13.374 & 0.704 & 0.010 & 4.521 & 0.568 \\
\oursbs{}, $k=20.0$ & 0.028 & 0.018 & 2.335 & 13.395 & 0.715 & 0.010 & 4.520 & 0.569 \\
\bottomrule
\end{tabular}%
}
\end{table*}

\end{document}